\documentclass{article}

\usepackage[nonatbib,preprint]{neurips_2026}
\usepackage{microtype}
\usepackage{graphicx}
\usepackage{subcaption}
\usepackage{booktabs} 

\usepackage{hyperref}
\usepackage{lscape}

\usepackage{amsmath}
\usepackage{amssymb}
\usepackage{mathtools}
\usepackage{amsthm}

\theoremstyle{plain}
\newtheorem{theorem}{Theorem}[section]

\newtheorem{lemma}[theorem]{Lemma}
\newtheorem{corollary}[theorem]{Corollary}
\theoremstyle{definition}

\newtheorem{assumption}[theorem]{Assumption}
\theoremstyle{remark}
\newtheorem{remark}[theorem]{Remark}

\usepackage{amsmath,amssymb,amsthm}
\usepackage{bm}
\usepackage[square,numbers]{natbib}
\usepackage{graphicx,here,comment}
\usepackage{algorithmic}
\usepackage{algorithm}
\usepackage{xcolor}
\usepackage{wrapfig}
\usepackage{footnote}
\usepackage{multirow}
\makesavenoteenv{table}  

\usepackage{mathtools}

\newcommand{\rbr}[1]{\left(#1\right)}
\newcommand{\sbr}[1]{\left[#1\right]}
\newcommand{\cbr}[1]{\left\{#1\right\}}

\newcommand{\R}{\mathbb{R}}
\newcommand{\N}{\mathbb{N}}
\newcommand{\mF}{\mathcal{F}}

\newcommand{\mH}{\mathcal{H}}

\newcommand{\mP}{\mathcal{P}}

\newcommand{\mN}{\mathcal{N}}
\newcommand{\mX}{\mathcal{X}}

\newcommand{\mR}{\mathcal{R}}
\newcommand{\mS}{\mathcal{S}}

\newcommand{\Ep}{\mathbb{E}}
\renewcommand{\Pr}{\mathbb{P}}

\renewcommand{\hat}{\widehat}
\renewcommand{\tilde}{\widetilde}

\newcommand{\argmin}{\operatornamewithlimits{argmin}}
\newcommand{\argmax}{\operatornamewithlimits{argmax}}

\def\bA{\bm{A}}
\def\bX{\mathbf{X}}

\def\bK{\mathbf{K}}
\def\bS{\mathbf{S}}
\def\bQ{\mathbf{Q}}
\def\bPsi{\mathbf{\Psi}}
\def\bPhi{\mathbf{\Phi}}

\def\bk{\bm{k}}

\def\bI{\bm{I}}

\def\bx{\bm{x}}

\def\by{\bm{y}}
\def\bz{\bm{z}}

\usepackage{color}

\usepackage{newtxtext,newtxmath}

\newcommand{\secref}[1]{Sec.~\ref{#1}}
\newcommand{\appref}[1]{Appendix~\ref{#1}}
\newcommand{\algoref}[1]{Alg.~\ref{#1}}
\newcommand{\asmpref}[1]{Assumption~\ref{#1}}

\newcommand{\thmref}[1]{Theorem~\ref{#1}}
\newcommand{\lemref}[1]{Lemma~\ref{#1}}
\newcommand{\corref}[1]{Corollary~\ref{#1}}

\newcommand{\tabref}[1]{Tab.~\ref{#1}}

\newcommand{\1}{\mbox{1}\hspace{-0.25em}\mbox{l}}

\newcommand{\sek}{k_{\mathrm{SE}}}
\newcommand{\matk}{k_{\mathrm{Mat}}}

\title{Nearly Optimal Algorithms for \\ Adversarial Kernelized Bandits}

\author{%
  Shogo Iwazaki \\
  LY Corporation \\
  Tokyo, Japan \\
  \texttt{siwazaki@lycorp.co.jp} \\
}

\begin{document}
\maketitle

\begin{abstract}
    This paper studies kernelized bandits (also known as Gaussian process bandits) in an adversarial environment, where the reward functions in a known reproducing kernel Hilbert space (RKHS) may be adversarially chosen at each round. 
    We show that the exponential-weight algorithm achieves $\tilde{O}(\sqrt{T \gamma_T})$ adversarial regret, where $T$ and $\gamma_T$ denote the number of total rounds and the maximum information gain, respectively. For squared exponential (SE) and $\nu$-Mat\'ern kernels, we also show algorithm-independent lower bounds that guarantee the optimality of our algorithm up to polylogarithmic factors.
    Furthermore, we present a computationally efficient variant of our algorithm using Nystr\"om approximation while maintaining nearly optimal regret guarantees.
\end{abstract}

\section{Introduction}
\label{sec:intro}
The kernelized bandit (KB) problem, also known as the Gaussian process (GP) bandit, is an effective sequential decision-making framework for nonlinear reward functions. Its applications are extensive, e.g., robotics~\citep{lizotte2007automatic}, recommendation systems~\citep{yadav2024gaussian}, and hyperparameter tuning~\citep{snoek2012practical}. In prior research on KB, most existing works focus on the \emph{stochastic setting} in which the learner makes decisions from noisy observations of a common \emph{fixed} reward function over rounds.
A key desideratum in the application of KB is to develop algorithms that bypass the restrictive assumptions in a stochastic setting, which are often invalid in practice. For example, in recommendation systems, the presence of a malicious user may critically degrade the performance of the stochastic KB algorithm~\citep{han2022adversarial}. Decision-making in the stock market may also encounter adversarial behavior from other market participants. 
To overcome the limitations of the stochastic setting, several existing works~\citep{chatterji2019online,takemori2021approximation} have studied the \emph{adversarial} KB problem, aiming to design algorithms under minimal smoothness assumptions encoded via the kernel. 

This paper addresses a significant open problem in adversarial KB: the construction of a nearly-optimal algorithm. Specifically, to our knowledge, existing algorithms fail to attain even a no-regret guarantee for the commonly used $\nu$-Mat\'ern kernels. Our algorithm and theoretical findings substantially enhance the existing theoretical guarantees in adversarial KB.

\paragraph{Contributions.} Our contributions are described below.
\begin{itemize}
    \item In \secref{sec:kernelized_exp3}, we propose \textbf{kernelized Exp3}, which is the kernelized extension of the exponential-weight algorithm for exploration and exploitation (Exp3). We provide an $\tilde{O}(\sqrt{T \gamma_T})$ regret upper bound for kernelized Exp3, where $T$ and $\gamma_T$ represent the total number of rounds and the \emph{maximum information gain} (MIG), respectively (the definition is in \secref{sec:prelim})\footnote{The notations $\tilde{O}(\cdot)$ and $\tilde{\Omega}(\cdot)$ ignore the polylogarithmic factors. A comprehensive notation table is in \appref{app:notation}.
    }. 
    In addition, this is the first guarantee for adversarial KB under an \emph{adaptive} adversary.
    \item In \secref{sec:exp3_nystrom_approx}, we propose \textbf{RLS-kernelized Exp3}, which is a variant of kernelized Exp3 that alleviates the heavy computational burden of the original kernelized Exp3. Specifically, by utilizing the existing ridge-leverage-score (RLS)-based Nystr\"om method~\citep{musco2017recursive}, the per-round computational complexity of RLS-kernelized Exp3 becomes $\tilde{O}(|\mX| \gamma_T^2 + \gamma_T^3)$, while that of the original kernelized Exp3 is $O(|\mX|^3)$. Here, $|\mX|$ denotes the cardinality of the input domain $\mX \subset \R^d$\footnote{As described in \secref{sec:kernelized_exp3}, our analysis is applicable to any compact input domain via the discretization. }.
    Furthermore, RLS-kernelized Exp3 also achieves an $\tilde{O}(\sqrt{T\gamma_T})$ regret.
    \item Finally, in \secref{sec:lower_bound}, we show the algorithm-independent lower bounds for the adversarial KB problem with SE or $\nu$-Mat\'ern kernels. Our lower bounds guarantee that the proposed algorithms are optimal up to polylogarithmic factors for SE and $\nu$-Mat\'ern kernels with $\nu > 1$.
\end{itemize}
\tabref{tab:summary} summarizes the existing algorithm and our results.

\subsection{Related works}
The KB algorithms in the stochastic setting have been widely studied, 
including e.g., GP-upper confidence bound (GP-UCB)~\citep{srinivas10gaussian}, GP-Thompson sampling (GP-TS)~\citep{chowdhury2017kernelized}, and theoretically more advanced non-adaptive sampling-based algorithms~\citep{camilleri2021high,li2022gaussian,salgia2021domain,valko2013finite}. In addition to the standard stochastic KB settings, several variants are also studied, e.g., the contextual model~\citep{krause2011contextual,valko2013finite}, the multi-objective setting~\citep{chowdhury2021no,zuluaga2016pal}, the parallel setting~\citep{desautels2014parallelizing}, the noise-free setting~\citep{bull2011convergence,iwazaki2025gaussian,iwazaki2025improvedregretanalysisgaussian,vakili2022open}, and the Bayesian setting~\citep{iwazaki2025improved,russo2014learning,Russo2014-learning,scarlett2018tight,takeno2023-randomized,takeno2026regret}. In addition to these settings, some existing studies have examined more robust or general variants of the standard stochastic KB model, which are orthogonal to the adversarial KB model in this paper. For example, robust objective settings~\citep{bogunovic2018adversarially,iwazaki2021mean,kirschner2020distributionally,nguyen2021value,saday2023robust}, adversarial corruptions of rewards~\citep{bogunovic2020corruption,bogunovic2022robust,han2022adversarial}, adversarial contextual information~\citep{sessa2019no,sessa2020learning}, and non-stationary rewards~\citep{cai2025lower,deng2022weighted,iwazaki2025near,zhou2021no} have been considered.

Regarding the adversarial KB problem, to our knowledge, \citep{chatterji2019online} is the first work to tackle it. However, their analysis is only valid under the quite restrictive rank $1$ assumption; specifically, the reward function $f_t$ at each round $t$ must have the form $f_t(\cdot) = k(\bz_t, \cdot)$. 
In addition, their lower bounds~(Theorem 43 in \citep{chatterji2019online}) only apply to an unrealistic infinite-dimensional linear kernel and do not imply lower bounds for an adversarial KB model with practically used SE or Mat\'ern kernels.\footnote{The concurrent work by \citet{zhang2026nearoptimalregretadversarialkernel} explicitly discusses this point. See Appendix B in \citep{zhang2026nearoptimalregretadversarialkernel} for details.} 
\citet{takemori2021approximation} proposed the algorithm, which broke the limitation of \citep{chatterji2019online}, called APG-Exp3. 
The core idea of \citep{takemori2021approximation} is to reduce the adversarial KB problem to an approximated (misspecified) linear bandit (LB) problem. However, due to the non-negligible effect of the linear approximation, the current analysis of APG-Exp3 fails to provide even a no-regret guarantee under the commonly-used $\nu$-Mat\'ern kernel with $d \geq \nu$ (see \tabref{tab:summary}). 
Another related work for adversarial KB is \citep{neu2024adversarial}, which studies the $K$-armed adversarial contextual KB problem and remains distinct from our setting. Their model considers independent $K$ reward functions on the RKHS defined over the context space (with $K \ll T$). In contrast, \citep{chatterji2019online,takemori2021approximation} and our analysis assume that the reward function is an element of RKHS defined on \emph{entire} decision space $\mX$, whose cardinality may be $|\mX| \gg T$ or infinite.

From a technical perspective, our algorithm construction and several parts of our analysis leverage the existing results in the adversarial LB and stochastic KB. Firstly, from the high-level view, our algorithm construction is interpreted as an extension of the existing Exp3-based algorithms for adversarial LB~\citep{bartlett2008high,lattimore2020bandit,zimmert2022return}. The differences from the existing Exp3 are the one-sample estimator of the reward function and the exploration distribution, which are redesigned by adapting the existing theoretical tools in the stochastic KB problem~\citep{cai2021lenient,camilleri2021high,hong2023optimization,vakili2021optimal}. 
Secondly, in our analysis of the regret upper bound, the existing information gain-based arguments by \citet{srinivas10gaussian} play a central role in quantifying the regret. To our knowledge, our paper is the first to bridge the adversarial KB theory with the information-gain-based arguments commonly used in stochastic KB analysis. Thirdly, regarding the RLS-kernelized Exp3 proposed in \secref{sec:exp3_nystrom_approx}, we construct our algorithm by combining our kernelized Exp3 with the existing RLS-based Nystr\"om approximation~\citep{musco2017recursive}. Furthermore, several parts of our regret analysis are motivated by the techniques in \citep{calandriello2019gaussian}, which studies the stochastic KB algorithm with Nystr\"om approximation. Finally, regarding the lower bound, our analysis reduces the adversarial KB lower bound to that of the stochastic KB analyzed in~\citep{cai2021on,scarlett2017lower}. Specifically, our hard problem instance in the proof of the lower bound is interpreted as the kernelized extension of the parameter noise model in LB (see Chap.~29 in \citep{lattimore2020bandit}), which is used to study the lower bound for adversarial LB~\citep{shamir2015complexity}.

\paragraph{Concurrent work.} The concurrent work by \citet{zhang2026nearoptimalregretadversarialkernel} studies an algorithm nearly identical to our kernelized Exp3 in \secref{sec:kernelized_exp3}. Indeed, our kernelized Exp3 becomes identical to theirs by replacing our MVR-based exploration distribution with a G-optimal design-based one (defined in Appendix~\ref{app:exp_dist_discuss}). Furthermore, their regret analysis and ours are almost identical. The summary of the minor differences between our algorithm and theirs can be found in Appendix B of \citep{zhang2026nearoptimalregretadversarialkernel}. We would like to note that our proposal of the computationally efficient nearly-optimal algorithm (RLS-kernelized Exp3 in \secref{sec:exp3_nystrom_approx}) and the lower bounds for SE and Mat\'ern kernels (in \secref{sec:lower_bound}) do not overlap with the contents in \citep{zhang2026nearoptimalregretadversarialkernel}.

\begin{table}
    \centering
    \caption{The summary of our results and the existing algorithm under a compact input domain $\mX \subset \R^d$. 
    Here, $D_T$ denotes the number of basis functions returned by the sub-routine in APG-Exp3. 
    For the $\nu$-Mat\'ern kernel, the reported bound $D_T = \tilde{O}\rbr{T^{\frac{d}{\nu}}}$~\citep{takemori2021approximation} has a larger exponent of $T$ than the bound on the MIG, $\gamma_T = \tilde{O}\rbr{T^{\frac{d}{2\nu+d}}}$~\citep{iwazaki2025improved,vakili2021information}. 
    Regarding the per-round computational cost in the table, we assume $|\mX| < \infty$. For the continuous input domain, the per-round computational cost can be obtained by substituting $|\mX|$ in the table with $O(T^{d/2})$ (the detail is in \secref{sec:kernelized_exp3}).}
    \begin{tabular}{ccccc}
        \toprule
        \multirow{2}{*}{Algorithm} & Regret &  Regret & Per-round & Adaptive \\
         & (General kernel) & ($\nu$-Mat\'ern) & computational cost & adversary \\ \midrule \midrule
        APG-Exp3 & \multirow{2}{*}{$\tilde{O}(\sqrt{T D_T})$}  & \multirow{2}{*}{$\tilde{O}\rbr{T^{\frac{\nu + d}{2\nu}}}$} & \multirow{2}{*}{$O\rbr{|\mX|D_T^2 + D_T^3}$} & \multirow{2}{*}{No} \\
        (\citet{takemori2021approximation}) & & & & \\ \midrule
        Kernelized Exp3 & \multirow{2}{*}{$\tilde{O}(\sqrt{T \gamma_T})$} & \multirow{2}{*}{$\tilde{O}\rbr{T^{\frac{\nu+d}{2\nu+d}}}$} & \multirow{2}{*}{$O\rbr{|\mX|^3}$} & \multirow{2}{*}{Yes} \\
        (\textbf{Ours}, in \secref{sec:kernelized_exp3}) & & & & \\ \midrule
        RLS-kernelized Exp3  & \multirow{2}{*}{$\tilde{O}(\sqrt{T \gamma_T})$} & \multirow{2}{*}{$\tilde{O}\rbr{T^{\frac{\nu+d}{2\nu+d}}}$} & \multirow{2}{*}{$\tilde{O}\rbr{|\mX|\gamma_T^2 + \gamma_T^3}$} & \multirow{2}{*}{Yes} \\
        (\textbf{Ours}, in \secref{sec:exp3_nystrom_approx}) & & & &
        \\ \midrule \midrule
        Lower bound & \multirow{2}{*}{N/A} &  \multirow{2}{*}{$\tilde{\Omega}\rbr{T^{\frac{\nu+d}{2\nu+d}}}$} & & \\
        (\textbf{Ours}, in \secref{sec:lower_bound}) & & & & \\
        \bottomrule
    \end{tabular}
    \label{tab:summary}
\end{table}

\section{Preliminaries}
\label{sec:prelim}

\paragraph{Problem setting.} We study the KB problem in an adversarial environment. Let $\mX \subset [0, 1]^d$ be the input domain of the problem. At each round $t$, the learner sequentially draws a query point $\bx_t \in \mX$ based on the history up to round $t-1$; subsequently, the bandit feedback of the reward $f_t(\bx_t)$ is revealed, where $f_t: \mX \rightarrow \R$ is the reward function drawn by the environment. 
This paper assumes the environment to be \emph{adaptive}. Namely, the environment can draw $f_t$ based on the learner's past history $\mH_{t-1} \coloneqq \{\bx_1, f_1(\bx_1), \ldots, \bx_{t-1}, f_{t-1}(\bx_{t-1})\}$, while ($f_t$, $\bx_t$) must be conditionally independent given $\mH_{t-1}$. Under this setup, the learner's goal is to minimize the following expected (pseudo) regret $\bar{R}_T$:
\begin{equation}
    \bar{R}_T = \sup_{\bx \in \mX} \Ep\sbr{R_T(\bx)},~\mathrm{where}~R_T(\bx) = \sum_{t \in [T]} f_t(\bx) - \sum_{t \in [T]} f_t(\bx_t)~\mathrm{and}~[T] = \{1, \ldots, T\}.
\end{equation}
In this paper, we adopt the following standard assumption of KB, which demands that the reward functions lie in a known RKHS.

\begin{assumption}
    \label{asmp:regularity}
    For some known positive definite kernel function $k: \mX \times \mX \rightarrow \R$ and some known bounded constant $B > 0$, 
    the reward function $f_t$ drawn by the environment satisfies $f_t \in \mH_k$ and $\|f_t\|_k \leq B$ for any $t \in \N_+$. 
    Here, $\mH_k$ and $\|\cdot\|_k$ denote the RKHS endowed with the kernel $k$ and its RKHS norm, respectively. Furthermore, we assume that $k(\bx, \bx) \leq 1$ for all $\bx \in \mX$.
\end{assumption}
As examples of the kernel, the following SE and $\nu$-Mat\'ern kernels are commonly studied:
\begin{align*}
    \sek(\bx, \bx') = \exp\rbr{-\frac{\|\bx - \bx'\|_2^2}{2\ell^2}},~\matk(\bx, \bx') = \frac{2^{1-\nu}}{\Gamma(\nu)} \rbr{\frac{\sqrt{2\nu} \|\bx - \bx'\|_2}{\ell}}^{\nu} K_{\nu}\rbr{\frac{\sqrt{2\nu} \|\bx - \bx'\|_2}{\ell}},
\end{align*}
where $\ell > 0$ and $\nu > 0$ are the lengthscale and smoothness parameters, respectively. Furthermore, $\Gamma(\nu)$ and $K_{\nu}(\cdot)$ denote the Gamma function and the modified Bessel function of the second kind, respectively.
Here, our analysis assumes that the underlying kernel function is known beforehand as in the vast majority of existing analyses~\citep{camilleri2021high,chowdhury2017kernelized,li2022gaussian,srinivas10gaussian,vakili2021optimal,valko2013finite}. Addressing more practical scenarios with unknown kernels is beyond the scope of this work and is future work.

\paragraph{Gaussian process (GP) model.}
In a stochastic setting of KB, the GP model~\citep{Rasmussen2005-Gaussian} is a commonly used tool for guiding the learner. In the adversarial setting, although GP is not directly used in our algorithm construction, it plays an important role in quantifying regret via GP's information gain in the analysis. Let us consider a function $f$ that follows a mean-zero GP prior $f \sim \mathcal{GP}(0, k)$ characterized by a covariance kernel $k$. Then, for some training inputs $\bX = (\bx^{(1)}, \ldots, \bx^{(t)})$ and outputs $\bm{y} = (y^{(1)}, \ldots, y^{(t)})^{\top}$, the GP model assumes that outputs are generated as $y^{(i)} = f(\bx^{(i)}) + \epsilon^{(i)}$, where $\epsilon^{(i)} \sim \mN(0, \lambda)$ is the Gaussian noise with variance $\lambda > 0$, and is independent across $i \in [t]$. Under the aforementioned modeling assumptions, the posterior distribution of $f$ given $(\bX, \by)$ is also a GP~\citep{Rasmussen2005-Gaussian}. Specifically, its posterior variance $\sigma^2(\bx; \bX, \lambda)$ of $f(\bx)$ is often used to upper-bound the regret of the KB problem:
\begin{align}
    \label{eq:posterior_var}
    \sigma^2(\bx; \bX, \lambda) = k(\bx, \bx) - \bk(\bx, \bX)^{\top} \rbr{\bK(\bX, \bX) + \lambda \bI_t}^{-1} \bk(\bx, \bX),
\end{align}
where $\bk(\bx, \bX) \coloneqq [k(\bx, \bx^{(i)})]_{i \in [t]} \in \R^t$, $\bK(\bX, \bX) \coloneqq [k(\bx^{(i)}, \bx^{(j)})]_{i, j \in [t]} \in \R^{t \times t}$, and $\bI_t$ are the kernel vector, the kernel matrix, and the $t\times t$ identity matrix, respectively.

\paragraph{Maximum information gain (MIG).}
The MIG of a GP is a widely-used complexity parameter for the stochastic KB problem~\citep{srinivas10gaussian}. Given some inputs $\bX = [\bx^{(i)}]_{i \in [t]}$, let $I(\bm{f}(\bX); \bm{y}) \coloneqq \frac{1}{2} \log \det(\bI_t + \lambda^{-1} \bK(\bX, \bX))$ be the mutual information between random vectors $\bm{f}(\bX) = (f(\bx^{(1)}), \ldots, f(\bx^{(t)}))^{\top} \sim \mN(\bm{0}, \bK(\bX, \bX))$ and $\bm{y} \coloneqq \bm{f}(\bX) + \bm{\epsilon}$, where $\bm{\epsilon} \sim \mN(\bm{0}, \lambda \bI_t)$. Namely, $I(\bm{f}(\bX); \bm{y})$ represents the mutual information between underlying function values and outputs under the Bayesian assumption of GP. Then, the MIG $\gamma_t(\lambda, \mX)$ is defined as the maximum amount of $I(\bm{f}(\bX); \bm{y})$ over all possible selections of $t$ inputs $\bX \in \mX^t$: $\gamma_t(\lambda, \mX) \coloneqq \sup_{\bX \in \mX^t} I(\bm{f}(\bX); \bm{y})$.
For simplicity, we use the notation $\gamma_t$ by omitting the dependence on $\lambda$ and $\mX$. For commonly used kernels, the growth rates of MIG are known. For example, for the SE kernel, $\gamma_t = O((\log t)^{d+1}/(\log \log t)^d)$ holds~\citep{iwazaki2026tighter}. Regarding $\nu$-Mat\'ern kernels, it is known that $\gamma_t = \tilde{O}(t^{\frac{d}{2\nu+d}})$~\citep{iwazaki2025improved,vakili2021information}.
Here, one of the well-known useful results for regret analysis of stochastic KB is the relation between MIG and the posterior variance of GP. Specifically, it is known that the cumulative posterior variance at the observed points is bounded from above by the MIG~\citep[Lemma 5.3 and 5.4 in][]{srinivas10gaussian}:
\begin{equation}
    \label{eq:sigma_2_cum_ub}
    \sum_{i \in [t]} \sigma^2 (\bx^{(i)}; \bX^{(i-1)}, \lambda) \leq \frac{2 \gamma_t}{\log (1 + \lambda^{-1})},
\end{equation}
where we define $\bX^{(i-1)}$ as $\bX^{(i-1)} = (\bx^{(1)}, \ldots, \bx^{(i-1)})$ for any $i$. Even in the adversarial KB setting, the above inequality plays a central role in upper-bounding the regret, as highlighted in \secref{sec:kernelized_exp3}.

\paragraph{Feature representation of kernel and GP.}
In our algorithm construction and analysis, we leverage the feature representation of kernels. In general, it is known that a positive definite kernel $k: \mX \times \mX \rightarrow \R$ cannot always be represented by a finite-dimensional feature map for a continuous input domain $\mX$; however, for a finite input domain $\mX$, we can always construct an $|\mX|$-dimensional feature map $\psi: \R^d \rightarrow \R^{|\mX|}$, which satisfies $k(\bx, \bx') = \psi(\bx)^{\top}\psi(\bx')$ for all $\bx, \bx' \in \mX$. 
For example, for kernel matrix $\bK(\mX, \mX) \coloneqq [k(\bx, \bx')]_{\bx, \bx' \in \mX}$, the Cholesky decomposition $\bK(\mX, \mX) \coloneqq \bm{L} \bm{L}^{\top}$ induces the feature map by defining $\psi(\bx^{(i)})$ as the transpose of the $i$-th row of $\bm{L} \in \R^{|\mX| \times |\mX|}$.

\begin{algorithm}[t!]
    \caption{Kernelized Exp3 for adversarial kernelized bandits}
    \label{alg:exp3}
    \begin{algorithmic}[1]
        \REQUIRE Finite input domain $\mX$, kernel $k: \mX \times \mX \rightarrow \R$, learning rate $\eta > 0$, regularization parameter $\lambda \geq 1$, mixing ratio $\alpha \in (0, 1)$, confidence width parameter $\beta > 0$, 
        total budget $T \in \N_+$.
        \STATE Prepare feature map $(\psi(\bx))_{\bx \in \mX}$ (e.g., by Cholesky decomposition).
        \STATE Prepare the exploration distribution $\pi(\bx)$ by using MVR-input sequence $(\bx_t^{(\mathrm{MVR})})_{t \leq \lceil T \alpha \rceil}$:
        \begin{equation}
            \pi(\bx) = \frac{1}{\lceil T \alpha \rceil} \sum_{t=1}^{\lceil T \alpha \rceil} \1\cbr{\bx_t^{(\mathrm{MVR)}} = \bx},
        \end{equation}
        where $\bx_t^{(\mathrm{MVR})} \in \argmax_{\bx \in \mX} \sigma^2\rbr{\bx; \bX_{t-1}^{(\mathrm{MVR})}, \lambda}$ and $\bX_{t-1}^{(\mathrm{MVR})} = \left(\bx_1^{(\mathrm{MVR})}, \ldots, \bx_{t-1}^{(\mathrm{MVR})}\right)$.
        \STATE Initialize sampling distribution: $P_1(\bx) \leftarrow \alpha \pi(\bx) + (1 - \alpha) \tilde{P}_1(\bx)$, where $\tilde{P}_1(\bx) = \frac{1}{|\mX|}$.
        \FOR {$t = 1, 2, \ldots, T$}
            \STATE The environment draws $f_t$, and the learner draws $\bm{x}_{t} \sim P_t$ and receives the reward $f_t(\bx_t)$.
            \STATE Calculate the reward estimator $\hat{f}_t$ as in Eq.~\eqref{eq:reg_ips}.
            \STATE Calculate optimistic estimator $u_t$ as $u_t(\bx) = \hat{f}_t(\bx) + \beta \|\psi(\bx)\|_{G_t(\lambda)^{-1}}$.
            \STATE Update the sampling distribution:
            \begin{equation}
                P_{t+1}(\bx) = \alpha \pi(\bx) + (1 - \alpha) \tilde{P}_{t+1}(\bx),~\text{where}~\tilde{P}_{t+1}(\bx) = \frac{\exp\rbr{\eta \sum_{i=1}^t u_i(\bx)}}{\sum_{\bx' \in \mX} \exp\rbr{\eta \sum_{i=1}^t u_i(\bx')}}.
            \end{equation}
        \ENDFOR
    \end{algorithmic}
\end{algorithm}

\section{Exponential-weight algorithm for adversarial kernelized bandits}
\label{sec:kernelized_exp3} 
The pseudocode for our first algorithm, kernelized Exp3, is presented in \algoref{alg:exp3}. While \algoref{alg:exp3} assumes that $\mX$ is a finite set, the extension to a continuous input domain is obtained by discretizing $\mX$, as discussed later in this section.
From a high-level perspective, \algoref{alg:exp3} is constructed based on the Exp3 algorithm for adversarial LB~\citep{bartlett2008high,zimmert2022return}, using a kernel feature map. 
The primary differences lie in (i) the choice of the estimator for $f_t(\cdot)$ (Lines 6--7) and (ii) the choice of the exploration distribution (Line~2).
\paragraph{Estimator.} To use the Exp3 algorithm, at each round $t$, we require an estimator $\hat{f}_t(\cdot)$ of the reward function. The standard Exp3 algorithm in the adversarial LB adopts an unbiased estimator based on inverse propensity score (IPS); however, the IPS estimator suffers from high variance as the feature map dimension increases, making it unsuitable for a kernelized setting. To address this issue, we adopt the following regularized version of the IPS estimator, whose effectiveness has been demonstrated in the stochastic KB literature~\citep{camilleri2021high,hong2023optimization,mason2022nearly}:
\begin{equation}
    \label{eq:reg_ips}
    \hat{f}_t(\bx) = \psi(\bx)^{\top} G_t(\lambda)^{-1} \psi(\bx_t) f_t(\bx_t),~\text{where}~~G_t(\lambda) = \sum_{\bx \in \mX} P_t(\bx) \psi(\bx) \psi(\bx)^{\top} + \frac{\lambda}{T} \bI_{|\mX|}.
\end{equation}
Here, $P_t(\cdot)$ is the sampling distribution of the learner at round $t$, and $\lambda > 0$ is the regularization parameter. Note that, when $\lambda = 0$, 
$\hat{f}_t(\cdot)$ reduces to the standard IPS estimator used in LB~\citep[see, e.g., Chap.~27 in ][]{lattimore2020bandit}.
Due to the regularization, the estimator $\hat{f}_t(\bx)$ is not unbiased when $\lambda > 0$. To prevent the exploration behavior from being undermined by the bias of the estimator, we use the optimistic estimator $u_t(\cdot) \coloneqq \hat{f}_t(\cdot) + \beta \|\psi(\cdot)\|_{G_t(\lambda)^{-1}}$ (Line 7) in the exponential-weight procedures by adding the optimistic term $\beta \|\psi(\bx)\|_{G_t(\lambda)^{-1}}$, where $\|\psi(\bx)\|_{G_t(\lambda)^{-1}} \coloneqq \sqrt{\psi(\bx)^{\top} G_t(\lambda)^{-1} \psi(\bx)}$. 
Note that the addition of such a term is the common algorithm construction strategy in the adversarial LB~\citep{bartlett2008high,zimmert2022return}.

\paragraph{Exploration distribution.} As with Exp3 for adversarial LB, we adopt an exploration distribution $\pi(\bx)$ to prevent the estimator $\hat{f}_t(\cdot)$ from becoming unstable due to the extremely small sampling probabilities. As we will see later in \lemref{lem:est_var}, the variance (the second moment) of the estimator $\hat{f}_t(\cdot)$ is dominated by the posterior variance of a GP, whose inputs are sampled from the distribution $P_t$.
Based on this fact, we propose using the empirical distribution of the input sequence generated by \emph{maximum variance reduction} (MVR)~\citep{vakili2021optimal} (see, \algoref{algo:mvr} in Appendix~\ref{app:pseudocode}), which greedily minimizes the maximum posterior variance of the GP.

\paragraph{Regret upper bound.} We provide the following regret upper bound for kernelized Exp3.

\begin{theorem}[General regret upper bound for kernelized Exp3]
    \label{thm:reg_exp3}
    Suppose that Assumption~\ref{asmp:regularity} holds. Assume that $\mX$ is finite. Furthermore, set $\lambda \geq 1$, $\beta = B\sqrt{\lambda/T}$, $\eta \in (0, \overline{\eta})$, and $\alpha = 2\eta (2 B \gamma_T + \beta \sqrt{\gamma_T})$, where $\overline{\eta} = 2^{-1} (2B \gamma_T + \beta \sqrt{\gamma_T})^{-1}$.
    Then, \algoref{alg:exp3} satisfies
    \begin{equation}
        \label{eq:raw_regret_ub}
        \bar{R}_T \leq 2\alpha B T + \frac{\log |\mX|}{\eta} + 8 \eta (\lambda + T) B^2 \gamma_T + 4B\sqrt{\lambda T \gamma_T}.
    \end{equation}
\end{theorem}
The full proof is provided in \appref{app:proof_reg_exp3}.
The desired regret upper bound is obtained by tuning the learning rate $\eta$ depending on $T$ and $|\mX|$, as formally stated in the following \corref{coro:ke3_reg_tuned}.

\begin{corollary}[Regret upper bound for kernelized Exp3 with tuned learning rate $\eta$]
    \label{coro:ke3_reg_tuned}
    Fix $\lambda \geq 1$ and set $\beta = B\sqrt{\lambda/T}$, $\eta = \Theta(\sqrt{\log |\mX|}/\sqrt{T \gamma_T})$, and $\alpha = 2\eta (2 B \gamma_T + \beta \sqrt{\gamma_T})$.  
    Assume $|\mX| = o(\exp(T \gamma_T^{-1}))$\footnote{The condition $|\mX| = o(\exp(T \gamma_T^{-1}))$ is required to guarantee $\eta < \overline{\eta}$, where $\overline{\eta}$ is defined in \thmref{thm:reg_exp3}. Note that, if $|\mX| = \Omega(\exp(T \gamma_T^{-1}))$, the upper bound $O(\sqrt{T \gamma_T \log |\mX|})$ becomes meaningless since $\sqrt{T \gamma_T \log |\mX|} = \Omega(T)$.}. Furthermore, suppose that $B = \Theta(1)$ and $d = \Theta(1)$ hold. Then, we have $\bar{R}_T = O(\sqrt{T \gamma_T \log |\mX|})$.
\end{corollary}

\paragraph{Proof overview of \thmref{thm:reg_exp3}.} The basic proof strategy follows from that of Exp3. Compared with existing analyses in adversarial LB, the primary technical challenge is how to address the high dimensionality of the feature map $\psi(\bx)$. Specifically, by following the standard Exp3 analysis, we can observe that the quadratic form $\|\psi(\bx)\|_{G_t(\lambda)^{-1}}^2$ is the central quantity for governing the analysis. For example, the second moment $\sum_{\bx \in \mX} P_t(\bx) \hat{f}_t^2(\bx)$ and the supremum of the estimator, which are crucial quantities for Exp3 analysis~\citep[e.g., see Chaps.~11 and 27 in ][]{lattimore2020bandit}, are dominated by this quadratic form.
Without regularization $\lambda = 0$, in general, the worst-case value of this quantity depends on the dimensionality of the feature map, i.e.,~$\|\psi(\bx)\|_{G_t(0)^{-1}}^2 = \Theta(|\mX|)$. However, under the regularization, we can observe a simple yet useful connection among the quadratic form, the posterior variance of the GP, and the MIG. We introduce the following lemma, which is used in the various parts of our analysis to handle the quadratic form $\|\psi(\bx)\|_{G_t(\lambda)^{-1}}^2$ via MIG.

\begin{lemma}[Quadratic form of feature map, posterior variance, and MIG]
\label{lem:est_var}
    Fix any $T \in \N_+$ and $\lambda > 0$. Fix any finite input domain $\mX$, any kernel $k: \mX \times \mX \rightarrow \R$, and any probability mass function $P: \mX \rightarrow [0, 1]$ on $\mX$. Let $\bx^{(1)}, \ldots, \bx^{(T)}, \bx^{(T+1)} \sim P$ be an i.i.d. input sequence drawn from $P$. Furthermore, let $\psi(\cdot): \mX \rightarrow \R^{|\mX|}$ be any feature map of $k$, i.e., $k(\bx, \bx') = \psi(\bx)^{\top} \psi(\bx')$ for any $\bx, \bx' \in \mX$.
    Then, for any $\bz \in \mX$, the inequality: $\|\psi(\bz)\|_{G(T, \lambda, P)^{-1}}^2 \leq \frac{T}{\lambda}\Ep_{\bX^{(T)}} \sbr{\sigma^2 (\bz; \bX^{(T)}, \lambda)}$ holds,
    where $G(T, \lambda, P) = \sum_{\bx \in \mX} P(\bx) \psi(\bx) \psi(\bx)^{\top} + \frac{\lambda}{T} \bI_{|\mX|}$ and 
    $\bX^{(T)} = (\bx^{(1)}, \ldots, \bx^{(T)})$. Furthermore, if $\forall \bx \in \mX,~k(\bx, \bx) \leq 1$, we have
    \begin{align}
        \label{eq:iid_z_error}
        \Ep_{\bx^{(T+1)}}\sbr{\|\psi(\bx^{(T+1)})\|_{G(T, \lambda, P)^{-1}}^2}
        \leq \frac{T}{\lambda}\Ep_{\bX^{(T)}, \bx^{(T+1)}} \sbr{\sigma^2 (\bx^{(T+1)}; \bX^{(T)}, \lambda)} \leq \frac{2 \gamma_T}{\lambda \log(1 + \lambda^{-1})}.
    \end{align}
    In addition, if $\lambda \geq 1$, the rightmost quantity in Eq.~\eqref{eq:iid_z_error} is bounded from above as $\frac{2 \gamma_T}{\lambda \log(1 + \lambda^{-1})} \leq 4\gamma_T$.
\end{lemma}

\lemref{lem:est_var} is derived by combining the direct consequence of matrix Jensen's inequality with information gain inequality in Eq.~\eqref{eq:sigma_2_cum_ub} from \citep{srinivas10gaussian}. We give the full proof in \appref{app:est_var_proof}. 
Our proof for \thmref{thm:reg_exp3} is obtained by carefully combining \lemref{lem:est_var} with the existing Exp3 analysis and the existing MVR analysis, in e.g., \citep{cai2021lenient,vakili2021optimal}.

\begin{remark}[Comparison with \citep{camilleri2021high,hong2023optimization}]
Results similar to those in \lemref{lem:est_var} have already been reported in \citep{camilleri2021high,hong2023optimization} using the generalized notion of MIG. Specifically, Lemma C.4 in \citep{hong2023optimization} claims $\Ep_{\bx^{(T+1)}}[\|\psi(\bx^{(T+1)})\|_{G(T, \lambda, P)^{-1}}^2] = O( \gamma_{\psi, T})$, where $\gamma_{\psi, T} = \max_{P \in \mP_{\mX}} \log \det(T\lambda^{-1}\sum_{\bx \in \mX} P(\bx) \psi(\bx) \psi(\bx)^{\top} + \bI_{|\mX|})$, and $\mP_{\mX}$ is the set of all distributions over $\mX$. Here, $\gamma_{\psi, T}$ is greater than or equal to the original MIG: $\gamma_{T} \leq \gamma_{\psi, T}$ from the definition~\citep{hong2023optimization}, and the converse inequality $\gamma_{\psi, T} = O(\gamma_T)$ is provided by the concurrent work~\citep{zhang2026nearoptimalregretadversarialkernel}.
\citet{zhang2026nearoptimalregretadversarialkernel} provide the proof for the inequality $\gamma_{\psi, T} \leq 4 \gamma_T$ when $\lambda \geq 1$. See Eq.~(16) in the proof of Lemma~8 in \citep{zhang2026nearoptimalregretadversarialkernel} for details. The combination of this inequality with the existing results by \citet{camilleri2021high,hong2023optimization} also leads to the same guarantees as our Eq.~\eqref{eq:iid_z_error} when $\lambda \geq 1$.
\end{remark}

\paragraph{Extension to a continuous domain.}
The extension to a continuous domain\footnote{For an exponentially large finite input domain $|\mX| = \Omega(\exp(T))$, we can also obtain a tighter regret than that in \corref{coro:ke3_reg_tuned} by the discretization argument.} can be directly obtained by relying on the discretization argument used in the stochastic KB literature~\citep[e.g.,][]{chowdhury2017kernelized,li2022gaussian}. The only additional requirement is the following uniform Lipschitz condition of RKHS.
\begin{assumption}[Uniform Lipschitz condition]
    \label{asmp:uniform_lip}
    The function class $\mF_k(B) = \{f \in \mH_k \mid \|f\|_k \leq B\}$, from which the environment chooses $(f_t)$, satisfies $\forall f \in \mF_k(B),~\forall \bx, \bx' \in \mX,~|f(\bx) - f(\bx')| \leq L(B) \|\bx - \bx'\|_{2}$
    for some constant $L(B) \in (0, \infty)$, which may depend on $B$.
\end{assumption}
The above Lipschitz condition is relatively mild. For example, SE and Mat\'ern kernels with $\nu > 1$ satisfy this condition with $L(B) = CB$ for some constant $C > 0$\footnote{This constant $C > 0$ may depend on $d$, $\ell$, and $\nu$.}~\citep[e.g., Proposition 1 and Remark 5 in ][]{shekhar2020multi}.
By \asmpref{asmp:uniform_lip}, we can ignore the regret incurred from a sufficiently fine discretization. Specifically, we use a discretized input set $\tilde{\mX} \subset \mX \subset [0, 1]^d$ as a $\Theta(1/\sqrt{T})$-net of $\mX$ with respect to $\|\cdot\|_{2}$, as with \citep{li2022gaussian}\footnote{$\epsilon$-net $\tilde{\mX}$ of $\mX$ with respect to the norm $\|\cdot\|_{2}$ is a subset of $\mX$ such that $\forall \bx \in \mX,~\exists \bz \in \tilde{\mX}, \|\bx - \bz\|_{2} \leq \epsilon$ holds.}. Then, the existence of $\tilde{\mX}$ with $|\tilde{\mX}| = O(T^{\frac{d}{2}})$ is guaranteed~(e.g., Corollary 4.2.13 in \citep{vershynin2018high}), and the regret incurred from the cumulative discretization error is at most $O(\sqrt{T})$, which does not affect the growth rate of regret. Thus, we obtain the following corollary of \thmref{thm:reg_exp3}.

\begin{corollary}[Regret upper bound for kernelized Exp3 with general compact input domain]
    \label{eq:continuous_ke3_regret}
    Suppose that Assumptions~\ref{asmp:regularity} and \ref{asmp:uniform_lip} hold. Assume $\mX \subset [0, 1]^d$. Furthermore, let $\tilde{\mX} \subset \mX$ be a $\Theta(1/\sqrt{T})$-net of $\mX$.  
    Then, when running \algoref{alg:exp3} on $\tilde{\mX}$ with $\beta = B\sqrt{\lambda/T}$, $\eta = \Theta(\sqrt{\log T}/\sqrt{T \gamma_T})$, and $\alpha = 2\eta (2 B \gamma_T + \beta \sqrt{\gamma_T})$, we have $\bar{R}_T = O(\sqrt{T \gamma_T \log T})$.
\end{corollary}

\corref{eq:continuous_ke3_regret} shows that kernelized Exp3 achieves an $\tilde{O}(\sqrt{T \gamma_T})$ regret, which is identical to the nearly-optimal regret obtained in the stochastic KB literature~\citep{li2022gaussian,salgia2021domain,valko2013finite}. By substituting the explicit upper bounds of the MIG, we can also obtain kernel-specific regret results. For the SE and $\nu$-Mat\'ern kernels, $\gamma_T = O((\log T)^{d+1}(\log \log T)^{-d})$~\citep{iwazaki2026tighter} and $\gamma_T = \tilde{O}(T^{\frac{d}{2\nu+d}})$~\citep{iwazaki2025improved,vakili2021information}, respectively.
Thus, for SE and $\nu$-Mat\'ern kernels, we have $\bar{R}_T = O(\sqrt{T (\log T)^{d+2}(\log \log T)^{-d}})$ and $\bar{R}_T = \tilde{O}(T^{\frac{\nu+d}{2\nu+d}})$, respectively. In addition, for the linear kernel, we have $\bar{R}_T = O(\sqrt{T} \log(T))$ since $\gamma_T = O(\log T)$~\citep{srinivas10gaussian}. This result matches the regret of Exp3 in the adversarial LB up to a logarithmic factor~\citep{bartlett2008high,lattimore2020bandit,zimmert2022return}.

\paragraph{Computational cost.} The primary drawback of kernelized Exp3 is its prohibitive computational cost, stemming from the manipulation of the $|\mX|$-dimensional feature map.
To obtain the MVR sequence used in the exploration distribution, we require $O(|\mX| \sum_{t=1}^{\lceil T\alpha \rceil} t^2) = O(|\mX| (T\alpha)^3)$ computation via rank-one update of the matrix inverse. With the parameter settings in \corref{coro:ke3_reg_tuned}, this is $\tilde{O}(|\mX| (T\gamma_T)^{3/2})$ for fixed $\lambda$. Regarding the per-round computational cost, we need to calculate the inverse of the $|\mX| \times |\mX|$ matrix $G_t(\lambda)$, which requires $O(|\mX|^3)$ computation\footnote{Note that the difference $G_t(\lambda) - G_{t-1}(\lambda)$ is generally full-rank since $P_{t}(\bx)$ is updated from $P_{t-1}(\bx)$ for all $\bx \in \mX$; thus, we cannot update $G_t(\lambda)^{-1}$ efficiently by relying on standard low-rank update procedures.}. Since the KB problem mainly focuses on the regime where $|\mX|$ is extremely large (e.g., as in \corref{eq:continuous_ke3_regret}, we must consider the case $|\mX| = \Theta(T^{d/2})$ for a continuous domain), $O(|\mX|^3)$ computation per round is computationally infeasible. We tackle this efficiency issue in the next section.

\section{Exponential-weight algorithm with RLS-based Nystr\"om approximation}
\label{sec:exp3_nystrom_approx}
Our second algorithm is a computationally efficient variant of kernelized Exp3, called RLS-kernelized Exp3. Due to space limitations, the pseudocode is provided in \algoref{alg:exp3_nystrom_approx} in Appendix~\ref{app:pseudocode}. 
The huge computational cost in kernelized Exp3 stems from calculating the inverse of the $|\mX| \times |\mX|$-matrix $G_t(\lambda)$. We mitigate this burden via a low-rank approximation of the feature map weighted by the distribution $P_t$. Specifically, the resulting approximation is built on a Nystr\"om approximation for the \emph{weighted} kernel function $\tilde{k}_t(\bx, \bx') \coloneqq \sqrt{P_t(\bx) P_t(\bx')}k(\bx, \bx')$\footnote{For readers unfamiliar with the Nystr\"om approximation of kernels, we refer to the related KB or kernel approximation literature \citep{calandriello2018statistical,calandriello2017distributed,calandriello2019gaussian,chowdhury2021no,musco2017recursive}.}.

At each round $t$, let $\bS_t \in \R^{|\mX| \times s_t}$ be a matrix with exactly one strictly positive entry in each column and zeros elsewhere. Here, $s_t \in [|\mX|]$ is determined within the algorithm as described later. Note that we can use $\bS_t$ to extract the $s_t$ weighted columns of some matrix $\bm{M} \in \R^{|\mX| \times |\mX|}$ by multiplying from the right as $\bm{M}\bS_t$. Thus, as with \citep{musco2017recursive}, we call $\bS_t$ a \emph{sampling matrix}.
We also define the \emph{weighted} feature matrix $\bar{\bPsi}_t = (\sqrt{P_t(\bx^{(1)})} \psi(\bx^{(1)}), \ldots, \sqrt{P_t(\bx^{(|\mX|)})} \psi(\bx^{(|\mX|)})) \in \R^{|\mX| \times |\mX|}$, where we arbitrarily index the elements of $\mX$ as $\mX = \{\bx^{(i)}\}_{i\in[|\mX|]}$. Note that $G_t(\lambda) = \bar{\bPsi}_t \bar{\bPsi}_t^{\top} + \frac{\lambda}{T} \bI_{|\mX|}$ holds from the definition. Furthermore, note that $\bar{\bPsi}_t$ corresponds to the feature matrix of the weighted kernel function $\tilde{k}_t(\bx, \bx')$.
In our algorithm, we approximate $G_t(\lambda)$ by projecting $\bar{\bPsi}_t$ into the column span of $\bar{\bPsi}_t \bS_t$. To be more specific, let us define the orthogonal projection matrix $\bQ_t \coloneqq (\bar{\bPsi}_t \bS_t) [(\bar{\bPsi}_t \bS_t)^{\top} (\bar{\bPsi}_t \bS_t)]^{\dagger}(\bar{\bPsi}_t \bS_t)^{\top} \in \R^{|\mX| \times |\mX|}$, where $\bm{M}^{\dagger}$ represents the pseudo-inverse of matrix $\bm{M}$. Then, RLS-kernelized Exp3 is constructed by replacing quantities $\hat{f}_t$ and $\|\psi(\bx)\|_{G_t(\lambda)^{-1}}$ in kernelized Exp3 with their low-rank approximated versions: $\tilde{f}_t(\bx) \coloneqq \psi(\bx)^{\top} \tilde{G}_t(\lambda)^{-1} (\bQ_t \psi(\bx_t)) f_t(\bx_t)$ and $\tilde{\sigma}_t(\bx) \coloneqq \|\psi(\bx)\|_{\tilde{G}_t(\lambda)^{-1}}$, respectively, where $\tilde{G}_t(\lambda) = (\bQ_t \bar{\bPsi}_t) (\bQ_t \bar{\bPsi}_t)^{\top} + \frac{\lambda}{T} \bI_{|\mX|}$. After elementary matrix calculations, we can obtain the following identities, which can be efficiently calculated~(See Lemma~\ref{lem:identity_nystrom} for the detailed derivation):
\begin{align}
    \label{eq:tilde_f_nystrom}
    \tilde{f}_t(\bx) &= \frac{f_t(\bx_t)}{\sqrt{P_t(\bx)P_t(\bx_t)}}\phi_t(\bx)^{\top} \rbr{\sum_{\bx' \in \mX} \phi_t(\bx') \phi_t(\bx')^{\top} + \frac{\lambda}{T} \bI_{s_t}}^{-1} \phi_t(\bx_t), \\
    \label{eq:tilde_sigma_nystrom}
    \begin{split}
    \tilde{\sigma}_t^2(\bx) &=
    \frac{T}{\lambda P_t(\bx)} \rbr{\tilde{k}_t(\bx, \bx) - \phi_t(\bx)^{\top} \phi_t(\bx)} \\
    &\quad + \frac{1}{P_t(\bx)} \phi_t(\bx)^{\top} \rbr{\sum_{\bx' \in \mX} \phi_t(\bx') \phi_t(\bx')^{\top} + \frac{\lambda}{T} \bI_{s_t}}^{-1} \phi_t(\bx),
    \end{split}
\end{align}
where $\phi_t(\bx) = [(\bS_t^{\top} \tilde{\bK}_t \bS_t)^{1/2}]^{\dagger} \bS_t^{\top} \tilde{\bk}_t(\bx) \in \R^{s_t}$ with the weighted kernel matrix $\tilde{\bK}_t \coloneqq [\tilde{k}_t(\bx^{(i)}, \bx^{(j)})]_{i,j \in [|\mX|]} \in \R^{|\mX| \times |\mX|}$ and the weighted kernel vector $\tilde{\bk}_t(\bx) \coloneqq [\tilde{k}_t(\bx, \bx^{(i)})]_{i \in [|\mX|]} \in \R^{|\mX|}$.
The quantity $\phi_t(\bx)$ is referred to as the \emph{Nystr\"om embedding}~\citep{calandriello2019gaussian,chowdhury2021no} associated with the weighted kernel function $\tilde{k}_t$. From the above expressions, we can easily confirm that the calculation of $[\tilde{f}_t(\bx)]_{\bx \in \mX}$ and $[\tilde{\sigma}_t^2(\bx)]_{\bx \in \mX}$ requires only $O(|\mX| s_t^2 + s_t^3)$-computation (see the first paragraph in Appendix~\ref{app:proof_random_exp3} for details). This drastically improves on the $O(|\mX|^3)$-computational cost of the original kernelized Exp3 when $s_t \ll |\mX|$.

\paragraph{Selection of the sampling matrix.}
A remaining question is how to define the sampling matrix $\bS_t \in \R^{|\mX| \times s_t}$ such that the effect of the approximation does not severely degrade the algorithm performance, while maintaining $s_t \ll |\mX|$. To do so, we leverage the RLS-based Nystr\"om approximation methods~\citep{calandriello2017distributed,musco2017recursive}, in which we can guarantee both the approximation accuracy of the kernel and small $s_t$. Specifically, we adopt the recursive RLS-Nystr\"om algorithm proposed in \citep{musco2017recursive} (\algoref{algo:recursive_RLS}) as a subroutine to obtain $\bS_t$~(Line~7 in \algoref{alg:exp3_nystrom_approx}). 

\paragraph{Regret and computational cost for RLS-kernelized Exp3.} 
By combining \lemref{lem:est_var} with the existing guarantees in \citep{musco2017recursive}, we observe 
that the recursive RLS-Nystr\"om subroutine attains sufficient approximation accuracy to ensure $\tilde{O}(\sqrt{T \gamma_T})$ regret with $s_t = \tilde{O}(\gamma_T)$ and requires only $\tilde{O}(|\mX|\gamma_T^2)$-computation with high probability. As a result, we can obtain the following \thmref{thm:reg_exp3_nystrom_approx}, which formally describes the regret and computational cost guarantees of RLS-kernelized Exp3.
\begin{theorem}[Regret upper bound and computational cost for RLS-kernelized Exp3]
    \label{thm:reg_exp3_nystrom_approx}
    Suppose that Assumption~\ref{asmp:regularity} holds. Assume that $\mX$ is finite and $|\mX| = o(\exp(T\gamma_T^{-1}))$ holds. Furthermore, suppose that $d$ and $B$ are $\Theta(1)$. Set algorithm parameters as $\beta = B(1 + \sqrt{2})\sqrt{\lambda/T}$, $\lambda \geq 1$, $\eta = \Theta(\sqrt{\log |\mX|}/\sqrt{T \gamma_T})$, $\alpha =  2 \eta\sqrt{3} (2(\sqrt{3} + \sqrt{2}) B \gamma_{T} + \beta \sqrt{\gamma_T})$, and $\delta = 1/T^2$. Then, when running \algoref{alg:exp3_nystrom_approx}, we have $\bar{R}_T = O(\sqrt{T \gamma_T \log |\mX|})$.
    Furthermore, with probability at least $1 - 1/T^2$, for any $t \in [T]$, the computational cost at round $t$ is $\tilde{O}(|\mX|\gamma_T^2 + \gamma_T^3)$.
\end{theorem}
The full proof is provided in \appref{app:proof_random_exp3}.
The above theorem confirms that RLS-kernelized Exp3 is nearly optimal while offering improved computational complexity. Compared with the regret guarantees of the original kernelized Exp3~(\corref{coro:ke3_reg_tuned}), the price of the improved computational efficiency is only a constant factor in the regret upper bound, which is hidden in the implied constant of $\bar{R}_T = O(\sqrt{T \gamma_T \log |\mX|})$.

\paragraph{Proof sketch.} We first consider a high-probability good event where the recursive RLS-Nyström subroutine accurately approximates the true weighted feature matrix (Eq.~\eqref{eq:thm_rke3_proof_eps_acc} in the appendix, derived from the existing RLS-Nystr\"om guarantees in Lemma~\ref{lem:acc_comp_rls}). Under this event, we derive some useful spectral properties (Lemmas~\ref{lem:Ppsi_Qpsi} and \ref{lem:ort_proj_rel}) to control the approximation errors of both the inverse projected matrix $\tilde{G}_t(\lambda)^{-1}$ and the projected feature map $\bQ_t \bar{\bm{\Psi}}_t$. Leveraging these properties within the standard Exp3 analysis framework demonstrates that the additional regrets due to approximation errors are negligible up to a constant factor, thereby yielding the same regret rate as in Theorem~\ref{thm:reg_exp3}. Finally, by crudely bounding the regret under the complementary bad event, we confirm that setting the confidence parameter to $\delta = 1/T^2$ is sufficient to achieve near-optimal regret.

\paragraph{Remarks on Nystr\"om approximation.}
Although we leverage the recursive RLS-Nystr\"om algorithm in \citep{musco2017recursive} as a subroutine of our algorithm, other approximation methods may be applicable as long as (i) they have the spectral guarantees of the projection matrix~(Eq.~\eqref{eq:epsilon_acc_dict} in the appendix), and (ii) their computational cost guarantees are quantified by \emph{effective dimension} of the weighted kernel $\tilde{k}_t$ (see the proof of \lemref{lem:acc_comp_rls}). However, to our knowledge, no existing method provides both strictly better computational efficiency and tighter regret bounds than those in \thmref{thm:reg_exp3_nystrom_approx}.

\paragraph{Remarks on exploration distribution.}
Note that an $O(|\mX| (T\alpha)^3)$ pre-processing cost for the preparation of the exploration distribution is required even for RLS-kernelized Exp3 (see the last paragraph in \secref{sec:kernelized_exp3}). If this overhead is unacceptable, one can approximate the MVR-sequence itself via recursive RLS-Nystr\"om, which alleviates $O(|\mX| (T\alpha)^3)$ computation while maintaining the nearly-optimal guarantees of regret. This variant of RLS-kernelized Exp3 is given in Appendix~\ref{app:rke3_approx_mvr}. Another approach is to replace the MVR-based distribution with a simpler one, e.g., the uniform distribution. We discuss such possible choices for the exploration distribution in Appendix~\ref{app:exp_dist_discuss}; however, based on the known theoretical tools for KB, we are not aware of any choice that strictly reduces the computational cost of MVR-based exploration distributions without imposing further restrictive assumptions on the kernel.

Finally, although this paper mainly focuses on the theoretical side, the simulation experiments for supporting our superior guarantees are provided in \appref{app:sim_experiment}.

\section{Lower bounds}
\label{sec:lower_bound}
The following \thmref{thm:lb} provides the lower bounds for adversarial KB under $k = \sek$ or $k = \matk$.
\begin{theorem}[Lower bounds for adversarial KB]
    \label{thm:lb}
    Fix $\mX = [0, 1]^d$. Suppose that $k = \sek$ or $k = \matk$ holds. Suppose that $\nu, \ell, d$, and $B$ are $\Theta(1)$.
    Then, for any learner's algorithm, there exists a strategy of the environment that satisfies \asmpref{asmp:regularity} and
    \begin{equation}
        \label{eq:lb_result}
        \bar{R}_T = 
        \begin{cases}
            \Omega\rbr{\sqrt{T (\log T)^{\frac{d}{2}-1}}}~~&\mathrm{if}~~k = \sek, \\
            \Omega\rbr{T^{\frac{\nu + d}{2\nu + d}} (\log T)^{-\frac{\nu}{2\nu + d}}}~~&\mathrm{if}~~k = \matk.
        \end{cases}
    \end{equation}
\end{theorem}
The full proof is in \appref{app:proof_lb}.
From Theorem~\ref{thm:lb}, we confirm that kernelized Exp3 and RLS-kernelized Exp3 are optimal up to polylogarithmic factors for the SE and Mat\'ern kernels with $\nu > 1$. Compared with the lower bounds for the stochastic setting~\citep{scarlett2017lower}, there are additional multiplicative logarithmic factors $(\log T)^{-1/2}$ and $(\log T)^{-\frac{\nu}{2\nu + d}}$ in the lower bounds for the SE and $\nu$-Mat\'ern kernels, respectively. As clarified in the proof sketch below, these logarithmic factors arise because our proof reduces the adversarial problem to a stochastic KB lower bound under Gaussian noise.

\paragraph{Proof sketch.} We consider an environment that reduces the lower bound for adversarial KB to that for stochastic KB by setting: $f_t(\cdot) = f(\cdot) + \tilde{\eta}_t k(\bm{0}, \cdot)$, where $\tilde{\eta}_t \sim \mathcal{N}(0, \sigma^2)$ for some variance $\sigma^2 > 0$, and $\|f\|_k \leq B/2$. Through a slight modification of the lower bounds for stochastic KB~\citep{cai2021lenient,scarlett2017lower}, we can show that $\sup_{\bx\in \mX} \Ep[\sum_{t=1}^T f_t(\bx) - f_t(\bx_t)] = \Omega(\underline{R}(T, \sigma^2))$, where $\underline{R}(T, \sigma^2) = \sqrt{\sigma^{2} T(\log (T/\sigma^2))^{d/2}}$ and $\underline{R}(T, \sigma^2) = \sigma^{\frac{2\nu}{2\nu+d}} T^{\frac{\nu+d}{2\nu+d}}$ for $k = \sek$ and $k = \matk$, respectively.
While this constitutes a lower bound on the expected regret, the above problem instance does not satisfy \asmpref{asmp:regularity} since $\|f_t\|_k \leq B$ may be violated due to the unboundedness of the Gaussian noise $\tilde{\eta}_t$. Therefore, the only lower bound we can rigorously obtain from the above problem instance is the looser lower bound: $\Omega(\underline{R}(T, \sigma^2)) - \sup_{\bx\in \mX} \Ep[\1\{\exists t \in [T], \|f_t\|_k > B\}(\sum_{t=1}^T f_t(\bx) - f_t(\bx_t))]$, where the contribution of the event $\{\exists t \in [T], \|f_t\|_k > B\}$ in the original lower bound is explicitly eliminated (rigorous argument is in \appref{app:proof_lb}). To ensure the second term remains negligible, it is sufficient to set $\sigma^2$ as $\sigma^2 = \Theta((\log T)^{-1})$ based on the tail properties of the Gaussian noise. Then, the lower bounds $\Omega(\underline{R}(T, \Theta((\log T)^{-1})))$ match Eq.~\eqref{eq:lb_result}.

\section{Discussions}
\label{sec:discuss}
In this section, we discuss both the limitations of our work and potential avenues for future research.

\paragraph{Algorithm without discretization.} For a continuous input domain, our algorithms require reducing the original problem to a finite input domain via discretization. Such discretization is commonly used in the stochastic KB analysis~\citep{iwazaki2025improvedregretanalysisgaussian,li2022gaussian}; however, a limitation of this approach is that the algorithm becomes computationally infeasible as the dimension $d$ increases. This drawback motivates the development of algorithms that operate directly on continuous domains. For example, the kernelized extension of continuous exponential weight and follow-the-regularized leader in the adversarial LB setting~\citep[e.g., Chaps.~27 and 28 in][]{lattimore2020bandit} may be a promising direction for future research.

\paragraph{Refining lower bounds.} Our lower bound has the undesirable dependence on the logarithmic factors. As described in \secref{sec:lower_bound}, the source of this worse dependence is that we use the lower bound for stochastic KB under the Gaussian noise model in the proof. Thus, this gap could be addressed in future research by investigating lower bounds for other noise models, such as bounded noise.

\paragraph{Kernelized variants for further advanced topics.} Finally, it is interesting to study the other kernelized variants of further advanced topics in the adversarial bandit literature, such as first-order bounds~\citep{abernethy2012interior,allenberg2006hannan}, variation-bounded regret~\citep{bubeck2018sparsity,hazan2011simple},
and best-of-both-worlds algorithms~\citep{bubeck2012best,lee2021achieving}. We believe that the theoretical tools leveraged in Secs.~\ref{sec:kernelized_exp3} and \ref{sec:exp3_nystrom_approx} are helpful for these future directions.

\section{Conclusion}
We investigated a nearly-optimal algorithm for adversarial KB. We showed that the kernelized version of the Exp3 algorithm achieves nearly-optimal regret by bridging the gap between the existing information-gain-based arguments for GP and the analysis of Exp3. Furthermore, we proposed a computationally efficient algorithm that leverages RLS-based Nystr\"om approximation. As noted in the discussions in the previous section, we believe that this research opens several avenues for future research on the adversarial KB problem. 

\section*{Acknowledgment}
We appreciate the helpful comments from Jonathan Scarlett and Yu-Jie Zhang on revising the paper.

\bibliography{main}

@InProceedings{cai2021on,
  title = 	 {On Lower Bounds for Standard and Robust {G}aussian Process Bandit Optimization},
  author =       {Cai, Xu and Scarlett, Jonathan},
  booktitle = 	 {Proceedings of the 38th International Conference on Machine Learning},
  pages = 	 {1216--1226},
  year = 	 {2021},
  volume = 	 {139},
  series = 	 {Proceedings of Machine Learning Research},
  publisher =    {PMLR}
}

@article{russo2014learning,
  title={Learning to optimize via information-directed sampling},
  author={Russo, Daniel and Van Roy, Benjamin},
  journal={Advances in neural information processing systems},
  volume={27},
  year={2014}
}

@inproceedings{salgia2021domain,
 author = {Salgia, Sudeep and Vakili, Sattar and Zhao, Qing},
 booktitle = {Advances in Neural Information Processing Systems},
 pages = {28836--28847},
 publisher = {Curran Associates, Inc.},
 title = {A Domain-Shrinking based {B}ayesian Optimization Algorithm with Order-Optimal Regret Performance},
 volume = {34},
 year = {2021}
}

@inproceedings{valko2013finite,
author = {Valko, Michal and Korda, Nathan and Munos, R\'{e}mi and Flaounas, Ilias and Cristianini, Nello},
title = {Finite-time analysis of kernelised contextual bandits},
year = {2013},
publisher = {AUAI Press},
booktitle = {Proceedings of the Twenty-Ninth Conference on Uncertainty in Artificial Intelligence},
pages = {654–663},
numpages = {10},
series = {UAI'13}
}

@article{Russo2014-learning,
  title={Learning to optimize via posterior sampling},
  author={Russo, Daniel and Van Roy, Benjamin},
  journal={Mathematics of Operations Research},
  volume={39},
  number={4},
  pages={1221--1243},
  year={2014},
  publisher={INFORMS}
}

@InProceedings{takeno2023-randomized,
  title = 	 {Randomized {G}aussian Process Upper Confidence Bound with Tighter {B}ayesian Regret Bounds},
  author =       {Takeno, Shion and Inatsu, Yu and Karasuyama, Masayuki},
  booktitle = 	 {Proceedings of the 40th International Conference on Machine Learning},
  pages = 	 {33490--33515},
  year = 	 {2023},
  volume = 	 {202},
  series = 	 {Proceedings of Machine Learning Research},
  publisher =    {PMLR}
}

@article{riutort2023practical,
  title={Practical Hilbert space approximate {B}ayesian {G}aussian processes for probabilistic programming},
  author={Riutort-Mayol, Gabriel and B{\"u}rkner, Paul-Christian and Andersen, Michael R and Solin, Arno and Vehtari, Aki},
  journal={Statistics and Computing},
  volume={33},
  number={1},
  pages={17},
  year={2023},
  publisher={Springer}
}

@inproceedings{cai2021lenient,
  title={Lenient regret and good-action identification in {G}aussian process bandits},
  author={Cai, Xu and Gomes, Selwyn and Scarlett, Jonathan},
  booktitle={International Conference on Machine Learning},
  pages={1183--1192},
  year={2021},
  organization={PMLR}
}

@Book{Rasmussen2005-Gaussian,
 author = {Rasmussen, Carl Edward and Williams, Christopher K. I.},
 title = {{G}aussian Processes for Machine Learning (Adaptive Computation and Machine Learning)},
 year = {2005},
 publisher = {The MIT Press}
}

@inproceedings{lizotte2007automatic,
  title={Automatic Gait Optimization With {G}aussian Process Regression},
  author={Lizotte, Daniel and Wang, Tao and Bowling, Michael and Schuurmans, Dale},
  booktitle={Proc. International Joint Conference on Artificial Intelligence (IJCAI)},
  year={2007}
}

@inproceedings{srinivas10gaussian,
	author = {Niranjan Srinivas and Andreas Krause and Sham Kakade and Matthias Seeger},
	booktitle = {Proc. International Conference on Machine Learning (ICML)},
	title = {Gaussian Process Optimization in the Bandit Setting: No Regret and Experimental Design},
	year = {2010}}

@article{bull2011convergence,
  title={Convergence rates of efficient global optimization algorithms.},
  author={Bull, Adam D},
  journal={Journal of Machine Learning Research},
  year={2011}
}

@inproceedings{krause2011contextual,
  title={Contextual {G}aussian process bandit optimization},
  author={Krause, Andreas and Ong, Cheng},
  booktitle={Proc. Neural Information Processing Systems (NeurIPS)},
  year={2011}
}

@inproceedings{snoek2012practical,
 author = {Snoek, Jasper and Larochelle, Hugo and Adams, Ryan P},
 booktitle = {Proc. Neural Information Processing Systems (NeurIPS)},
 title = {Practical {B}ayesian Optimization of Machine Learning Algorithms},
 year = {2012}
}

@article{desautels2014parallelizing,
  author  = {Thomas Desautels and Andreas Krause and Joel W. Burdick},
  title   = {Parallelizing Exploration-Exploitation Tradeoffs in {G}aussian Process Bandit Optimization},
  journal = {Journal of Machine Learning Research},
  year    = {2014},
}

@article{zuluaga2016pal,
  title={e-pal: An active learning approach to the multi-objective optimization problem},
  author={Zuluaga, Marcela and Krause, Andreas and others},
  journal={Journal of Machine Learning Research},
  year={2016}
}

@inproceedings{scarlett2017lower,
  title={Lower bounds on regret for noisy {G}aussian process bandit optimization},
  author={Scarlett, Jonathan and Bogunovic, Ilija and Cevher, Volkan},
  booktitle={Proc. Conference on Learning Theory (COLT)},
  year={2017},
}

@inproceedings{chowdhury2017kernelized,
  title={On kernelized multi-armed bandits},
  author={Sayak Ray Chowdhury and Aditya Gopalan},
  booktitle={Proc. International Conference on Machine Learning (ICML)},
  year={2017}
}

@inproceedings{bogunovic2018adversarially,
  title={Adversarially robust optimization with {G}aussian processes},
  author={Ilija Bogunovic and Jonathan Scarlett and Stefanie Jegelka and Volkan Cevher},
  booktitle={Proc. Neural Information Processing Systems (NeurIPS)},
  year={2018}
}

@book{vershynin2018high,
  title={High-dimensional probability: {A}n introduction with applications in data science},
  author={Vershynin, Roman},
  year={2018},
  publisher={Cambridge university press}
}

@inproceedings{kirschner2020distributionally,
	author = {Johannes Kirschner and Ilija Bogunovic and Stefanie Jegelka and Andreas Krause},
	booktitle = {Proc. International Conference on Artificial Intelligence and Statistics (AISTATS)},
	title = {Distributionally Robust {B}ayesian Optimization},
	year = {2020}}

@article{shekhar2020multi,
  title={Multi-scale zero-order optimization of smooth functions in an {RKHS}},
  author={Shekhar, Shubhanshu and Javidi, Tara},
  journal={arXiv preprint arXiv:2005.04832},
  year={2020}
}

@inproceedings{camilleri2021high,
  title={High-dimensional experimental design and kernel bandits},
  author={Camilleri, Romain and Jamieson, Kevin and Katz-Samuels, Julian},
  booktitle={Proc. International Conference on Machine Learning (ICML)},
  year={2021},
}

@inproceedings{vakili2021optimal,
  title={Optimal order simple regret for {G}aussian process bandits},
  author={Sattar Vakili and Nacime Bouziani and Sepehr Jalali and Alberto Bernacchia and Da-shan Shiu},
  booktitle={Proc. Neural Information Processing Systems (NeurIPS)},
  year={2021}
}

@inproceedings{iwazaki2021mean,
  title={Mean-variance analysis in {B}ayesian optimization under uncertainty},
  author={Shogo Iwazaki and Yu Inatsu and Ichiro Takeuchi},
  booktitle={Proc. International Conference on Artificial Intelligence and Statistics (AISTATS)},
  year={2021},
}

@inproceedings{nguyen2021value,
  title={Value-at-risk optimization with {G}aussian processes},
  author={Nguyen, Quoc Phong and Dai, Zhongxiang and Low, Bryan Kian Hsiang and Jaillet, Patrick},
  booktitle={Proc. International Conference on Machine Learning (ICML)},
  year={2021},
}

@inproceedings{vakili2021information,
  title={On information gain and regret bounds in {G}aussian process bandits},
  author={Vakili, Sattar and Khezeli, Kia and Picheny, Victor},
  booktitle={Proc. International Conference on Artificial Intelligence and Statistics (AISTATS)},
  year={2021},
}

@inproceedings{zhou2021no,
  title={No-regret algorithms for time-varying {B}ayesian optimization},
  author={Zhou, Xingyu and Shroff, Ness},
  booktitle={2021 55th Annual Conference on Information Sciences and Systems (CISS)},
  year={2021},
  organization={IEEE}
}

@inproceedings{li2022gaussian,
  title={{G}aussian process bandit optimization with few batches},
  author={Li, Zihan and Scarlett, Jonathan},
  booktitle={Proc. International Conference on Artificial Intelligence and Statistics (AISTATS)},
  year={2022}
}

@inproceedings{bogunovic2022robust,
  title={A robust phased elimination algorithm for corruption-tolerant {G}aussian process bandits},
  author={Bogunovic, Ilija and Li, Zihan and Krause, Andreas and Scarlett, Jonathan},
  booktitle={Proc. Neural Information Processing Systems (NeurIPS)},
  year={2022}
}

@inproceedings{deng2022weighted,
  title={Weighted {G}aussian process bandits for non-stationary environments},
  author={Deng, Yuntian and Zhou, Xingyu and Kim, Baekjin and Tewari, Ambuj and Gupta, Abhishek and Shroff, Ness},
  booktitle={Proc. International Conference on Artificial Intelligence and Statistics (AISTATS)},
  year={2022},
}

@inproceedings{vakili2022open,
  title={Open problem: {R}egret bounds for noise-free kernel-based bandits},
  author={Vakili, Sattar},
  booktitle={Proc. Conference on Learning Theory (COLT)},
  year={2022},
}

@inproceedings{hong2023optimization,
  title={An optimization-based algorithm for non-stationary kernel bandits without prior knowledge},
  author={Hong, Kihyuk and Li, Yuhang and Tewari, Ambuj},
  booktitle={Proc. International Conference on Artificial Intelligence and Statistics (AISTATS)},
  year={2023},
}

@inproceedings{salgiarandom,
  title={Random Exploration in {B}ayesian Optimization: {O}rder-Optimal Regret and Computational Efficiency},
  author={Salgia, Sudeep and Vakili, Sattar and Zhao, Qing},
  booktitle={Proc. International Conference on Machine Learning (ICML)},
  year={2024}
}

@inproceedings{iwazaki2025improvedregretanalysisgaussian,
  title={Improved Regret Analysis in {G}aussian Process Bandits: Optimality for Noiseless Reward, {RKHS} norm, and Non-Stationary Variance},
  author={Iwazaki, Shogo and Takeno, Shion},
  booktitle={International Conference on Machine Learning},
  pages={26642--26672},
  year={2025},
  organization={PMLR}
}

@InProceedings{scarlett2018tight,
  title = 	 {Tight Regret Bounds for {B}ayesian Optimization in One Dimension},
  author =       {Scarlett, Jonathan},
  booktitle = 	 {Proceedings of the 35th International Conference on Machine Learning},
  pages = 	 {4500--4508},
  year = 	 {2018},
  volume = 	 {80},
  series = 	 {Proceedings of Machine Learning Research},
  publisher =    {PMLR}
}

@inproceedings{shekhar2022instance,
  title={Instance dependent regret analysis of kernelized bandits},
  author={Shekhar, Shubhanshu and Javidi, Tara},
  booktitle={International Conference on Machine Learning},
  year={2022},
}

@phdthesis{janz2022sequential,
  title={Sequential decision making with feature-linear models},
  author={Janz, David},
  year={2022}
}

@article{solin2020hilbert,
  title={Hilbert space methods for reduced-rank Gaussian process regression},
  author={Solin, Arno and S{\"a}rkk{\"a}, Simo},
  journal={Statistics and Computing},
  year={2020},
}

@inproceedings{iwazaki2025gaussian,
  title={Gaussian Process Upper Confidence Bound Achieves Nearly-Optimal Regret in Noise-Free {G}aussian Process Bandits},
  author={Iwazaki, Shogo},
  booktitle={The Thirty-ninth Annual Conference on Neural Information Processing Systems},
  year={2025}
}

@inproceedings{iwazaki2025improved,
  title={Improved Regret Bounds for {G}aussian Process Upper Confidence Bound in {B}ayesian Optimization},
  author={Iwazaki, Shogo},
  booktitle={The Thirty-ninth Annual Conference on Neural Information Processing Systems},
  year={2025}
}

@inproceedings{iwazaki2025near,
  title={Near-Optimal Algorithm for Non-Stationary Kernelized Bandits},
  author={Iwazaki, Shogo and Takeno, Shion},
  booktitle={International Conference on Artificial Intelligence and Statistics},
  pages={406--414},
  year={2025},
  organization={PMLR}
}

@inproceedings{cai2025lower,
  title={Lower Bounds for Time-Varying Kernelized Bandits},
  author={Cai, Xu and Scarlett, Jonathan},
  booktitle={International Conference on Artificial Intelligence and Statistics},
  pages={73--81},
  year={2025},
  organization={PMLR}
}

@article{kanagawa2018gaussian,
  title={Gaussian processes and kernel methods: A review on connections and equivalences},
  author={Kanagawa, Motonobu and Hennig, Philipp and Sejdinovic, Dino and Sriperumbudur, Bharath K},
  journal={arXiv preprint arXiv:1807.02582},
  year={2018}
}

@article{musco2017recursive,
  title={Recursive sampling for the {N}ystr\"om method},
  author={Musco, Cameron and Musco, Christopher},
  journal={Advances in neural information processing systems},
  volume={30},
  year={2017}
}

@inproceedings{takemori2021approximation,
  title={Approximation theory based methods for {RKHS} bandits},
  author={Takemori, Sho and Sato, Masahiro},
  booktitle={International Conference on Machine Learning},
  year={2021},
}

@article{yadav2024gaussian,
  title={Gaussian process bandits for top-k recommendations},
  author={Yadav, Mohit and Sheldon, Daniel and Musco, Cameron},
  journal={Advances in Neural Information Processing Systems},
  year={2024}
}

@inproceedings{chatterji2019online,
  title={Online learning with kernel losses},
  author={Chatterji, Niladri and Pacchiano, Aldo and Bartlett, Peter},
  booktitle={International Conference on Machine Learning},
  year={2019}
}

@inproceedings{han2022adversarial,
  title={Adversarial attacks on {G}aussian process bandits},
  author={Han, Eric and Scarlett, Jonathan},
  booktitle={International Conference on Machine Learning},
  year={2022},
}

@article{iwazaki2026tighter,
  title={Tighter Regret Lower Bound for {G}aussian Process Bandits with Squared Exponential Kernel in Hypersphere},
  author={Iwazaki, Shogo},
  journal={arXiv preprint arXiv:2602.17940},
  year={2026}
}

@book{lattimore2020bandit,
  title={Bandit algorithms},
  author={Lattimore, Tor and Szepesv{\'a}ri, Csaba},
  year={2020},
  publisher={Cambridge University Press}
}

@inproceedings{calandriello2019gaussian,
  title={Gaussian process optimization with adaptive sketching: Scalable and no regret},
  author={Calandriello, Daniele and Carratino, Luigi and Lazaric, Alessandro and Valko, Michal and Rosasco, Lorenzo},
  booktitle={Conference on learning theory},
  year={2019},
}

@article{calandriello2018statistical,
  title={Statistical and computational trade-offs in kernel k-means},
  author={Calandriello, Daniele and Rosasco, Lorenzo},
  journal={Advances in neural information processing systems},
  volume={31},
  year={2018}
}

@inproceedings{neu2024adversarial,
  title={Adversarial contextual bandits go kernelized},
  author={Neu, Gergely and Olkhovskaya, Julia and Vakili, Sattar},
  booktitle={International Conference on Algorithmic Learning Theory},
  year={2024},
}

@article{takeno2026regret,
  title={On Regret Bounds of {T}hompson Sampling for {B}ayesian Optimization},
  author={Takeno, Shion and Iwazaki, Shogo},
  journal={arXiv preprint arXiv:2603.09276},
  year={2026}
}

@inproceedings{chowdhury2021no,
  title={No-regret algorithms for multi-task {B}ayesian optimization},
  author={Chowdhury, Sayak Ray and Gopalan, Aditya},
  booktitle={International Conference on Artificial Intelligence and Statistics},
  year={2021},
}

@inproceedings{bogunovic2020corruption,
  title={Corruption-tolerant {G}aussian process bandit optimization},
  author={Bogunovic, Ilija and Krause, Andreas and Scarlett, Jonathan},
  booktitle={International Conference on Artificial Intelligence and Statistics},
  year={2020},
}

@article{saday2023robust,
  title={Robust {B}ayesian satisficing},
  author={Saday, Artun and Y{\i}ld{\i}r{\i}m, Y Cahit and Tekin, Cem},
  journal={Advances in Neural Information Processing Systems},
  year={2023}
}

@article{sessa2019no,
  title={No-regret learning in unknown games with correlated payoffs},
  author={Sessa, Pier Giuseppe and Bogunovic, Ilija and Kamgarpour, Maryam and Krause, Andreas},
  journal={Advances in Neural Information Processing Systems},
  volume={32},
  year={2019}
}

@article{sessa2020learning,
  title={Learning to play sequential games versus unknown opponents},
  author={Sessa, Pier Giuseppe and Bogunovic, Ilija and Kamgarpour, Maryam and Krause, Andreas},
  journal={Advances in neural information processing systems},
  volume={33},
  pages={8971--8981},
  year={2020}
}

@inproceedings{shamir2015complexity,
  title={On the complexity of bandit linear optimization},
  author={Shamir, Ohad},
  booktitle={Conference on Learning Theory},
  year={2015},
}

@inproceedings{zimmert2022return,
  title={Return of the bias: Almost minimax optimal high probability bounds for adversarial linear bandits},
  author={Zimmert, Julian and Lattimore, Tor},
  booktitle={Conference on Learning Theory},
  year={2022},
}

@inproceedings{bartlett2008high,
  title={High-probability regret bounds for bandit online linear optimization},
  author={Bartlett, Peter and Dani, Varsha and Hayes, Thomas and Kakade, Sham and Rakhlin, Alexander and Tewari, Ambuj},
  booktitle={Proceedings of the 21st annual conference on learning theory-COLT 2008},
  year={2008},
}

@inproceedings{calandriello2017distributed,
  title={Distributed adaptive sampling for kernel matrix approximation},
  author={Calandriello, Daniele and Lazaric, Alessandro and Valko, Michal},
  booktitle={Artificial Intelligence and Statistics},
  year={2017},
}

@inproceedings{lee2021achieving,
  title={Achieving near instance-optimality and minimax-optimality in stochastic and adversarial linear bandits simultaneously},
  author={Lee, Chung-Wei and Luo, Haipeng and Wei, Chen-Yu and Zhang, Mengxiao and Zhang, Xiaojin},
  booktitle={International Conference on Machine Learning},
  year={2021},
}

@inproceedings{bubeck2012best,
  title={The best of both worlds: Stochastic and adversarial bandits},
  author={Bubeck, S{\'e}bastien and Slivkins, Aleksandrs},
  booktitle={Conference on Learning Theory},
  pages={42--1},
  year={2012},
  organization={JMLR Workshop and Conference Proceedings}
}

@inproceedings{hazan2011simple,
  title={A simple multi-armed bandit algorithm with optimal variation-bounded regret},
  author={Hazan, Elad and Kale, Satyen},
  booktitle={Proceedings of the 24th Annual Conference on Learning Theory},
  pages={817--820},
  year={2011},
  organization={JMLR Workshop and Conference Proceedings}
}

@inproceedings{bubeck2018sparsity,
  title={Sparsity, variance and curvature in multi-armed bandits},
  author={Bubeck, S{\'e}bastien and Cohen, Michael and Li, Yuanzhi},
  booktitle={Algorithmic Learning Theory},
  pages={111--127},
  year={2018},
  organization={PMLR}
}

@inproceedings{allenberg2006hannan,
  title={Hannan consistency in on-line learning in case of unbounded losses under partial monitoring},
  author={Allenberg, Chamy and Auer, Peter and Gy{\"o}rfi, L{\'a}szl{\'o} and Ottucs{\'a}k, Gy{\"o}rgy},
  booktitle={International Conference on Algorithmic Learning Theory},
  pages={229--243},
  year={2006},
  organization={Springer}
}

@article{abernethy2012interior,
  title={Interior-point methods for full-information and bandit online learning},
  author={Abernethy, Jacob D and Hazan, Elad and Rakhlin, Alexander},
  journal={IEEE Transactions on Information Theory},
  volume={58},
  number={7},
  pages={4164--4175},
  year={2012},
  publisher={IEEE}
}

@book{fedorov2013theory,
  title={Theory of optimal experiments},
  author={Fedorov, Valerii Vadimovich},
  year={2013},
  publisher={Elsevier}
}

@inproceedings{mason2022nearly,
  title={Nearly Optimal Algorithms for Level Set Estimation},
  author={Mason, Blake and Jain, Lalit and Mukherjee, Subhojyoti and Camilleri, Romain and Jamieson, Kevin and Nowak, Robert},
  booktitle={International Conference on Artificial Intelligence and Statistics},
  year={2022},
}

@misc{zhang2026nearoptimalregretadversarialkernel,
      title={Near-Optimal Regret in Adversarial Kernel Bandits}, 
      author={Yu-Jie Zhang and Hao Qiu and Jonathan Scarlett and Kevin Jamieson},
      year={2026},
      eprint={2605.26585},
      archivePrefix={arXiv},
      primaryClass={cs.LG},
      url={https://arxiv.org/abs/2605.26585}, 
}
\bibliographystyle{abbrvnat}

\newpage
\appendix

\onecolumn

\section{Notation table}
\label{app:notation}
\tabref{tab:notation} summarizes the notations used in this paper.
\begin{table}
    \centering
    \caption{Notation table.}
    \resizebox{\linewidth}{!}{%
    \begin{tabular}{cl}
    \toprule
        Symbol & Definition/Description \\ \midrule
        $\N_+$ & Set of natural numbers without $0$: $\N_+ \coloneqq \{1,~2, \ldots\}$. \\
        $[t]$ & Set of elements of $\N_+$ up to $t \in \N_+$: $[t] \coloneqq \{1, \ldots, t\}$. \\
        $\|\bx\|_{\bm{M}}$ & Mahalanobis seminorm with respect to a positive semi-definite matrix $\bm{M}$: $\|\bx\|_{\bm{M}} \coloneqq \sqrt{\bx^{\top} \bm{M}\bx}$. \\
        $\|\bx\|_2$ & L2-norm of $\bx \in \R^m$: $\|\bx\|_2 \coloneqq \sqrt{\sum_{i=1}^m x_i^2}$. \\
        $\|\bm{M}\|$ & Spectral norm of a matrix $\bm{M} \in \R^{n\times m}$: $\|\bm{M}\| \coloneqq \sup_{\bx \in \R^m \setminus \{\bm{0}\}} \|\bm{M}\bx\|_2/\|\bx\|_2$. \\
        $\preceq$ & Loewner order, i.e., $\bm{M}_1 \preceq \bm{M}_2$ means $\bm{M}_2 - \bm{M}_1$ is a positive semi-definite matrix. \\
        $\bm{M}^{\dagger}$ & Pseudo-inverse of matrix $\bm{M}$. \\
        $\tilde{O}(\cdot)$ & $g_1(T) = \tilde{O}(g_2(T))$ denotes $g_1(T) = O(g_2(T)(\log T)^c)$ for an absolute constant $c \geq 0$. \\
        $\tilde{\Omega}(\cdot)$ & $g_1(T) = \tilde{\Omega}(g_2(T))$ denotes $g_1(T) = \Omega(g_2(T)(\log T)^{-c})$ for an absolute constant $c \geq 0$. \\
        $\psi(\cdot)$ & Feature map of $k$ on a finite input domain $\mX$, i.e., $\psi: \mX \rightarrow \R^{|\mX|}$ satisfies \\
        & $k(\bx, \bx') = \psi(\bx)^{\top} \psi(\bx')$ for all $\bx, \bx' \in \mX$. \\
        $\bPsi$ & Feature matrix on finite input domain $\mX \coloneqq \{\bx^{(1)}, \ldots, \bx^{(|\mX|)}\}$: \\ 
        & $\bPsi \coloneqq (\psi(\bx^{(1)}), \ldots, \psi(\bx^{(|\mX|)})) \in \R^{|\mX| \times |\mX|}$. \\
        $\bm{\theta}_t$ & Feature representation of $f_t$, i.e., $\bm{\theta}_t \in \R^{|\mX|}$ satisfies $f_t(\cdot) = \psi(\cdot)^{\top} \bm{\theta}_t$ (see \lemref{lem:feature_map_rkhs}). \\
        $\mH_{t}$ & Learner's history of kernelized Exp3 up to round $t$: $\mH_{t} \coloneqq \{(\bx_i, f_i(\bx_i))\}_{i \in [t]}$. \\
        $\tilde{\mH}_{t}$ & Learner's history of RLS-kernelized Exp3 up to round $t$: $\tilde{\mH}_{t} \coloneqq \{(\bx_i, \bS_i, f_i(\bx_i))\}_{i \in [t]}$. \\
        $\pi(\cdot)$ & Exploration distribution $\pi(\bx) \coloneqq \frac{1}{\lceil T\alpha \rceil} \sum_{t=1}^{\lceil T\alpha \rceil} \1\{\bx = \bx_t^{(\mathrm{MVR})}\}$. $(\bx_t^{(\mathrm{MVR})})$ is generated \\ 
        & by \algoref{algo:mvr}. \\
        $\tilde{\pi}(\cdot)$ & Approximation of $\pi$: $\tilde{\pi}(\bx) \coloneqq \frac{1}{\lceil T\alpha \rceil} \sum_{t=1}^{\lceil T\alpha \rceil} \1\{\bx = \tilde{\bx}_t^{(\mathrm{MVR})}\}$. $(\tilde{\bx}_t^{(\mathrm{MVR})})$ is generated by \\ 
        & \algoref{algo:approx_mvr}. \\
        $\tilde{P}_t(\cdot)$ & Exponential-weight distribution: $\tilde{P}_t(\cdot) \coloneqq \exp(\eta \sum_{i=1}^{t-1} u_i(\cdot))/\sum_{\bx' \in \mX}\exp(\eta \sum_{i=1}^{t-1} u_i(\bx'))$ \\
        & or $\tilde{P}_t(\cdot) \coloneqq \exp(\eta \sum_{i=1}^{t-1} \tilde{u}_i(\cdot))/\sum_{\bx' \in \mX}\exp(\eta \sum_{i=1}^{t-1} \tilde{u}_i(\bx'))$. \\
        $P_t(\cdot)$ & Sampling distribution of the learner at round $t$: $P_t(\cdot) \coloneqq \alpha \pi(\cdot) + (1 - \alpha) \tilde{P}_t(\cdot)$. \\
        $\gamma_T$ & Maximum information gain up to $T$-inputs: $\gamma_T \coloneqq \sup_{\bX \in \mX^T}\frac{1}{2} \log \det(\bI_T + \lambda^{-1} \bK(\bX, \bX))$. \\
        $s_t$ & Rank determined in recursive RLS-Nystr\"om at round $t$ in RLS-kernelized Exp3. \\
        $\bS_t$ & $(|\mX|\times s_t)$-sampling matrix at round $t$ of RLS-kernelized Exp3. \\
        $\bQ_t$ & Orthogonal projection matrix into the subspace spanned by column vectors of $\bar{\bPsi}_t \bS_t$, \\
        & i.e., $\bQ_t \coloneqq (\bar{\bPsi}_t\bS_t)[(\bar{\bPsi}_t \bS_t)^{\top}(\bar{\bPsi}_t \bS_t)]^{\dagger} (\bar{\bPsi}_t \bS_t)^{\top}$. \\
        $\tilde{k}_t$ & Weighted kernel at round $t$ of RLS-kernelized Exp3: $\tilde{k}_t(\bx, \bx') \coloneqq \sqrt{P_t(\bx) P_t(\bx')}k(\bx, \bx')$. \\
        $\bar{\psi}_t(\cdot)$ & Weighted feature map $\bar{\psi}_t(\bx) \coloneqq \sqrt{P_t(\bx)} \psi(\bx)$. \\
        $\bar{\bPsi}_t$ & Weighted feature matrix on finite input domain $\mX \coloneqq \{\bx^{(1)}, \ldots, \bx^{(|\mX|)}\}$: \\
        & $\bar{\bPsi}_t \coloneqq (\bar{\psi}_t(\bx^{(1)}), \ldots, \bar{\psi}_t(\bx^{(|\mX|)})) \in \R^{|\mX| \times |\mX|}$. \\
        $\tilde{\psi}_t(\cdot)$ & Projected feature map: $\tilde{\psi}_t(\bx) \coloneqq \bQ_t \psi(\bx)$. \\
        $G_t(\lambda)$ & Inverse weight matrix for estimator: $G_t(\lambda) \coloneqq \bar{\bPsi}_t \bar{\bPsi}_t^{\top} + \frac{\lambda}{T} \bI_{|\mX|}$. \\
        $\tilde{G}_t(\lambda)$ & Approximation of $G_t$: $\tilde{G}_t(\lambda) \coloneqq (\bQ_t \bar{\bPsi}_t) (\bQ_t\bar{\bPsi}_t)^{\top} + \frac{\lambda}{T} \bI_{|\mX|}$. \\
        $\hat{f}_t(\cdot)$ & Reward estimator for kernelized Exp3: $\hat{f}_t(\bx) \coloneqq \psi(\bx)^{\top} G_t(\lambda)^{-1} \psi(\bx_t) f_t(\bx_t)$. \\
        $\tilde{f}_t(\cdot)$ & Approximated reward estimator for RLS-kernelized Exp3: $\tilde{f}_t(\bx) \coloneqq \psi(\bx)^{\top} \tilde{G}_t(\lambda)^{-1} \tilde{\psi}_t(\bx_t) f_t(\bx_t)$. \\
        $u_t(\cdot)$ & Optimistic estimator for kernelized Exp3: $u_t(\cdot) \coloneqq \hat{f}_t(\cdot) + \beta \|\psi(\cdot)\|_{G_t(\lambda)^{-1}}$. \\
        $\tilde{u}_t(\cdot)$ & Approximated optimistic estimator for RLS-kernelized Exp3: $\tilde{u}_t(\cdot) \coloneqq \tilde{f}_t(\cdot) + \beta \|\psi(\cdot)\|_{\tilde{G}_t(\lambda)^{-1}}$. \\
        $\sigma^2(\cdot; \bX, \lambda)$ & Posterior variance of GP: $\sigma^2(\cdot; \bX, \lambda) \coloneqq k(\cdot, \cdot) - \bk(\cdot, \bX)^{\top}(\bK(\bX, \bX) + \lambda \bI)^{-1} \bk(\cdot, \bX)$. \\
        $\tilde{\sigma}_t^2(\cdot)$ & Approximation of $\|\psi(\cdot)\|_{G_t(\lambda)^{-1}}^2$: $\tilde{\sigma}_t^2(\cdot) \coloneqq \|\psi(\cdot)\|_{\tilde{G}_t(\lambda)^{-1}}^2$. \\
        $A_t$ & Event: $\frac{1}{2}\rbr{\bar{\bPsi}_t \bar{\bPsi}_t^{\top} + \frac{\lambda}{T} \bI_{|\mX|}} \preceq \bar{\bPsi}_t \bS_t \bS_t^{\top} \bar{\bPsi}_t^{\top} + \frac{\lambda}{T} \bI_{|\mX|} \preceq \frac{3}{2}\rbr{\bar{\bPsi}_t \bar{\bPsi}_t^{\top} + \frac{\lambda}{T} \bI_{|\mX|}}$ (see \lemref{lem:acc_comp_rls}). \\
        $A$ & $A \coloneqq \cap_{t \in [T]} A_t$. \\
        \bottomrule
    \end{tabular}%
    }
    \label{tab:notation}
\end{table}

\section{Pseudocode}
\label{app:pseudocode}
In this section, we provide the pseudocode for the algorithms that could not be included in the main text due to space limitations. 
\algoref{algo:mvr} is the pseudocode of the MVR algorithm leveraged in the preparation of the exploration distribution $\pi$ in our algorithm. Algs.~\ref{algo:recursive_RLS} and \ref{alg:exp3_nystrom_approx} are the pseudocode of the recursive RLS-Nystr\"om algorithm and the RLS-kernelized Exp3 algorithm described in \secref{sec:exp3_nystrom_approx}, respectively. 
Note that \algoref{algo:recursive_RLS} is identical to the original Alg.~2 in \citep{musco2017recursive} except for some changes of notations.
Furthermore, \algoref{alg:exp3_nystrom_approx_mvr} represents the pseudocode of the variant of RLS-kernelized Exp3, which uses an approximated MVR sequence for the exploration distribution as described in the paragraph ``Remarks on exploration distribution'' of \secref{sec:exp3_nystrom_approx} and in \appref{app:rke3_approx_mvr}. 
Finally, \algoref{algo:approx_mvr} is the pseudocode of the algorithm for generating an approximated MVR sequence leveraged in \algoref{alg:exp3_nystrom_approx_mvr}.

\begin{algorithm}
    \caption{Maximum variance reduction (MVR)}
    \label{algo:mvr}
    \begin{algorithmic}[1]
        \item[] \textbf{input:} Input domain $\mathcal{X}$, kernel function $k: \mathcal{X} \times \mathcal{X} \to \mathbb{R}$, regularization parameter $\lambda > 0$, total budget $M$.
        \item[] \textbf{output:} MVR-sequence $(\bx_t^{(\text{MVR})})_{t \in [M]}$.
        \STATE $\bX_0^{(\text{MVR})} \leftarrow ()$.
        \FOR {$t = 1, 2, \ldots, M$}
            \STATE $\bx_t^{(\mathrm{MVR})} \in \argmax_{\bx \in \mX} \sigma^2\rbr{\bx; \bX_{t-1}^{(\mathrm{MVR})}, \lambda}$.
            \STATE Update training input: $\bX_{t}^{(\mathrm{MVR})} \leftarrow  (\bx_1^{(\mathrm{MVR})}, \ldots, \bx_t^{(\mathrm{MVR})})$.
        \ENDFOR
        \RETURN $(\bx_t^{(\text{MVR})})_{t \in [M]}$.
    \end{algorithmic}
\end{algorithm}

\begin{algorithm}
    \caption{RecursiveRLS-Nystr\"{o}m ($\mX$, $k$, $\lambda$, $\delta$)~\citep{musco2017recursive}}
    \label{algo:recursive_RLS}
    \begin{algorithmic}[1]
        \item[] \textbf{input:} Finite input set $\mathcal{X} \coloneqq \{\bx^{(1)}, \ldots, \bx^{(|\mX|)}\}$, kernel function $k: \mathcal{X} \times \mathcal{X} \to \mathbb{R}$, regularization parameter $\lambda > 0$, confidence level parameter $\delta \in (0, 1/32)$
        \item[] \textbf{output:} weighted sampling matrix $\mathbf{S} \in \mathbb{R}^{|\mX| \times s}$
        \IF{$|\mX| \leq 192 \log(1/\delta)$} 
            \RETURN $\mathbf{S} \coloneqq \mathbf{I}_{|\mX| \times |\mX|}$.
        \ENDIF
        \STATE Let $\bar{\mS}$ be a random subset of $\{1, \dots, |\mX|\}$, with each $i \in \{1, \dots, |\mX|\}$ included independently with probability $\frac{1}{2}$.
        \STATE \quad $\triangleright$ Let $\bar{\mX} = \{\bx^{(i_1)}, \bx^{(i_2)}, \dots, \bx^{(i_{|\bar{\mS}|})}\}$ for $i_j \in \bar{\mathcal{S}}$ be the data sample corresponding to $\bar{\mathcal{S}}$.
        \STATE \quad $\triangleright$ Let $\bar{\bS} = [\mathbf{e}_{i_1}, \mathbf{e}_{i_2}, \dots, \mathbf{e}_{i_{|\bar{\mathcal{S}}|}}] \in \R^{|\mX| \times |\bar{\mS}|}$ be the sampling matrix corresponding to $\bar{\mathcal{S}}$.
        \STATE \quad \quad $\triangleright$ $\mathbf{e}_{i} \in \R^{|\mX|}$ denotes the unit vector whose $i$-th element is only $1$.
        \STATE $\tilde{\bS} \leftarrow \text{RecursiveRLS-Nystr\"{o}m} (\bar{\mX}, k, \lambda, \delta/3)$. 
        \STATE $\hat{\bS} \coloneqq \bar{\mathbf{S}} \tilde{\mathbf{S}}$.
        \STATE Set $\tilde{l}_i^{\lambda} \coloneqq \frac{3}{2\lambda} \left( \mathbf{K} - \mathbf{K}\hat{\mathbf{S}} \left( \hat{\mathbf{S}}^T \mathbf{K} \hat{\mathbf{S}} + \lambda \mathbf{I} \right)^{-1} \hat{\mathbf{S}}^T \mathbf{K} \right)_{i,i}$ for each $i \in \{1, \dots, |\mX|\}$.
        \STATE Set $p_i \coloneqq \min \{1, \tilde{l}_i^{\lambda} \cdot 16 \log(\sum \tilde{l}_i^{\lambda} / \delta) \}$ for each $i \in \{1, \dots, |\mX|\}$.
        \STATE Initially set weighted sampling matrix $\bS$ to be empty. For each $i \in \{1, \dots, |\mX|\}$, with probability $p_i$, append the column $\frac{1}{\sqrt{p_i}} \mathbf{e}_i$ onto $\bS$.
        \RETURN $\bS$.
    \end{algorithmic}
\end{algorithm}

\begin{algorithm}
    \caption{Approximated-MVR($\mX$, $k$, $\lambda$, $M$, $\delta$)}
    \label{algo:approx_mvr}
    \begin{algorithmic}[1]
        \item[] \textbf{input:} Input domain $\mathcal{X}$, kernel function $k: \mathcal{X} \times \mathcal{X} \to \mathbb{R}$, regularization parameter $\lambda > 0$, total budget $M \in \N_+$, confidence level $\delta \in (0, \min \{1, M/32\})$.
        \item[] \textbf{output:} approximated MVR-sequence $(\tilde{\bx}_t^{(\text{MVR})})_{t \in [M]}$.
        \STATE $\tilde{\bX}_0^{(\text{MVR})} \leftarrow \emptyset$.
        \FOR {$t = 1, 2, \ldots, M$}
            \STATE Calculate sampling matrix: $\tilde{\bS}_t \leftarrow \text{RecursiveRLS-Nystr\"om}(\tilde{\bX}_{t-1}^{(\text{MVR})}, k, \lambda, \delta/M)$
            \STATE Calculate approximated posterior variance $\bar{\sigma}^2\rbr{\bx; \tilde{\bX}_{t-1}^{(\mathrm{MVR})}, \tilde{\bS}_t, \lambda}$:
            \begin{equation}
            \label{eq:nystrom_approx_mvr_expression}
            \begin{split}
                \bar{\sigma}^2\rbr{\bx; \tilde{\bX}_{t-1}^{(\mathrm{MVR})}, \tilde{\bS}_t, \lambda} &= 
    k(\bx, \bx) - \bar{\phi}_t(\bx)^{\top} \bar{\phi}_t(\bx) \\
        &~~~~+ \lambda \bar{\phi}_t(\bx)^{\top} \rbr{\sum_{i=1}^{t-1} \bar{\phi}_t(\tilde{\bx}_i^{(\mathrm{MVR})}) \bar{\phi}_t(\tilde{\bx}_i^{(\mathrm{MVR})})^{\top} + \lambda \bI_{\tilde{s}_t}}^{-1} \bar{\phi}_t(\bx).
            \end{split}
            \end{equation}
            \STATE \quad $\triangleright$ $\bar{\phi}_t(\bx) \coloneqq [(\tilde{\bS}_t^{\top} \bK_t \tilde{\bS}_t)^{1/2}]^{\dagger} \tilde{\bS}_t^{\top} \bk_t(\bx) \in \R^{\tilde{s}_t}$ with $\tilde{\bS}_t \in \R^{(t-1) \times \tilde{s}_t}$.
            \STATE \quad $\triangleright$ $\bK_t \coloneqq \bK\rbr{\tilde{\bX}_{t-1}^{(\mathrm{MVR})}, \tilde{\bX}_{t-1}^{(\mathrm{MVR})}} \in \R^{(t-1) \times (t-1)}$, $\bk_t(\bx) \coloneqq \bk\rbr{\bx, \tilde{\bX}_{t-1}^{(\mathrm{MVR})}} \in \R^{t-1}$.
            \STATE Calculate $\tilde{\bx}_t^{(\mathrm{MVR})} \in \argmax_{\bx \in \mX} \bar{\sigma}^2\rbr{\bx; \tilde{\bX}_{t-1}^{(\mathrm{MVR})}, \tilde{\bS}_t, \lambda}$.
            \STATE Update training input: $\tilde{\bX}_{t}^{(\mathrm{MVR})} \leftarrow (\tilde{\bx}_1^{(\text{MVR})},\ldots,\tilde{\bx}_{t-1}^{(\text{MVR})},\tilde{\bx}_t^{(\text{MVR})})$.
        \ENDFOR
        \RETURN $(\tilde{\bx}_t^{(\text{MVR})})_{t \in [M]}$.
    \end{algorithmic}
\end{algorithm}

\begin{algorithm}[t!]
    \caption{RLS-kernelized Exp3 for adversarial kernelized bandits}
    \label{alg:exp3_nystrom_approx}
    \begin{algorithmic}[1]
        \REQUIRE Finite input domain $\mX$, kernel $k: \mX \times \mX \rightarrow \R$, learning rate $\eta > 0$, regularization parameter $\lambda \geq 1$, mixing ratio $\alpha \in (0, 1)$, confidence width parameter $\beta > 0$, total budget $T \in \N_+$, RKHS norm upper bound $B \in (0, \infty)$, confidence level parameter $\delta \in (0, \min\{1, 3T/32\})$.
        \STATE Prepare the exploration distribution $\pi(\bx)$ by using MVR-input sequence $(\bx_t^{(\mathrm{MVR})})_{t \leq \lceil T \alpha \rceil}$:
        \begin{equation}
            \pi(\bx) = \frac{1}{\lceil T \alpha \rceil} \sum_{t=1}^{\lceil T \alpha \rceil} \1\cbr{\bx_t^{(\mathrm{MVR)}} = \bx},
        \end{equation}
        where $\bx_t^{(\mathrm{MVR})} \in \argmax_{\bx \in \mX} \sigma^2\rbr{\bx; \bX_{t-1}^{(\mathrm{MVR})}, \lambda}$ and $\bX_{t-1}^{(\mathrm{MVR})} = \left(\bx_1^{(\mathrm{MVR})}, \ldots, \bx_{t-1}^{(\mathrm{MVR})}\right)$.
        \STATE Initialize sampling distribution: $P_1(\bx) \leftarrow \alpha \pi(\bx) + (1 - \alpha) \tilde{P}_1(\bx)$, where $\tilde{P}_1(\bx) = \frac{1}{|\mX|}$.
        \FOR {$t = 1, 2, \ldots, T$}
            \STATE Draw $\bm{x}_{t} \sim P_t$.
            \STATE Receive reward $f_t(\bx_t)$.
            \STATE Construct weighted kernel function $\tilde{k}_t(\bx, \bx') = \sqrt{P_t(\bx) P_t(\bx')} k(\bx, \bx')$.
            \STATE Calculate the sampling matrix $\bm{S}_t$ by running \algoref{algo:recursive_RLS}: 
            $\bm{S}_t \leftarrow \text{RecursiveRLS-Nystr\"om}(\mX, \tilde{k}_t, \frac{\lambda}{T}, \frac{\delta}{3T})$.
            \STATE Calculate the reward estimator $\tilde{f}_t(\cdot)$ as follows:
            \begin{equation}
            \tilde{f}_t(\bx) = \frac{f_t(\bx_t)}{\sqrt{P_t(\bx)P_t(\bx_t)}}\phi_t(\bx)^{\top} \rbr{\sum_{\bx' \in \mX} \phi_t(\bx') \phi_t(\bx')^{\top} + \frac{\lambda}{T} \bI_{s_t}}^{-1} \phi_t(\bx_t).
            \end{equation}
            \STATE Calculate optimistic estimate $\tilde{u}_t$ as $\tilde{u}_t(\bx) = \tilde{f}_t(\bx) + \beta \tilde{\sigma}_t(\bx)$, where
            \begin{equation}
            \begin{split}
                \tilde{\sigma}_t^2(\bx) &= 
    \frac{T}{\lambda P_t(\bx)} \rbr{\tilde{k}_t(\bx, \bx) - \phi_t(\bx)^{\top} \phi_t(\bx)} \\
        &~~~~+ \frac{1}{P_t(\bx)} \phi_t(\bx)^{\top} \rbr{\sum_{\bx' \in \mX} \phi_t(\bx') \phi_t(\bx')^{\top} + \frac{\lambda}{T} \bI_{s_t}}^{-1} \phi_t(\bx).
            \end{split}
            \end{equation}
            \STATE Update the sampling distribution:
            \begin{equation}
                P_{t+1}(\bx) = \alpha \pi(\bx) + (1 - \alpha) \tilde{P}_{t+1}(\bx),~\text{where}~\tilde{P}_{t+1}(\bx) = \frac{\exp\rbr{\eta \sum_{i=1}^t \tilde{u}_i(\bx)}}{\sum_{\bx' \in \mX} \exp\rbr{\eta \sum_{i=1}^t \tilde{u}_i(\bx')}}.
            \end{equation}
        \ENDFOR
    \end{algorithmic}
\end{algorithm}

\begin{algorithm}[t!]
    \caption{RLS-kernelized Exp3 for adversarial kernelized bandits with approximated MVR}
    \label{alg:exp3_nystrom_approx_mvr}
    \begin{algorithmic}[1]
        \REQUIRE Finite input domain $\mX$, kernel $k: \mX \times \mX \rightarrow \R$, learning rate $\eta > 0$, regularization parameter $\lambda \geq 1$, mixing ratio $\alpha \in (0, 1)$, confidence width parameter $\beta > 0$, total budget $T \in \N_+$, RKHS norm upper bound $B \in (0, \infty)$, confidence level parameter $\delta \in (0, \min\{1, 3 \lceil T \alpha \rceil/16\})$.
        \STATE Prepare the exploration distribution $\tilde{\pi}(\bx)$ by using approximated MVR-sequence as follows:
        \begin{equation}
            \tilde{\pi}(\bx) = \frac{1}{\lceil T \alpha \rceil} \sum_{t=1}^{\lceil T \alpha \rceil} \1\cbr{\tilde{\bx}_t^{(\mathrm{MVR)}} = \bx},
        \end{equation}
        where $(\tilde{\bx}_t^{(\mathrm{MVR})})_{t \in [\lceil T \alpha \rceil]} \leftarrow \text{Approximated-MVR}(\mX, k, \lambda, \lceil T \alpha \rceil, \delta/6)$ (\algoref{algo:approx_mvr}).
        \STATE Initialize sampling distribution: $P_1(\bx) \leftarrow \alpha \tilde{\pi}(\bx) + (1 - \alpha) \tilde{P}_1(\bx)$, where $\tilde{P}_1(\bx) = \frac{1}{|\mX|}$.
        \FOR {$t = 1, 2, \ldots, T$}
            \STATE Draw $\bm{x}_{t} \sim P_t$.
            \STATE Receive reward $f_t(\bx_t)$.
            \STATE Construct weighted kernel function $\tilde{k}_t(\bx, \bx') = \sqrt{P_t(\bx) P_t(\bx')} k(\bx, \bx')$.
            \STATE Calculate the sampling matrix $\bm{S}_t$ by running \algoref{algo:recursive_RLS}: 
            $\bm{S}_t \leftarrow \text{RecursiveRLS-Nystr\"om}(\mX, \tilde{k}_t, \frac{\lambda}{T}, \frac{\delta}{6T})$.
            \STATE Calculate the reward estimator $\tilde{f}_t(\cdot)$ as follows:
            \begin{equation}
            \tilde{f}_t(\bx) = \frac{f_t(\bx_t)}{\sqrt{P_t(\bx)P_t(\bx_t)}}\phi_t(\bx)^{\top} \rbr{\sum_{\bx' \in \mX} \phi_t(\bx') \phi_t(\bx')^{\top} + \frac{\lambda}{T} \bI_{s_t}}^{-1} \phi_t(\bx_t).
            \end{equation}
            \STATE Calculate optimistic estimate $\tilde{u}_t$ as $\tilde{u}_t(\bx) = \tilde{f}_t(\bx) + \beta \tilde{\sigma}_t(\bx)$, where
            \begin{equation}
            \begin{split}
                \tilde{\sigma}_t^2(\bx) &= 
    \frac{T}{\lambda P_t(\bx)} \rbr{\tilde{k}_t(\bx, \bx) - \phi_t(\bx)^{\top} \phi_t(\bx)} \\
        &~~~~+ \frac{1}{P_t(\bx)} \phi_t(\bx)^{\top} \rbr{\sum_{\bx' \in \mX} \phi_t(\bx') \phi_t(\bx')^{\top} + \frac{\lambda}{T} \bI_{s_t}}^{-1} \phi_t(\bx).
            \end{split}
            \end{equation}
            \STATE Update the sampling distribution:
            \begin{equation}
                P_{t+1}(\bx) = \alpha \tilde{\pi}(\bx) + (1 - \alpha) \tilde{P}_{t+1}(\bx),~\text{where}~\tilde{P}_{t+1}(\bx) = \frac{\exp\rbr{\eta \sum_{i=1}^t \tilde{u}_i(\bx)}}{\sum_{\bx' \in \mX} \exp\rbr{\eta \sum_{i=1}^t \tilde{u}_i(\bx')}}.
            \end{equation}
        \ENDFOR
    \end{algorithmic}
\end{algorithm}

\section{Proof of \lemref{lem:est_var}}
\label{app:est_var_proof}
\begin{proof}[Proof of \lemref{lem:est_var}]
    Firstly, note that the following identity of the posterior variance holds (\lemref{lem:feature_map_gp}):
    \begin{align}
        \label{eq:feature_sigma2_rep}
        \sigma^2(\bx; \bX^{(T)}, \lambda) = \lambda \psi(\bx)^{\top} \left(\sum_{t=1}^T \psi(\bx^{(t)})\psi(\bx^{(t)})^{\top} + \lambda \bI_{|\mX|} \right)^{-1} \psi(\bx).
    \end{align}
    Here, from the definition of $\|\psi(\bz)\|_{G(T, \lambda, P)^{-1}}^2$, we have
    \begin{align}
        \|\psi(\bz)\|_{G(T, \lambda, P)^{-1}}^2 &= T \psi(\bz)^{\top} \rbr{T \sum_{\bx \in \mX} P(\bx) \psi(\bx) \psi(\bx)^{\top} + \lambda \bI_{|\mX|}}^{-1} \psi(\bz) \\
        &= T \psi(\bz)^{\top} \rbr{\Ep_{\bX^{(T)}}\sbr{\sum_{t=1}^T \psi(\bx^{(t)}) \psi(\bx^{(t)})^{\top} + \lambda \bI_{|\mX|}}}^{-1} \psi(\bz).
    \end{align}
    Applying Jensen's inequality for a positive definite matrix by noting the convexity of the matrix inversion, we obtain
    \begin{align}
        \|\psi(\bz)\|_{G(T, \lambda, P)^{-1}}^2 
        &\leq T \Ep_{\bX^{(T)}}\sbr{\psi(\bz)^{\top} \rbr{\sum_{t=1}^T \psi(\bx^{(t)}) \psi(\bx^{(t)})^{\top} + \lambda \bI_{|\mX|}}^{-1} \psi(\bz)} \\
        \label{eq:fix_z_error}
        &= \frac{T}{\lambda} \Ep_{\bX^{(T)}} \sbr{\sigma^2 (\bz; \bX^{(T)}, \lambda)},    
        \end{align}
    where the last equality follows from the identity in Eq.~\eqref{eq:feature_sigma2_rep}. Next, by taking the expectation, we immediately obtain the first inequality $\Ep_{\bx^{(T+1)}}\sbr{\|\psi(\bx^{(T+1)})\|_{G(T, \lambda, P)^{-1}}^2}
        \leq \frac{T}{\lambda}\Ep_{\bX^{(T)}, \bx^{(T+1)}} \sbr{\sigma^2 (\bx^{(T+1)}; \bX^{(T)}, \lambda)}$ in Eq.~\eqref{eq:iid_z_error}
    from Eq.~\eqref{eq:fix_z_error}. Regarding the second inequality in Eq.~\eqref{eq:iid_z_error}, we have
    \begin{align}
        T \Ep_{\bX^{(T)}, \bx^{(T+1)}} \sbr{\sigma^2 (\bx^{(T+1)}; \bX^{(T)}, \lambda)} &= 
        \Ep_{\bX^{(T)}, \bx^{(T+1)}} \sbr{\sum_{t=1}^T \sigma^2 (\bx^{(T+1)}; \bX^{(T)}, \lambda)} \\
        &\leq \Ep_{\bX^{(T)}, \bx^{(T+1)}} \sbr{\sum_{t=1}^T \sigma^2 (\bx^{(T+1)}; \bX^{(t-1)}, \lambda)},
    \end{align}
    where the last inequality follows from the monotonicity of the posterior variance with respect to training input data. Since $\bx^{(T+1)}$ and $\bx^{(t)}$ are identically distributed, we have $\Ep_{\bX^{(T)}, \bx^{(T+1)}} \sbr{\sum_{t=1}^T \sigma^2 (\bx^{(T+1)}; \bX^{(t-1)}, \lambda)} = \Ep_{\bX^{(T)}, \bx^{(T+1)}} \sbr{\sum_{t=1}^T \sigma^2 (\bx^{(t)}; \bX^{(t-1)}, \lambda)} \leq 2\gamma_T/\log(1 + \lambda^{-1})$, where the final inequality follows from Eq.~\eqref{eq:sigma_2_cum_ub}\footnote{Note that Eq.~\eqref{eq:sigma_2_cum_ub} requires the condition $\forall \bx \in \mX$, $k(\bx, \bx) \leq 1$. See Lemma 5.4 in \citep{srinivas10gaussian}.}. This concludes the second inequality in Eq.~\eqref{eq:iid_z_error}. Finally, the inequality $2\lambda^{-1} \gamma_T/\log(1 + \lambda^{-1}) \leq 4 \gamma_T$ for $\lambda \geq 1$ is the direct consequence of the elementary inequality $\forall a \in [0, 1], \log(1 + a) \geq a/2$.
\end{proof}

\section{Proof of \thmref{thm:reg_exp3}}
\label{app:proof_reg_exp3}
From Lemma~\ref{lem:feature_map_rkhs}, each reward function $f_t$ can be rewritten as $f_t(\cdot) = \psi(\cdot)^{\top} \bm{\theta}_t$, where $\bm{\theta}_t$ is drawn by the environment.
Note that $\bm{\theta}_t$ and $\bx_t$ are conditionally independent given $\mH_{t-1}$ as defined in the problem setting.
Here, we first recall the following three basic facts used in various parts of our proof: (i) $\Ep[g(\bx_t) | \mH_{t-1}] = \sum_{\bx \in \mX} P_t(\bx) g(\bx)$ holds for any (measurable) fixed function $g$, (ii) $\Ep[g(\bx_t) | \mH_{t-1}] = \Ep[g(\bx_t) | \mH_{t-1}, \bm{\theta}_{t}]$ holds from the conditional independence of $\bx_t$ and $\bm{\theta}_t$ given $\mH_{t-1}$, and (iii) $\Ep[h(\mH_{t-1}, \bm{\theta}_t) g(\bx_t) | \mH_{t-1}, \bm{\theta}_t] = h(\mH_{t-1}, \bm{\theta}_t) \Ep[g(\bx_t) | \mH_{t-1}, \bm{\theta}_t]$ for any (measurable) fixed functions $g$ and $h$.

Next, from the tower property of conditional expectation, we have $\Ep[\sum_{t=1}^T f_t(\bx_t)] = \sum_{t=1}^T \Ep[\Ep[f_t(\bx_t) \mid \mH_{t-1}, \bm{\theta}_t]] = \sum_{t=1}^T \Ep[\sum_{\bx \in \mX} P_t(\bx) f_t(\bx)]$, where the second equality follows from
\begin{align}
    \Ep[f_t(\bx_t) \mid \mH_{t-1}, \bm{\theta}_t] 
    &= \Ep[\bm{\theta}_t^{\top} \psi(\bx_t) \mid \mH_{t-1}, \bm{\theta}_t] 
    = \bm{\theta}_t^{\top}\Ep[\psi(\bx_t) \mid \mH_{t-1}, \bm{\theta}_t] \\
    &= \bm{\theta}_t^{\top}\Ep[\psi(\bx_t) \mid \mH_{t-1}] = \sum_{\bx \in \mX} P_t(\bx) \bm{\theta}_t^{\top} \psi(\bx) = \sum_{\bx \in \mX} P_t(\bx) f_t(\bx).
\end{align}
Thus, we have
\begin{equation}
    \label{eq:first_tower_property}
    \bar{R}_T = \sum_{t=1}^T \Ep\sbr{f_t(\bx^{\ast}) - \sum_{\bx \in \mX} P_t(\bx) f_t(\bx)},
\end{equation}
where $\bx^{\ast} \in \argmax_{\bx \in \mX} \Ep[\sum_{t=1}^T f_t(\bx)]$. Note that $\bx^{\ast}$ is always well-defined since $|\mX| < \infty$. 
Furthermore,
\begin{align}
    &\Ep\sbr{f_t(\bx^{\ast}) - \sum_{\bx \in \mX} P_t(\bx) f_t(\bx)} \\
    &= \alpha \Ep\sbr{f_t(\bx^{\ast}) - \sum_{\bx \in \mX} \pi(\bx) f_t(\bx)} + (1 - \alpha) \Ep\sbr{f_t(\bx^{\ast}) - \sum_{\bx \in \mX} \tilde{P}_t(\bx) f_t(\bx)} \\
    \label{eq:first_perreg_ub}
    &\leq 2\alpha B + (1 - \alpha) \Ep\sbr{f_t(\bx^{\ast}) - \sum_{\bx \in \mX} \tilde{P}_t(\bx) f_t(\bx)},
\end{align}
where the last inequality follows from $\|f_t\|_{\infty} \leq \|f_t\|_{k} \leq B$.
Here, we further decompose the second term of the above equation by using the optimistic estimate $u_t(\bx)$ as follows:
\begin{align}
    &\Ep\sbr{f_t(\bx^{\ast}) - \sum_{\bx \in \mX} \tilde{P}_t(\bx) f_t(\bx)} \\
    \label{eq:decomp}
    & = \underbrace{\Ep\sbr{f_t(\bx^{\ast}) - u_t(\bx^{\ast})}}_{\text{(i)}} + 
    \underbrace{\Ep\sbr{u_t(\bx^{\ast}) - \sum_{\bx \in \mX} \tilde{P}_t(\bx) u_t(\bx)}}_{\text{(ii)}} 
    + \underbrace{\Ep\sbr{\sum_{\bx \in \mX} \tilde{P}_t(\bx) \rbr{u_t(\bx) -  f_t(\bx)}}}_{\text{(iii)}}.
\end{align}

\paragraph{Upper bound for the first term (i) in Eq.~\eqref{eq:decomp}.} 
From the definition of $u_t(\bx^{\ast})$, we have
\begin{equation}
    f_t(\bx^{\ast}) - u_t(\bx^{\ast}) 
    = f_t(\bx^{\ast}) - \hat{f}_t(\bx^{\ast}) - \beta \|\psi(\bx^{\ast})\|_{G_t(\lambda)^{-1}}.
\end{equation}
By noting the definition of $\hat{f}_t$ and $f_t(\cdot) = \psi(\cdot)^{\top} \bm{\theta}_t$ (Lemma~\ref{lem:feature_map_rkhs}), the error term $f_t(\bx^{\ast}) - \hat{f}_t(\bx^{\ast})$ in the above equation can be rewritten as follows:
\begin{align}
    f_t(\bx^{\ast}) - \hat{f}_t(\bx^{\ast}) 
    &= \psi(\bx^{\ast})^{\top} \bm{\theta}_t - \psi(\bx^{\ast})^{\top} G_t(\lambda)^{-1} \psi(\bx_t) f_t(\bx_t) \\
    &= \psi(\bx^{\ast})^{\top} \bm{\theta}_t - \psi(\bx^{\ast})^{\top} G_t(\lambda)^{-1} \psi(\bx_t) \psi(\bx_t)^{\top} \bm{\theta}_t \\
    &= \psi(\bx^{\ast})^{\top} \rbr{\bI_{|\mX|} - G_t(\lambda)^{-1} \psi(\bx_t) \psi(\bx_t)^{\top}} \bm{\theta}_t.
\end{align}
By noting that $\bm{\theta}_t$ is conditionally independent of $\bx_t$ given $\mH_{t-1}$, we have
\begin{align}
    \Ep[f_t(\bx^{\ast}) - \hat{f}_t(\bx^{\ast}) \mid \mH_{t-1}, \bm{\theta}_t]
    &= \psi(\bx^{\ast})^{\top} \rbr{\bI_{|\mX|} - G_t(\lambda)^{-1} \Ep[ \psi(\bx_t) \psi(\bx_t)^{\top} \mid \mH_{t-1}, \bm{\theta}_t]} \bm{\theta}_t \\
    &= \psi(\bx^{\ast})^{\top} \sbr{G_t(\lambda)^{-1} \rbr{G_t(\lambda) - \sum_{\bx \in \mX} P_t(\bx)\psi(\bx) \psi(\bx)^{\top} }} \bm{\theta}_t \\
    \label{eq:cond_error_i}
    &= \frac{\lambda}{T} \psi(\bx^{\ast})^{\top} G_t(\lambda)^{-1} \bm{\theta}_t,
\end{align}
where the last line follows from the definition of $G_t(\lambda) \coloneqq \sum_{\bx \in \mX} P_t(\bx)\psi(\bx) \psi(\bx)^{\top} + \frac{\lambda}{T}\bm{I}_{|\mX|}$. From the Cauchy--Schwarz inequality and $\|\bm{\theta}_t\|_2 = \|f_t\|_{k} \leq B$ (Lemma~\ref{lem:feature_map_rkhs}), we have
\begin{align}
    \frac{\lambda}{T} \psi(\bx^{\ast})^{\top} G_t(\lambda)^{-1} \bm{\theta}_t 
    &\leq \frac{\lambda}{T} \|\psi(\bx^{\ast})^{\top} G_t(\lambda)^{-1}\|_2 \|\bm{\theta}_t\|_2 \\
    &\leq B \frac{\lambda}{T} \sqrt{\psi(\bx^{\ast})^{\top} G_t(\lambda)^{-2} \psi(\bx^{\ast})} \\
    &\leq B \sqrt{\frac{\lambda}{T}} \sqrt{\psi(\bx^{\ast})^{\top} G_t(\lambda)^{-1} \psi(\bx^{\ast})} \\
    \label{eq:ub_schwartz}
    &= B \sqrt{\frac{\lambda}{T}} \|\psi(\bx^{\ast})\|_{G_t(\lambda)^{-1}},
\end{align}
where the last inequality follows from $\psi(\bx^{\ast})^{\top} G_t(\lambda)^{-2} \psi(\bx^{\ast}) \leq \frac{T}{\lambda} \psi(\bx^{\ast})^{\top} G_t(\lambda)^{-1} \psi(\bx^{\ast})$, which is implied by the fact that the maximum eigenvalue of $G_t(\lambda)^{-1}$ is bounded from above by $T/\lambda$. 
From Eqs.~\eqref{eq:cond_error_i} and \eqref{eq:ub_schwartz}, we have
\begin{align}
    \Ep[f_t(\bx^{\ast}) - u_t(\bx^{\ast})] 
    &= \Ep\sbr{f_t(\bx^{\ast}) - \hat{f}_t(\bx^{\ast}) - \beta \|\psi(\bx^{\ast})\|_{G_t(\lambda)^{-1}}} \\
    &= \Ep[\Ep[f_t(\bx^{\ast}) - \hat{f}_t(\bx^{\ast})\mid \mH_{t-1}, \bm{\theta}_t]] - \Ep[\beta \|\psi(\bx^{\ast})\|_{G_t(\lambda)^{-1}}] \\
    &\leq \Ep\sbr{\rbr{B \sqrt{\frac{\lambda}{T}} - \beta} \|\psi(\bx^{\ast})\|_{G_t(\lambda)^{-1}}} \\
    \label{eq:final_i}
    &= 0.
\end{align}
Thus, the first term (i) in Eq.~\eqref{eq:decomp} is bounded from above by $0$.

\paragraph{Upper bound for the second term (ii) in Eq.~\eqref{eq:decomp}.}
For any $\bx \in \mX$, the definition of the exploration distribution implies the following upper bound of $u_t(\bx)$ (Lemma~\ref{lem:ub_ut}):
\begin{equation}
    u_t(\bx) \leq \frac{2}{\alpha} (2 B \gamma_{T} + \beta \sqrt{\gamma_T}).
\end{equation}
By combining the above inequality with the definition of $\alpha$, we can observe that $\eta u_t(\bx) \leq 1$ holds for any $\bx \in \mX$; therefore, by following the standard proof for Exp3~(Lemma~\ref{lem:exp3_standard}), we obtain
\begin{align}
    \label{eq:ut_put_first_ub}
    \sum_{t=1}^T u_t(\bx^{\ast}) - \sum_{t=1}^T \sum_{\bx \in \mX} \tilde{P}_t(\bx) u_t(\bx) 
    \leq \frac{\log |\mX|}{\eta} + \eta \sum_{t=1}^T \sum_{\bx \in \mX} \tilde{P}_t(\bx) u_t^2(\bx).
\end{align}
Regarding the term $\sum_{\bx \in \mX} \tilde{P}_t(\bx) u_t^2(\bx)$, we further obtain the upper bound as follows:
\begin{align}
    \sum_{\bx \in \mX} \tilde{P}_t(\bx) u_t^2(\bx) 
    &= \sum_{\bx \in \mX} \tilde{P}_t(\bx) \rbr{\hat{f}_t(\bx) + \beta \|\psi(\bx)\|_{G_t(\lambda)^{-1}}}^2 \\
    \label{eq:ut_var_decomp}
    &\leq 2 \sum_{\bx \in \mX} \tilde{P}_t(\bx) \hat{f}_t^2(\bx) + 2 \beta^2 \sum_{\bx \in \mX} \tilde{P}_t(\bx) \|\psi(\bx)\|_{G_t(\lambda)^{-1}}^2,
\end{align}
where we used $\forall a, b \in \R,~(a + b)^2 \leq 2a^2 + 2b^2$ in the last inequality. Here, note the following relation:
\begin{align}
    \label{eq:Ptilde_Gt_rel}
    (1-\alpha) \sum_{\bx \in \mX} \tilde{P}_t(\bx) \psi(\bx) \psi(\bx)^{\top} + \frac{\lambda}{T} \bI_{|\mX|} 
    \preceq \sum_{\bx \in \mX} P_t(\bx) \psi(\bx) \psi(\bx)^{\top} + \frac{\lambda}{T} \bI_{|\mX|} \coloneqq G_t(\lambda).
\end{align}
Thus, for $M_t(\lambda) \coloneqq \sum_{\bx \in \mX} \tilde{P}_t(\bx) \psi(\bx) \psi(\bx)^{\top} + \frac{\lambda}{\lceil T (1 - \alpha) \rceil} \bI_{|\mX|}$, we can confirm
\begin{equation}
    G_t(\lambda)^{-1} \preceq \frac{1}{1-\alpha} \rbr{\sum_{\bx \in \mX} \tilde{P}_t(\bx) \psi(\bx) \psi(\bx)^{\top} + \frac{\lambda}{T (1 - \alpha) } \bI_{|\mX|}}^{-1} \preceq \frac{1}{1-\alpha} M_t(\lambda)^{-1}.
\end{equation}
The above relation further implies 
\begin{equation}
    \label{eq:tilde_p_psi}
    \sum_{\bx \in \mX} \tilde{P}_t(\bx) \|\psi(\bx)\|_{G_t(\lambda)^{-1}}^2 
    \leq \frac{1}{1-\alpha} \sum_{\bx \in \mX} \tilde{P}_t(\bx) \|\psi(\bx)\|_{M_t(\lambda)^{-1}}^2 \leq \frac{4 \gamma_{\lceil T (1 - \alpha) \rceil}}{1-\alpha} \leq \frac{4 \gamma_{T}}{1-\alpha},
\end{equation}
where the second inequality follows from \lemref{lem:est_var}. 
Therefore, by noting the definition of $\beta$, the second term in Eq.~\eqref{eq:ut_var_decomp} is bounded from above as follows:
\begin{equation}
    \label{eq:ut_var_decomp_ub2}
    2 \beta^2 \sum_{\bx \in \mX} \tilde{P}_t(\bx) \|\psi(\bx)\|_{G_t(\lambda)^{-1}}^2 \leq \frac{8 \lambda B^2 \gamma_{T}}{(1-\alpha) T}.
\end{equation}
Next, regarding the first term $\sum_{\bx \in \mX} \tilde{P}_t(\bx) \hat{f}_t^2(\bx)$ in Eq.~\eqref{eq:ut_var_decomp}, we have
\begin{align}
    \sum_{\bx \in \mX} \tilde{P}_t(\bx) \hat{f}_t^2(\bx) 
    &= \sum_{\bx \in \mX} \tilde{P}_t(\bx) \sbr{\psi(\bx)^{\top} G_t(\lambda)^{-1} \psi(\bx_t)f_t(\bx_t)}^2 \\
    &\leq B^2 \sum_{\bx \in \mX} \tilde{P}_t(\bx) \sbr{\psi(\bx)^{\top} G_t(\lambda)^{-1} \psi(\bx_t)}^2 \\
    &= \frac{B^2}{1 - \alpha} \psi(\bx_t)^{\top} G_t(\lambda)^{-1} \sbr{\sum_{\bx \in \mX} (1 - \alpha) \tilde{P}_t(\bx) \psi(\bx) \psi(\bx)^{\top}} G_t(\lambda)^{-1} \psi(\bx_t) \\
    \label{eq:ut_var_decomp_ub1}
    &\leq \frac{B^2}{1 - \alpha} \psi(\bx_t)^{\top} G_t(\lambda)^{-1} \psi(\bx_t).
\end{align}
Thus, by aggregating Eqs.~\eqref{eq:ut_put_first_ub}, \eqref{eq:ut_var_decomp}, \eqref{eq:ut_var_decomp_ub2}, and \eqref{eq:ut_var_decomp_ub1}, we have
\begin{align}
    &\Ep\sbr{\sum_{t=1}^T u_t(\bx^{\ast}) - \sum_{t=1}^T \sum_{\bx \in \mX} \tilde{P}_t(\bx) u_t(\bx) } \\
    &\leq \frac{\log |\mX|}{\eta} + \frac{8 \eta \lambda B^2 \gamma_T}{1 - \alpha} + \frac{2\eta B^2}{1 - \alpha} \sum_{t=1}^T \Ep\sbr{\psi(\bx_t)^{\top} G_t(\lambda)^{-1} \psi(\bx_t)} \\
    \label{eq:final_ii_tower_prop}
    &= \frac{\log |\mX|}{\eta} + \frac{8 \eta \lambda B^2 \gamma_T}{1 - \alpha} + \frac{2\eta B^2}{1-\alpha} \sum_{t=1}^T \Ep\sbr{ \sum_{\bx \in \mX} P_t(\bx) \|\psi(\bx)\|_{G_t(\lambda)^{-1}}^2} \\
    \label{eq:final_ii}
    &\leq \frac{\log |\mX|}{\eta} + \frac{8 \eta \lambda B^2 \gamma_T}{1 - \alpha} + \frac{8\eta B^2 T \gamma_T}{1-\alpha},
\end{align}
where Eq.~\eqref{eq:final_ii_tower_prop} follows from the tower property of conditional expectation, and Eq.~\eqref{eq:final_ii} follows from \lemref{lem:est_var}.

\paragraph{Upper bound for the third term (iii) in Eq.~\eqref{eq:decomp}.}
By noting the definition of $u_t$, we have
\begin{align}
    \label{eq:third_iii_ub}
    \sum_{\bx \in \mX} \tilde{P}_t(\bx) \rbr{u_t(\bx) -  f_t(\bx)}
    &= \beta \sum_{\bx \in \mX} \tilde{P}_t(\bx) \|\psi(\bx)\|_{G_t(\lambda)^{-1}} + \sum_{\bx \in \mX} \tilde{P}_t(\bx) \rbr{\hat{f}_t(\bx) - f_t(\bx)}.
\end{align}
Regarding the first term in Eq.~\eqref{eq:third_iii_ub}, the applications of Jensen's inequality for the square root $\sqrt{x}$ and Eq.~\eqref{eq:tilde_p_psi} imply
\begin{align}
    \beta \sum_{\bx \in \mX} \tilde{P}_t(\bx) \|\psi(\bx)\|_{G_t(\lambda)^{-1}}
    &\leq \beta \sqrt{\sum_{\bx \in \mX} \tilde{P}_t(\bx) \|\psi(\bx)\|_{G_t(\lambda)^{-1}}^2} \\
    \label{eq:bias_ub_iii}
    &\leq 2B\sqrt{\frac{\lambda \gamma_T}{(1 - \alpha)T}}.
\end{align}
Regarding the second term in Eq.~\eqref{eq:third_iii_ub}, we have 
\begin{align}
    \hat{f}_t(\bx) - f_t(\bx) 
    &= \psi(\bx)^{\top} G_t(\lambda)^{-1} \psi(\bx_t) f_t(\bx_t) - f_t(\bx) \\
    &= \psi(\bx)^{\top} G_t(\lambda)^{-1} \psi(\bx_t) \psi(\bx_t)^{\top} \bm{\theta}_t - \psi(\bx)^{\top} \bm{\theta}_t \\
    &= \psi(\bx)^{\top} \rbr{G_t(\lambda)^{-1} \psi(\bx_t) \psi(\bx_t)^{\top} - \bI_{|\mX|}} \bm{\theta}_t \\
    &= \psi(\bx)^{\top} \sbr{G_t(\lambda)^{-1} \rbr{\psi(\bx_t) \psi(\bx_t)^{\top} - G_t(\lambda)}} \bm{\theta}_t.
\end{align}
Thus, we have the following inequality from the above equation:
\begin{align}
    &\Ep\sbr{\sum_{\bx \in \mX} \tilde{P}_t(\bx)\rbr{\hat{f}_t(\bx) - f_t(\bx)}} \\
    &= \Ep\sbr{\Ep\left[ \sum_{\bx \in \mX} \tilde{P}_t(\bx)\rbr{\hat{f}_t(\bx) - f_t(\bx)} ~\middle|~ \mH_{t-1}, \bm{\theta}_t \right]} \\
    &= \Ep\sbr{\sum_{\bx \in \mX} \tilde{P}_t(\bx) \psi(\bx)^{\top} \sbr{G_t(\lambda)^{-1} \rbr{\Ep\left[\psi(\bx_t) \psi(\bx_t)^{\top}~\middle|~ \mH_{t-1}, \bm{\theta}_t\right] - G_t(\lambda)}} \bm{\theta}_t} \\
    \label{eq:exp_identity_iii}
    &\leq \frac{\lambda}{T}\Ep\sbr{\sum_{\bx \in \mX} |\tilde{P}_t(\bx) \psi(\bx)^{\top} G_t(\lambda)^{-1} \bm{\theta}_t|}.
\end{align}
The upper bound of the quantity in the above expectation can be further obtained as
\begin{align}
    \sum_{\bx \in \mX} |\tilde{P}_t(\bx) \psi(\bx)^{\top} G_t(\lambda)^{-1} \bm{\theta}_t|
    &\leq \sum_{\bx \in \mX} \tilde{P}_t(\bx) \sqrt{\psi(\bx)^{\top} G_t(\lambda)^{-2} \psi(\bx)} \|\bm{\theta}_t\|_2 \\
    &\leq B \sum_{\bx \in \mX} \tilde{P}_t(\bx) \sqrt{\psi(\bx)^{\top} G_t(\lambda)^{-2} \psi(\bx)} \\
    &\leq B \sqrt{\frac{T}{\lambda}} \sum_{\bx \in \mX} \tilde{P}_t(\bx) \sqrt{\psi(\bx)^{\top} G_t(\lambda)^{-1} \psi(\bx)} \\
    &\leq B \sqrt{\frac{T}{\lambda}} \sqrt{\sum_{\bx \in \mX} \tilde{P}_t(\bx) \psi(\bx)^{\top} G_t(\lambda)^{-1} \psi(\bx)} \\
    \label{eq:exp_inside_iii}
    &\leq 2B \sqrt{\frac{T \gamma_T}{(1 - \alpha)\lambda}}.
\end{align}
Thus, by aggregating Eqs.~\eqref{eq:third_iii_ub},~\eqref{eq:bias_ub_iii},~\eqref{eq:exp_identity_iii}, and \eqref{eq:exp_inside_iii}, the upper bound of the third term (iii) in Eq.~\eqref{eq:decomp} is obtained as follows:
\begin{equation}
    \label{eq:final_iii}
    \Ep\sbr{\sum_{\bx \in \mX} \tilde{P}_t(\bx) \rbr{u_t(\bx) -  f_t(\bx)}}
    \leq 4B\sqrt{\frac{\lambda \gamma_T}{(1-\alpha)T}}.
\end{equation}

\paragraph{Aggregate all the upper bounds.}
From Eqs.~\eqref{eq:first_tower_property}, ~\eqref{eq:first_perreg_ub},~\eqref{eq:decomp},~\eqref{eq:final_i},~\eqref{eq:final_ii}, and \eqref{eq:final_iii}, we finally obtain
\begin{align}
    \bar{R}_T &\leq 2\alpha B T + \frac{(1 - \alpha) \log |\mX|}{\eta} + 8 \eta \lambda B^2 \gamma_T + 8\eta B^2 T \gamma_T + 4B\sqrt{(1-\alpha) \lambda T \gamma_T} \\
    &\leq 2\alpha B T + \frac{\log |\mX|}{\eta} + 8 \eta (\lambda + T) B^2 \gamma_T + 4B\sqrt{\lambda T \gamma_T}.
\end{align} \qed

\section{Proof of \thmref{thm:reg_exp3_nystrom_approx}}
\label{app:proof_random_exp3}
\paragraph{Proof for computational complexity.} Here, for any $t \in [T]$, we define the event $A_t$ as follows:
\begin{equation}
    \label{eq:thm_rke3_proof_eps_acc}
    \frac{1}{2}\rbr{\bar{\bPsi}_t \bar{\bPsi}_t^{\top} + \frac{\lambda}{T} \bI_{|\mX|}} \preceq \bar{\bPsi}_t \bS_t \bS_t^{\top} \bar{\bPsi}_t^{\top} + \frac{\lambda}{T} \bI_{|\mX|} \preceq \frac{3}{2}\rbr{\bar{\bPsi}_t \bar{\bPsi}_t^{\top} + \frac{\lambda}{T} \bI_{|\mX|}}.
\end{equation}
From \lemref{lem:acc_comp_rls}, we can confirm that $\mathbb{P}(A_t \mid \tilde{\mH}_{t-1}) \geq 1 - 1/T^3$ for any $t \in [T]$ and $\tilde{\mH}_{t-1} \coloneqq \{(\bS_i, \bx_i, f_i(\bx_i))\}_{i \in [t-1]}$, which implies $\mathbb{P}(A_t) \geq 1 - 1/T^3$ due to the tower property of conditional expectation. By taking the union bound, we can obtain $\Pr\rbr{A} \geq 1 - 1/T^2$, where $A = \bigcap_{t \in [T]} A_t$. 
Furthermore, \lemref{lem:acc_comp_rls} suggests that, with probability at least $1-1/T^2$, the following two statements simultaneously also hold:
\begin{enumerate}
    \item The number $s_t$ satisfies $s_t = O(\gamma_T \log(T\gamma_T)) = O(\gamma_T \log(T))$ for any $t \in [T]$\footnote{Note that $\gamma_T = O(T)$ always holds.}.
    \item The computational cost of recursive RLS-Nystr\"om is $O(|\mX|\gamma_T^2 (\log(T\gamma_T))^2) = O(|\mX|\gamma_T^2 (\log T)^2)$.
\end{enumerate}
Here, we confirm the computational cost of $(\tilde{f}_t(\bx))_{\bx \in \mX}$ and $(\tilde{\sigma}_t^2(\bx))_{\bx \in \mX}$ is $O(|\mX|s_t^2 + s_t^3)$.
Firstly, the computation of Nystr\"om embedding $\phi_t(\cdot)$ requires the calculation of pseudo-inverse $[(\bS_t^{\top} \tilde{\bK}_t \bS_t)^{1/2}]^{\dagger}$, which requires $O(s_t^3)$-computation. Once $[(\bS_t^{\top} \tilde{\bK}_t \bS_t)^{1/2}]^{\dagger}$ is computed, the calculation of $\phi_t(\bx) \coloneqq [(\bS_t^{\top} \tilde{\bK}_t \bS_t)^{1/2}]^{\dagger}\bS_t^{\top} \tilde{\bk}_t(\bx)$ can be conducted within $O(s_t^2)$ for some given $\bx \in \mX$. Thus, the computational cost for the preparation of $(\phi_t(\bx))_{\bx \in \mX}$ is $O(|\mX|s_t^2 + s_t^3)$. After calculating $(\phi_t(\bx))_{\bx \in \mX}$, we require $O(|\mX|s_t^2)$ and $O(s_t^3)$ computations for calculating $\sum_{\bx' \in \mX} \phi_t(\bx') \phi_t(\bx')^{\top}$ and $(\sum_{\bx' \in \mX} \phi_t(\bx') \phi_t(\bx')^{\top} + \lambda/T \bI_{s_t})^{-1}$, respectively. Once we have computed the inverse $(\sum_{\bx' \in \mX} \phi_t(\bx') \phi_t(\bx')^{\top} + \lambda/T \bI_{s_t})^{-1}$, the calculation of the quadratic form $\phi_t(\bx)^{\top} (\sum_{\bx' \in \mX} \phi_t(\bx') \phi_t(\bx')^{\top} + \lambda/T \bI_{s_t})^{-1} \phi_t(\bx)$ for all $\bx \in \mX$ requires $O(|\mX| s_t^2)$. In summary, given the sampling matrix $\bS_t$, the overall computational cost for round $t$ is $O(|\mX|s_t^2 + s_t^3)$. By combining this with the above two high probability statements, we can confirm that the total computational cost at each round $t$ is $\tilde{O}(|\mX|\gamma_T^2 + \gamma_T^3)$ with high probability.

\paragraph{Proof for regret upper bound.} We first define the following notations, while some of them are already defined in \secref{sec:exp3_nystrom_approx}:
\begin{itemize}
    \item Let $\bS_t \in \R^{|\mX| \times s_t}$ be the sampling matrix at round $t$ returned by recursive RLS-Nystr\"om.
    \item Let $\bar{\bPsi}_t \coloneqq (\bar{\psi}_t(\bx^{(1)}), \ldots, \bar{\psi}_t(\bx^{(|\mX|)})) \in \R^{|\mX| \times |\mX|}$ be the weighted feature matrix, where $\bar{\psi}_t(\bx) \coloneqq \sqrt{P_t(\bx)} \psi(\bx)$ is the weighted feature map.
    \item Let $\bQ_t \coloneqq \bar{\bPsi}_t \bS_t \sbr{(\bar{\bPsi}_t \bS_t)^{\top} (\bar{\bPsi}_t \bS_t)}^{\dagger} (\bar{\bPsi}_t \bS_t)^{\top}$ be the projection matrix into the subspace spanned by the column vectors of $\bar{\bPsi}_t \bS_t$. We also define $\tilde{\psi}_t(\bx)$ as $\tilde{\psi}_t(\bx) = \bQ_t \psi(\bx)$.
    \item Let $\tilde{G}_t(\lambda)$ be $\tilde{G}_t(\lambda) = \bQ_t \bar{\bPsi}_t (\bQ_t \bar{\bPsi}_t)^{\top} + \frac{\lambda}{T} \bI_{|\mX|}$, which is the approximated version of $G_t(\lambda) \coloneqq \bar{\bPsi}_t\bar{\bPsi}_t^{\top} + \frac{\lambda}{T} \bI_{|\mX|} = \sum_{\bx \in \mX} P_t(\bx) \psi(\bx) \psi(\bx)^{\top} + \frac{\lambda}{T} \bI_{|\mX|}$.
    \item Let $\tilde{\bPsi}_t \coloneqq \bQ_t \bar{\bPsi}_t$ be the projected version of $\bar{\bPsi}_t$.
\end{itemize}
In the proof of this section, we define $\tilde{\mH}_{t}$ as the learner's history up to round $t$: $\tilde{\mH}_{t} = \{(\bS_i, \bx_i, f_i(\bx_i))\}_{i \in [t]}$.
Here, from the algorithm construction, note that $\bm{\theta}_t$, $\bx_t$, and $\bS_t$ are conditionally independent given $\tilde{\mH}_{t-1}$.
As described in \secref{sec:exp3_nystrom_approx}, note that the following identities hold (see also \lemref{lem:identity_nystrom}):
\begin{align}
    \tilde{f}_t(\bx) &= \psi(\bx)^{\top} \tilde{G}_t(\lambda)^{-1} \tilde{\psi}_t(\bx_t) \psi(\bx_t)^{\top} \bm{\theta}_t, \\
    \tilde{u}_t(\bx) &= \tilde{f}_t(\bx) + \beta \|\psi(\bx)\|_{\tilde{G}_t(\lambda)^{-1}}.
\end{align}
To prove \thmref{thm:reg_exp3_nystrom_approx}, as with the proof of \thmref{thm:reg_exp3}, we decompose $f_t(\bx^{\ast}) - f_t(\bx_t)$ as follows:
\begin{align}
    \label{eq:decomp_reg_rls_exp3}
    \Ep\sbr{f_t(\bx^{\ast}) - f_t(\bx_t)} 
    \leq 2 \alpha B + (1 - \alpha) \Ep\sbr{f_t(\bx^{\ast}) - \sum_{\bx \in \mX} \tilde{P}_t(\bx) f_t(\bx)},
\end{align}
where $\bx^{\ast} \in \argmax_{\bx \in \mX} \Ep[\sum_{t=1}^T f_t(\bx)]$. 
Now, we decompose the second term in Eq.~\eqref{eq:decomp_reg_rls_exp3} into the following three terms similarly to the proof of \thmref{thm:reg_exp3}:
\begin{align}
    &\Ep\sbr{f_t(\bx^{\ast}) - \sum_{\bx \in \mX} \tilde{P}_t(\bx) f_t(\bx)} \\
    \label{eq:thm_rke3_decomp}
    &= \underbrace{\Ep\sbr{f_t(\bx^{\ast}) - \tilde{u}_t(\bx^{\ast})}}_{\text{(i)}}
    + \underbrace{\Ep\sbr{\tilde{u}_t(\bx^{\ast}) - \sum_{\bx \in \mX} \tilde{P}_t(\bx) \tilde{u}_t(\bx)}}_{\text{(ii)}}
    + \underbrace{\Ep\sbr{\sum_{\bx \in \mX} \tilde{P}_t(\bx) \rbr{\tilde{u}_t(\bx) -  f_t(\bx)}}}_{\text{(iii)}}.
\end{align}
Below, we consider the upper bound of terms (i), (ii), and (iii).

\paragraph{Upper bound for the first term (i) in Eq.~\eqref{eq:thm_rke3_decomp}.} 
By noting the definition of $\tilde{f}_t$ and $f_t(\cdot) = \psi(\cdot)^{\top} \bm{\theta}_t$, the error term $f_t(\bx^{\ast}) - \tilde{f}_t(\bx^{\ast})$ in the above equation can be rewritten as follows:
\begin{align}
    f_t(\bx^{\ast}) - \tilde{f}_t(\bx^{\ast}) 
    &= \psi(\bx^{\ast})^{\top} \bm{\theta}_t - \psi(\bx^{\ast})^{\top} \tilde{G}_t(\lambda)^{-1} \tilde{\psi}_t(\bx_t) f_t(\bx_t) \\
    &= \psi(\bx^{\ast})^{\top} \bm{\theta}_t - \psi(\bx^{\ast})^{\top} \tilde{G}_t(\lambda)^{-1} \tilde{\psi}_t(\bx_t) \psi(\bx_t)^{\top} \bm{\theta}_t \\
    &= \psi(\bx^{\ast})^{\top} \rbr{\bI_{|\mX|} - \tilde{G}_t(\lambda)^{-1} \tilde{\psi}_t(\bx_t) \psi(\bx_t)^{\top}} \bm{\theta}_t.
\end{align}
Then, we have
\begin{align}
    &\Ep[f_t(\bx^{\ast}) - \tilde{f}_t(\bx^{\ast}) \mid \tilde{\mH}_{t-1}, \bS_t, \bm{\theta}_t] \\
    &= \psi(\bx^{\ast})^{\top} \rbr{\bI_{|\mX|} - \tilde{G}_t(\lambda)^{-1} \Ep[ \tilde{\psi}_t(\bx_t) \psi(\bx_t)^{\top} \mid \tilde{\mH}_{t-1}, \bS_t, \bm{\theta}_t]} \bm{\theta}_t \\
    &= \psi(\bx^{\ast})^{\top} \sbr{\tilde{G}_t(\lambda)^{-1} \rbr{\tilde{G}_t(\lambda) - \sum_{\bx \in \mX} P_t(\bx) \tilde{\psi}_t(\bx) \psi(\bx)^{\top} }} \bm{\theta}_t \\
    &= \psi(\bx^{\ast})^{\top} \tilde{G}_t(\lambda)^{-1} \rbr{\tilde{\bPsi}_t \tilde{\bPsi}_t^{\top} - \tilde{\bPsi}_t \bar{\bPsi}_t^{\top} + \frac{\lambda}{T} \bI_{|\mX|}} \bm{\theta}_t \\
    \label{eq:decomp_est_ftilde}
    &= \frac{\lambda}{T} \psi(\bx^{\ast})^{\top} \tilde{G}_t(\lambda)^{-1} \bm{\theta}_t + \psi(\bx^{\ast})^{\top} \tilde{G}_t(\lambda)^{-1} \rbr{\tilde{\bPsi}_t \tilde{\bPsi}_t^{\top} - \tilde{\bPsi}_t \bar{\bPsi}_t^{\top}} \bm{\theta}_t.
\end{align}
Regarding the first term in Eq.~\eqref{eq:decomp_est_ftilde}, we have
\begin{align}
    \frac{\lambda}{T} \psi(\bx^{\ast})^{\top} \tilde{G}_t(\lambda)^{-1} \bm{\theta}_t 
    &\leq \frac{\lambda}{T} \|\psi(\bx^{\ast})^{\top} \tilde{G}_t(\lambda)^{-1}\|_2 \|\bm{\theta}_t\|_2 \\
    &\leq B \frac{\lambda}{T} \sqrt{\psi(\bx^{\ast})^{\top} \tilde{G}_t(\lambda)^{-2} \psi(\bx^{\ast})} \\
    &\leq B \sqrt{\frac{\lambda}{T}} \sqrt{\psi(\bx^{\ast})^{\top} \tilde{G}_t(\lambda)^{-1} \psi(\bx^{\ast})} \\
    &= B \sqrt{\frac{\lambda}{T}} \|\psi(\bx^{\ast})\|_{\tilde{G}_t(\lambda)^{-1}}.
\end{align}
Regarding the second term in Eq.~\eqref{eq:decomp_est_ftilde}, under $(\tilde{\mH}_{t-1}, \bm{\theta}_t, \bm{S}_t)$ such that event $A_t$ is true, we have
\begin{align}
    &\psi(\bx^{\ast})^{\top} \tilde{G}_t(\lambda)^{-1} \rbr{\tilde{\bPsi}_t \tilde{\bPsi}_t^{\top} - \tilde{\bPsi}_t \bar{\bPsi}_t^{\top}} \bm{\theta}_t \\
    &= \psi(\bx^{\ast})^{\top} \tilde{G}_t(\lambda)^{-1} \rbr{\tilde{\bPsi}_t \bar{\bPsi}_t^{\top} \bQ_t^{\top} - \tilde{\bPsi}_t \bar{\bPsi}_t^{\top}} \bm{\theta}_t \\
    &= \psi(\bx^{\ast})^{\top} \tilde{G}_t(\lambda)^{-1} \tilde{\bPsi}_t \bar{\bPsi}_t^{\top} \rbr{\bQ_t - \bI_{|\mX|}} \bm{\theta}_t \\
    &\leq \|\psi(\bx^{\ast})\|_{\tilde{G}_t(\lambda)^{-1}}  \|\tilde{G}_t(\lambda)^{-1/2} \tilde{\bPsi}_t \bar{\bPsi}_t^{\top} \rbr{\bQ_t - \bI_{|\mX|}} \bm{\theta}_t \|_2 \\
    &\leq B \|\psi(\bx^{\ast})\|_{\tilde{G}_t(\lambda)^{-1}}  \|\tilde{G}_t(\lambda)^{-1/2} \tilde{\bPsi}_t\| \|\bar{\bPsi}_t^{\top} \rbr{\bQ_t - \bI_{|\mX|}} \| \\
    \label{eq:thm_ke3_proof_G}
    &\leq B \|\psi(\bx^{\ast})\|_{\tilde{G}_t(\lambda)^{-1}}  \|\bar{\bPsi}_t^{\top} \rbr{\bQ_t - \bI_{|\mX|}} \| \\
    \label{eq:thm_ke3_proof_ort}
    &\leq B \sqrt{\frac{2\lambda}{T}}\|\psi(\bx^{\ast})\|_{\tilde{G}_t(\lambda)^{-1}},
\end{align}
where Eq.~\eqref{eq:thm_ke3_proof_G} follows from $\|\tilde{G}_t(\lambda)^{-1/2} \tilde{\bPsi}_t\| = \sqrt{\|\tilde{\bPsi}_t^{\top} \tilde{G}_t(\lambda)^{-1} \tilde{\bPsi}_t\|} \leq 1$, and Eq.~\eqref{eq:thm_ke3_proof_ort} follows from 
$\|\bar{\bPsi}_t^{\top} \rbr{\bQ_t - \bI_{|\mX|}}\| = \sqrt{\|\bar{\bPsi}_t^{\top} \rbr{\bQ_t - \bI_{|\mX|}} \bar{\bPsi}_t\|} \leq \sqrt{\frac{2\lambda}{T}}$, which is implied by the fact that $\bar{\bPsi}_t^{\top} \rbr{\bI_{|\mX|} - \bQ_t} \bar{\bPsi}_t \preceq \frac{2\lambda}{T} \bar{\bPsi}_t^{\top} (\bar{\bPsi}_t \bar{\bPsi}_t^{\top} + \frac{\lambda}{T} \bI_{|\mX|})^{-1} \bar{\bPsi}_t \preceq \frac{2\lambda}{T} \bI_{|\mX|}$ holds from \lemref{lem:ort_proj_rel}. Thus, we have
\begin{align}
    &\Ep[f_t(\bx^{\ast}) - \tilde{u}_t(\bx^{\ast})] \\
    &\leq \Ep\sbr{\1\{A_t\} \rbr{f_t(\bx^{\ast}) - \tilde{u}_t(\bx^{\ast})}} + \Pr\rbr{A_t^c}\rbr{B + \frac{BT}{\lambda} + \beta \sqrt{\frac{T}{\lambda}}} \\
    &\leq \Ep\sbr{\1\{A_t\} \Ep\sbr{f_t(\bx^{\ast}) - \tilde{u}_t(\bx^{\ast}) \mid \tilde{\mH}_{t-1}, \bm{\theta}_t, \bS_t}} + \frac{1}{T^3}\rbr{B + \frac{BT}{\lambda} + \beta \sqrt{\frac{T}{\lambda}}} \\
    &\leq \Ep\sbr{\1\{A_t\} \rbr{B \sqrt{\frac{\lambda}{T}} \rbr{\sqrt{2} + 1} - \beta} \|\psi(\bx^{\ast})\|_{\tilde{G}_t(\lambda)^{-1}}} + \frac{1}{T^3}\rbr{B + \frac{BT}{\lambda} + \beta \sqrt{\frac{T}{\lambda}}} \\
    \label{eq:rke3_final_i}
    &= \frac{B}{T^3}\rbr{\frac{T}{\lambda} + \sqrt{2} + 2}.
\end{align}

\paragraph{Upper bound for the second term (ii) in Eq.~\eqref{eq:thm_rke3_decomp}.}
For any $\bx \in \mX$, the definition of the exploration distribution implies the following upper bound of $\tilde{u}_t(\bx)$ under event $A$ (Lemma~\ref{lem:ub_ut_tilde}):
\begin{equation}
    \tilde{u}_t(\bx) \leq \frac{2\sqrt{3}}{\alpha} \rbr{2(\sqrt{3} + \sqrt{2}) B \gamma_{T} + \beta \sqrt{\gamma_T}}.
\end{equation}
By combining the above inequality with the definition of $\alpha$, we can observe that $\eta \tilde{u}_t(\bx) \leq 1$ holds for any $\bx \in \mX$ under $A$; therefore, by following the standard proof for Exp3~(Lemma~\ref{lem:exp3_standard}), we obtain the following inequality under event $A$:
\begin{align}
    \label{eq:rke3_ut_put_first_ub}
    \sum_{t=1}^T \tilde{u}_t(\bx^{\ast}) - \sum_{t=1}^T \sum_{\bx \in \mX} \tilde{P}_t(\bx) \tilde{u}_t(\bx) 
    \leq \frac{\log |\mX|}{\eta} + \eta \sum_{t=1}^T \sum_{\bx \in \mX} \tilde{P}_t(\bx) \tilde{u}_t^2(\bx).
\end{align}
Regarding the term $\sum_{\bx \in \mX} \tilde{P}_t(\bx) \tilde{u}_t^2(\bx)$, we further obtain the upper bound as follows:
\begin{align}
    \sum_{\bx \in \mX} \tilde{P}_t(\bx) \tilde{u}_t^2(\bx) 
    &= \sum_{\bx \in \mX} \tilde{P}_t(\bx) \rbr{\tilde{f}_t(\bx) + \beta \|\psi(\bx)\|_{\tilde{G}_t(\lambda)^{-1}}}^2 \\
    \label{eq:rke3_ut_var_decomp}
    &\leq 2 \sum_{\bx \in \mX} \tilde{P}_t(\bx) \tilde{f}_t^2(\bx) + 2 \beta^2 \sum_{\bx \in \mX} \tilde{P}_t(\bx) \|\psi(\bx)\|_{\tilde{G}_t(\lambda)^{-1}}^2,
\end{align}
As with the proof of \thmref{thm:reg_exp3} (Eqs.~\eqref{eq:Ptilde_Gt_rel}--\eqref{eq:ut_var_decomp_ub2}), under event $A$, the second term in Eq.~\eqref{eq:rke3_ut_var_decomp} is bounded from above as follows:
\begin{equation}
    \label{eq:rke3_ut_var_decomp_ub2}
    2 \beta^2 \sum_{\bx \in \mX} \tilde{P}_t(\bx) \|\psi(\bx)\|_{\tilde{G}_t(\lambda)^{-1}}^2 \leq 
    6 \beta^2 \sum_{\bx \in \mX} \tilde{P}_t(\bx) \|\psi(\bx)\|_{G_t(\lambda)^{-1}}^2 \leq \frac{24 (1 + \sqrt{2})^2 \lambda B^2 \gamma_{T}}{(1-\alpha) T},
\end{equation}
where the first inequality follows from \lemref{lem:Ppsi_Qpsi} and event $A_t$.
Thus, we have
\begin{align}
    &\Ep\sbr{\1\{A\}\rbr{\sum_{t=1}^T \tilde{u}_t(\bx^{\ast}) - \sum_{t=1}^T \sum_{\bx \in \mX} \tilde{P}_t(\bx) \tilde{u}_t(\bx)} } \\
    &\leq \Ep\sbr{\1\{A\}\rbr{\frac{\log |\mX|}{\eta} + \eta \sum_{t=1}^T \sum_{\bx \in \mX} \tilde{P}_t(\bx) \tilde{u}_t^2(\bx)} } \\
    &\leq \frac{\log |\mX|}{\eta} + \Ep\sbr{\1\{A\}\rbr{ 2\eta \sum_{t=1}^T \sum_{\bx \in \mX} \tilde{P}_t(\bx) \tilde{f}_t^2(\bx) + 2 \eta \beta^2 \sum_{t=1}^T \sum_{\bx \in \mX} \tilde{P}_t(\bx) \|\psi(\bx)\|_{\tilde{G}_t(\lambda)^{-1}}^2} } \\
    \label{eq:rke3_A_ut_Ptilde_ut}
    &\leq \frac{\log |\mX|}{\eta} + \frac{24 \eta (1 + \sqrt{2})^2 \lambda B^2 \gamma_{T}}{1-\alpha} + 2\eta \sum_{t=1}^T \Ep\sbr{\1\{A\} \sum_{\bx \in \mX} \tilde{P}_t(\bx) \tilde{f}_t^2(\bx) }.
\end{align}
Furthermore, we have
\begin{align}
    &\Ep\sbr{\1\{A\} \sum_{\bx \in \mX} \tilde{P}_t(\bx) \tilde{f}_t^2(\bx) } \\ 
    &\leq \Ep\sbr{\1\{A_t\} \sum_{\bx \in \mX} \tilde{P}_t(\bx) \tilde{f}_t^2(\bx) } \\
    &\leq B^2 \Ep\sbr{\1\{A_t\} \sum_{\bx \in \mX} \tilde{P}_t(\bx) \sbr{\psi(\bx)^{\top} \tilde{G}_t(\lambda)^{-1} \tilde{\psi}_t(\bx_t)}^2} \\
    &= B^2 \Ep\sbr{\1\{A_t\} \sum_{\bx \in \mX} \tilde{P}_t(\bx) \psi(\bx)^{\top} \tilde{G}_t(\lambda)^{-1} \tilde{\psi}_t(\bx_t) \tilde{\psi}_t(\bx_t)^{\top} \tilde{G}_t(\lambda)^{-1} \psi(\bx)} \\
    &= B^2 \Ep\sbr{\Ep\sbr{\1\{A_t\} \sum_{\bx \in \mX} \tilde{P}_t(\bx) \psi(\bx)^{\top} \tilde{G}_t(\lambda)^{-1} \tilde{\psi}_t(\bx_t) \tilde{\psi}_t(\bx_t)^{\top} \tilde{G}_t(\lambda)^{-1} \psi(\bx) \mid \tilde{\mH}_{t-1}, \bS_t}} \\
    \label{eq:rke3_Apf_in}
    &= B^2 \Ep\sbr{\1\{A_t\} \sum_{\bx \in \mX} \tilde{P}_t(\bx) \psi(\bx)^{\top} \tilde{G}_t(\lambda)^{-1} \Ep\sbr{\tilde{\psi}_t(\bx_t) \tilde{\psi}_t(\bx_t)^{\top} \mid \tilde{\mH}_{t-1}, \bS_t } \tilde{G}_t(\lambda)^{-1} \psi(\bx)} \\
    \label{eq:rke3_exp_cond}
    &\leq B^2 \Ep\sbr{\1\{A_t\} \sum_{\bx \in \mX} \tilde{P}_t(\bx) \psi(\bx)^{\top} \tilde{G}_t(\lambda)^{-1}  \psi(\bx)} \\
    \label{eq:rke3_tildeG_G}
    &\leq 3 B^2 \Ep\sbr{\1\{A_t\} \sum_{\bx \in \mX} \tilde{P}_t(\bx) \psi(\bx)^{\top} G_t(\lambda)^{-1}  \psi(\bx)} \\
    \label{eq:rke3_APf2_last}
    &\leq \frac{12 B^2 \gamma_T}{1 - \alpha},
\end{align}
where Eq.~\eqref{eq:rke3_Apf_in} follows from the fact that $A_t$, $\tilde{P}_t(\cdot)$, $\tilde{G}_t(\lambda)$, and $\tilde{\psi}_t(\cdot)$ are deterministic given $\tilde{\mH}_{t-1}$, $\bS_t$. Furthermore, Eq.~\eqref{eq:rke3_exp_cond} follows from $\|\tilde{G}_t(\lambda)^{-1/2} \Ep\sbr{\tilde{\psi}_t(\bx_t) \tilde{\psi}_t(\bx_t)^{\top} \mid \tilde{\mH}_{t-1}, \bS_t } \tilde{G}_t(\lambda)^{-1/2}\| \leq 1$ , which is implied by $\Ep\sbr{\tilde{\psi}_t(\bx_t) \tilde{\psi}_t(\bx_t)^{\top} \mid \tilde{\mH}_{t-1}, \bS_t} = \sum_{\bx \in \mX} P_t(\bx) \tilde{\psi}_t(\bx) \tilde{\psi}_t(\bx)^{\top}$ and $\tilde{G}_t(\lambda) = \sum_{\bx \in \mX} P_t(\bx) \tilde{\psi}_t(\bx) \tilde{\psi}_t(\bx)^{\top} + \frac{\lambda}{T} \bI_{|\mX|}$. In addition, Eq.~\eqref{eq:rke3_tildeG_G} is implied by leveraging event $A_t$ with \lemref{lem:Ppsi_Qpsi}.
Hence, from Eqs.~\eqref{eq:rke3_A_ut_Ptilde_ut} and \eqref{eq:rke3_APf2_last}, we have
\begin{align}
    &\Ep\sbr{\sum_{t=1}^T \tilde{u}_t(\bx^{\ast}) - \sum_{t=1}^T \sum_{\bx \in \mX} \tilde{P}_t(\bx) \tilde{u}_t(\bx) } \\
    &\leq \Ep\sbr{\1\{A^c\} \rbr{\sum_{t=1}^T \tilde{u}_t(\bx^{\ast}) - \sum_{t=1}^T \sum_{\bx \in \mX} \tilde{P}_t(\bx) \tilde{u}_t(\bx)} } + \frac{\log |\mX|}{\eta} + \frac{24 \eta (1 + \sqrt{2})^2 \lambda B^2 \gamma_{T}}{1-\alpha} + \frac{24 \eta B^2 T \gamma_T}{1 - \alpha} \\
    \label{eq:rke3_Ac_uPu}
    &\leq 2 T\rbr{\frac{BT}{\lambda} + \beta \sqrt{\frac{T}{\lambda}}} \Pr(A^c) + \frac{\log |\mX|}{\eta} + \frac{24 \eta (1 + \sqrt{2})^2 \lambda B^2 \gamma_{T}}{1-\alpha} + \frac{24 \eta B^2 T \gamma_T}{1 - \alpha} \\
    \label{eq:rke3_final_ii}
    &\leq \frac{2B}{T} \rbr{\frac{T}{\lambda} + \sqrt{2} + 1} + \frac{\log |\mX|}{\eta} + \frac{24 \eta B^2 ((1 + \sqrt{2})^2 \lambda + T) \gamma_{T}}{1-\alpha},
\end{align}
where Eq.~\eqref{eq:rke3_Ac_uPu} follows from \lemref{lem:sup-norm_utilde}, and Eq.~\eqref{eq:rke3_final_ii} follows from $\Pr(A^c) \leq 1/T^2$ and the definition of $\beta$.

\paragraph{Upper bound for the third term (iii) in Eq.~\eqref{eq:thm_rke3_decomp}.}
From \lemref{lem:sup-norm_utilde} and the inequality $\Pr(A_t^c) \leq 1/T^3$, we have
\begin{align}
    &\Ep\sbr{\sum_{\bx \in \mX} \tilde{P}_t(\bx) \rbr{\tilde{u}_t(\bx) -  f_t(\bx)}} \\
    &= \Ep\sbr{\1\{A_t^c\}\sum_{\bx \in \mX} \tilde{P}_t(\bx) \rbr{\tilde{u}_t(\bx) -  f_t(\bx)}} + \Ep\sbr{\1\{A_t\}\sum_{\bx \in \mX} \tilde{P}_t(\bx) \rbr{\tilde{u}_t(\bx) -  f_t(\bx)}} \\
    &\leq \frac{1}{T^3} \rbr{B + \frac{BT}{\lambda} + \beta \sqrt{\frac{T}{\lambda}}}  + \Ep\sbr{\1\{A_t\}\sum_{\bx \in \mX} \tilde{P}_t(\bx) \rbr{\tilde{u}_t(\bx) -  f_t(\bx)}} \\
    \label{eq:rke3_iii_first_exp}
    &\leq \frac{B}{T^3} \rbr{\frac{T}{\lambda} + \sqrt{2} + 2}  + \Ep\sbr{\1\{A_t\}\sum_{\bx \in \mX} \tilde{P}_t(\bx) \rbr{\tilde{u}_t(\bx) -  f_t(\bx)}}.
\end{align}
Here, we can rewrite the second term as
\begin{align}
    &\Ep\sbr{\1\{A_t\}\sum_{\bx \in \mX} \tilde{P}_t(\bx) \rbr{\tilde{u}_t(\bx) -  f_t(\bx)}} \\
    \label{eq:rke3_iii_At_decomp}
    &= \Ep\sbr{\1\{A_t\}\sum_{\bx \in \mX} \tilde{P}_t(\bx) \rbr{\tilde{f}_t(\bx) -  f_t(\bx)}} + \beta \Ep\sbr{\1\{A_t\} \sum_{\bx \in \mX} \tilde{P}_t(\bx) \|\psi(\bx)\|_{\tilde{G}_t(\lambda)^{-1}}}.
\end{align}
Then, we have
\begin{align}
    &\Ep\sbr{\1\{A_t\} \sum_{\bx \in \mX} \tilde{P}_t(\bx) \rbr{\tilde{f}_t(\bx) - f_t(\bx)}} \\
    &= \Ep\sbr{\1\{A_t\} \sum_{\bx \in \mX} \tilde{P}_t(\bx) \psi(\bx)^{\top} \rbr{\tilde{G}_t(\lambda)^{-1} \tilde{\psi}_t(\bx_t) \psi(\bx_t)^{\top} - \bI_{|\mX|}} \bm{\theta}_t} \\
    &= \Ep\sbr{\1\{A_t\} \sum_{\bx \in \mX} \tilde{P}_t(\bx) \psi(\bx)^{\top} \sbr{\tilde{G}_t(\lambda)^{-1} \rbr{\tilde{\psi}_t(\bx_t) \psi(\bx_t)^{\top} - \tilde{G}_t(\lambda)}} \bm{\theta}_t} \\
    &= \Ep\sbr{\1\{A_t\} \sum_{\bx \in \mX} \tilde{P}_t(\bx) \psi(\bx)^{\top} \sbr{\tilde{G}_t(\lambda)^{-1} \rbr{ \Ep\sbr{\tilde{\psi}_t(\bx_t) \psi(\bx_t)^{\top} \mid \tilde{\mH}_{t-1}, \bS_t, \bm{\theta}_t} - \tilde{G}_t(\lambda)}} \bm{\theta}_t} \\
    &= \Ep\sbr{\1\{A_t\} \sum_{\bx \in \mX} \tilde{P}_t(\bx) \psi(\bx)^{\top} \sbr{\tilde{G}_t(\lambda)^{-1} \rbr{ \tilde{\bPsi}_t \bar{\bPsi}_t^{\top} - \tilde{\bPsi}_t \tilde{\bPsi}_t^{\top} - \frac{\lambda}{T} \bI_{|\mX|}}} \bm{\theta}_t} \\
    \label{eq:rke3_ftilde_f_decomp}
    \begin{split}
        &= \Ep\sbr{\1\{A_t\} \sum_{\bx \in \mX} \tilde{P}_t(\bx) \psi(\bx)^{\top} \tilde{G}_t(\lambda)^{-1} \tilde{\bPsi}_t \bar{\bPsi}_t^{\top} \rbr{\bI_{|\mX|} - \bQ_t} \bm{\theta}_t} \\
    &~~~~- \frac{\lambda}{T}\Ep\sbr{\1\{A_t\} \sum_{\bx \in \mX} \tilde{P}_t(\bx) \psi(\bx)^{\top} \tilde{G}_t(\lambda)^{-1} \bm{\theta}_t}.
    \end{split}
\end{align}
Similar to the derivation for the upper bound of the first term (i) in Eq.~\eqref{eq:thm_rke3_decomp}, we examine the upper bound of the above two terms. Regarding the first term, we have
\begin{align}
    &\Ep\sbr{\1\{A_t\} \sum_{\bx \in \mX} \tilde{P}_t(\bx) \psi(\bx)^{\top} \tilde{G}_t(\lambda)^{-1} \tilde{\bPsi}_t \bar{\bPsi}_t^{\top} \rbr{\bI_{|\mX|} - \bQ_t} \bm{\theta}_t} \\
    &\leq B\Ep\sbr{\1\{A_t\} \sum_{\bx \in \mX} \tilde{P}_t(\bx) \|\psi(\bx)\|_{\tilde{G}_t(\lambda)^{-1}} \left\|\tilde{G}_t(\lambda)^{-1/2} \tilde{\bPsi}_t \bar{\bPsi}_t^{\top} \rbr{\bI_{|\mX|} - \bQ_t} \right\|} \\
    \label{eq:rke3_ftilde_f_first}
    &\leq B \sqrt{\frac{2\lambda}{T}} \Ep\sbr{\1\{A_t\} \sum_{\bx \in \mX} \tilde{P}_t(\bx) \|\psi(\bx)\|_{\tilde{G}_t(\lambda)^{-1}}}.
\end{align}
Regarding the second term $-\frac{\lambda}{T}\Ep\sbr{\1\{A_t\} \sum_{\bx \in \mX} \tilde{P}_t(\bx) \psi(\bx)^{\top} \tilde{G}_t(\lambda)^{-1} \bm{\theta}_t}$, we have
\begin{align}
    -\frac{\lambda}{T}\Ep\sbr{\1\{A_t\} \sum_{\bx \in \mX} \tilde{P}_t(\bx) \psi(\bx)^{\top} \tilde{G}_t(\lambda)^{-1} \bm{\theta}_t} 
    &\leq \frac{\lambda}{T}\Ep\sbr{\1\{A_t\} \sum_{\bx \in \mX} \tilde{P}_t(\bx) |\psi(\bx)^{\top} \tilde{G}_t(\lambda)^{-1} \bm{\theta}_t|} \\
    &\leq  \frac{B\lambda}{T}\Ep\sbr{\1\{A_t\} \sum_{\bx \in \mX} \tilde{P}_t(\bx) \|\psi(\bx)^{\top} \tilde{G}_t(\lambda)^{-1}\|_2} \\
    \label{eq:rke3_ftilde_f_second}
    &\leq B\sqrt{\frac{\lambda}{T}} \Ep\sbr{\1\{A_t\} \sum_{\bx \in \mX} \tilde{P}_t(\bx) \|\psi(\bx)\|_{\tilde{G}_t(\lambda)^{-1}}}.
\end{align}
Hence, by aggregating Eqs.~\eqref{eq:rke3_iii_first_exp},~\eqref{eq:rke3_iii_At_decomp},~\eqref{eq:rke3_ftilde_f_decomp},~\eqref{eq:rke3_ftilde_f_first}, and \eqref{eq:rke3_ftilde_f_second}, we have
\begin{align}
    &\Ep\sbr{\sum_{\bx \in \mX} \tilde{P}_t(\bx) \rbr{\tilde{u}_t(\bx) -  f_t(\bx)}} \\
    &\leq \frac{B}{T^3} \rbr{\frac{T}{\lambda} + \sqrt{2} + 2} + 2(1 + \sqrt{2}) B \sqrt{\frac{\lambda}{T}} \Ep\sbr{\1\{A_t\} \sum_{\bx \in \mX} \tilde{P}_t(\bx) \|\psi(\bx)\|_{\tilde{G}_t(\lambda)^{-1}}} \\
    &\leq \frac{B}{T^3} \rbr{\frac{T}{\lambda} + \sqrt{2} + 2} + 2\sqrt{3}(1 + \sqrt{2}) B \sqrt{\frac{\lambda}{T}} \Ep\sbr{\1\{A_t\} \sum_{\bx \in \mX} \tilde{P}_t(\bx) \|\psi(\bx)\|_{G_t(\lambda)^{-1}}} \\
    \label{eq:rke3_final_iii}
    &\leq \frac{B}{T^3} \rbr{\frac{T}{\lambda} + \sqrt{2} + 2} + 4  \sqrt{3}(1 + \sqrt{2}) B \sqrt{\frac{\lambda \gamma_T}{(1-\alpha)T}}.
\end{align}

\paragraph{Aggregate all the upper bounds.}
From Eqs.~\eqref{eq:decomp_reg_rls_exp3}, ~\eqref{eq:thm_rke3_decomp},~\eqref{eq:rke3_final_i},~\eqref{eq:rke3_final_ii}, and \eqref{eq:rke3_final_iii}, we finally obtain
\begin{align}
    \begin{split}
    \bar{R}_T &\leq 2\alpha B T + (1-\alpha) \Biggl[\frac{B}{T^2}\rbr{\frac{T}{\lambda} + \sqrt{2} + 2} + \frac{2B}{T} \rbr{\frac{T}{\lambda} + \sqrt{2} + 1} \\
    &~~~+ \frac{\log |\mX|}{\eta} + \frac{24 \eta B^2 ((1 + \sqrt{2})^2 \lambda + T) \gamma_{T}}{1-\alpha} + \frac{B}{T^2} \rbr{\frac{T}{\lambda} + \sqrt{2} + 2} + 4 \sqrt{3}(1 + \sqrt{2}) B \sqrt{\frac{\lambda T \gamma_T}{1-\alpha}}\Biggr]
    \end{split} \\
    \begin{split}
    &\leq 2\alpha B T + \frac{2B}{T^2}\rbr{\frac{T}{\lambda} + \sqrt{2} + 2} + \frac{2B}{T} \rbr{\frac{T}{\lambda} + \sqrt{2} + 1} \\
    &~~~+ \frac{\log |\mX|}{\eta} + 24 \eta B^2 ((1 + \sqrt{2})^2 \lambda + T) \gamma_{T} + 4 \sqrt{3}(1 + \sqrt{2}) B \sqrt{\lambda T \gamma_T}.
    \end{split}
\end{align} 
In the above upper bound, the additional term $\frac{2B}{T^2}\rbr{\frac{T}{\lambda} + \sqrt{2} + 2} + \frac{2B}{T} \rbr{\frac{T}{\lambda} + \sqrt{2} + 1}$ is $O(1)$ and does not affect the growth rate of regret. Furthermore, the remaining term: $2\alpha B T + \frac{\log |\mX|}{\eta}+ 24 \eta B^2 ((1 + \sqrt{2})^2 \lambda + T) \gamma_{T} + 4 \sqrt{3} (1 + \sqrt{2}) B \sqrt{\lambda T \gamma_T}$ is the same as the upper bound in \thmref{thm:reg_exp3} up to constant factors. Thus, we complete the proof.
\qed

\section{Proof of \thmref{thm:lb}}
\label{app:proof_lb}

We define the noise function $g_t$ as $g_t(\cdot) = \frac{\eta_t B}{2 \sqrt{2 \log (2T^2)}} k(\bm{0}, \cdot)$, where $\eta_t \sim \mN(0, 1)$. We assume that $(\eta_t)_{t \in \N_+}$ is mutually-independent. Then, we consider the following two problem instances (i.e., strategies of the environment) based on some common function $f \in \mF_k(B/2) \coloneqq \{f \in \mH_k \mid \|f\|_k \leq B/2\}$ and $(\eta_t)_{t \in \N_+}$:
\begin{itemize}
    \item \textbf{Problem instance 1}. The environment defines $f_t$ as  $f_t(\cdot) = f(\cdot) + g_t(\cdot)$.
    \item \textbf{Problem instance 2}. The environment defines $f_t$ as  $f_t(\cdot) = f(\cdot) + \1\{|\eta_t| < \sqrt{2 \log (2T^2)} \} g_t(\cdot) + \1\{|\eta_t| \geq \sqrt{2 \log (2T^2)} \} \text{sign}(\eta_t) \frac{B}{2} k(\bm{0}, \cdot)$.
\end{itemize}
In the above definition, the problem instance 1 does not always satisfy $\|f_t\|_k \leq B$. On the other hand, for any realization, problem instance 2 always satisfies $\|f_t\|_k \leq B$. Next, we derive a lower bound for problem instance 2 using problem instance 1. Here, we use the notations $\Ep^{(1)}[\cdot]$ and $\Pr^{(1)}(\cdot)$ as the expectation and probability under problem instance 1. Similarly, we also define $\Ep^{(2)}[\cdot]$ and $\Pr^{(2)}(\cdot)$. Then, for any algorithm, we have
\begin{align}
    \sup_{\bx \in \mX}\Ep^{(2)}[R_T(\bx)] 
    &= \sup_{\bx \in \mX}\Ep^{(2)}\sbr{\sum_{t=1}^T f_t(\bx) - f_t(\bx_t)} \\
    \label{eq:delete_residual}
    &= \sup_{\bx \in \mX}\Ep^{(2)}\sbr{\sum_{t=1}^T f(\bx) - f(\bx_t)} \\
    \label{eq:ind_lb}
    &\geq \sup_{\bx \in \mX}\Ep^{(2)}\sbr{ \1\cbr{\forall i \in [T], |\eta_i| < \sqrt{2 \log (2T^2)}} \sum_{t=1}^T (f(\bx) - f(\bx_t))} \\
    \label{eq:model_change}
    &= \sup_{\bx \in \mX}\Ep^{(1)}\sbr{ \1\cbr{\forall i \in [T], |\eta_i| < \sqrt{2 \log (2T^2)}} \sum_{t=1}^T (f(\bx) - f(\bx_t))} \\
    \label{eq:lb_rkhsnorm}
    &\geq \sup_{\bx \in \mX}\Ep^{(1)}\sbr{\sum_{t=1}^T f(\bx) - f(\bx_t)} - BT \Pr^{(1)}\rbr{\exists t \in [T],~|\eta_t| \geq \sqrt{2 \log (2T^2)}},
\end{align}
where:
\begin{itemize}
    \item Eq.~\eqref{eq:delete_residual} follows from the symmetry of the Gaussian distribution.  Namely, $\Ep[\1\{|\eta_t| < \sqrt{2 \log (2T^2)} \} g_t(\bx) + \1\{|\eta_t| \geq \sqrt{2 \log (2T^2)} \} \text{sign}(\eta_t) \frac{B}{2} k(\bm{0}, \bx)] = 0$ for any fixed $\bx \in \mX$.
    Furthermore, the independence between $\bx_t$ and $\eta_t$ also implies $\Ep[\1\{|\eta_t| < \sqrt{2 \log (2T^2)} \} g_t(\bx_t) + \1\{|\eta_t| \geq \sqrt{2 \log (2T^2)} \} \text{sign}(\eta_t) \frac{B}{2} k(\bm{0}, \bx_t)] = 0$.
    \item Eq.~\eqref{eq:ind_lb} follows from the fact that $\sum_{t=1}^T f(\bx^{\ast}) - f(\bx_t) \geq 0$ always holds for the maximizer $\bx^{\ast}$ of $f$. (Note that such $\bx^{\ast}$ always exists since $\mX$ is compact and $f$ is continuous under either the SE or Mat\'ern kernel.)
    \item Eq.~\eqref{eq:model_change} uses the fact that the expectation of $\1\cbr{\forall i \in [T], |\eta_i| < \sqrt{2 \log (2T^2)}} \sum_{t=1}^T (f(\bx) - f(\bx_t))$ under instance $1$ equals that under instance $2$ from the definitions. To see this rigorously, we first assume that the learner's algorithm is deterministic. Then, we can rewrite $\bx_t$ as the function of $\bm{\eta}_{t-1} \coloneqq (\eta_1, \ldots, \eta_{t-1})$ such as $\bx_t = \bx_t(\bm{\eta}_{t-1})$. Then, we can write 
    \begin{align}
        &\Ep^{(2)}\sbr{\1\cbr{\forall i \in [T], |\eta_i| < \sqrt{2 \log (2T^2)}} \sum_{t=1}^T (f(\bx) - f(\bx_t))} \\
        &= \int_{\bm{\eta}_{T} \in (-\sqrt{2 \log (2T^2)}, \sqrt{2 \log (2T^2)})^{T}} \sum_{t=1}^T \sbr{f(\bx) - f(\bx_t(\bm{\eta}_{t-1}))} \Pr_{\bm{\eta}_{T}}(\text{d}\bm{\eta}_{T}).
    \end{align}
    Note that the distribution of $\Pr_{\bm{\eta}_{T-1}}(\cdot)$ is the same across instances 1 and 2, and $\bx_t(\bm{\eta}_{t-1})$ is also common under event $\bm{\eta}_{T-1} \in (-\sqrt{2 \log (2T^2)}, \sqrt{2 \log (2T^2)})^{T-1}$, since the learner's history $\mH_{t-1} \coloneqq( \bx_1, f_1(\bx_1), \ldots, \bx_{t-1}, f_{t-1}(\bx_{t-1}))$ is the same in the two instances on this event. 
    Thus, we can confirm $\Ep^{(2)}\sbr{\1\cbr{\forall i \in [T], |\eta_i| < \sqrt{2 \log (2T^2)}} \sum_{t=1}^T (f(\bx) - f(\bx_t))} = \Ep^{(1)}\sbr{\1\cbr{\forall i \in [T], |\eta_i| < \sqrt{2 \log (2T^2)}} \sum_{t=1}^T (f(\bx) - f(\bx_t))}$. Finally, for a stochastic algorithm of the learner, the same argument is valid by rewriting $\bx_t$ as the function form: $\bx_t(\mH_{t-1}, (z_i)_{i \in [t]})$, where $z_i \sim \mathrm{Unif}(\mX)$ is used to represent additional randomness (see, e.g., Definition 1 in \citep{shekhar2022instance}).
    \item Eq.~\eqref{eq:lb_rkhsnorm} follows from $\Ep^{(1)}[\1\{\exists t \in [T],~|\eta_t| \geq \sqrt{2 \log (2T^2)}\}\sum_{t=1}^T(f(\bx^{\ast}) - f(\bx_t))] \leq BT \Ep^{(1)}[\1\{\exists t \in [T],~|\eta_t| \geq \sqrt{2 \log (2T^2)}\}]$, which is implied by $\|f\|_{\infty} \leq \|f\|_k \leq B/2$. 
\end{itemize}
Regarding the first term in Eq.~\eqref{eq:lb_rkhsnorm}, there exists a function $f \in \mF_k(B/2)$ such that $\sup_{\bx \in \mX}\Ep^{(1)}\sbr{\sum_{t=1}^T f(\bx) - f(\bx_t)} = \Omega(\underline{R}(T, \sigma_T^2))$, where $\sigma_T^2 = \frac{B^2}{8 \log (2T^2)}$ (see Lemma~\ref{lem:lower}).
Here, we defined $\underline{R}(T, \sigma^2)$ as $\underline{R}(T, \sigma^2) = \sqrt{\sigma^{2} T(\log (T/\sigma^2))^{d/2}}$ and $\underline{R}(T, \sigma^2) = \sigma^{\frac{2\nu}{2\nu+d}} T^{\frac{\nu+d}{2\nu+d}}$ for $k = \sek$ and $k = \matk$, respectively.
Regarding the second term, we use the union bound with the fact $\mathbb{P}(|\eta_t| \geq \sqrt{2 \log (2T^2)}) \leq 1/T^2$, which follows from the basic concentration inequality: $\forall m \geq 0, \mathbb{P}(|\eta_t| \geq m) \leq 2\exp(-m^2/2)$. Then, we obtain $BT \Pr^{(1)}\rbr{\exists t \in [T],~|\eta_t| \geq \sqrt{2 \log (2T^2)}} \leq B$. Thus, the second term is negligible compared to the first term since $B = \Theta(1)$.
Consequently, we conclude that there exists a problem instance (strategy of the environment) such that
\begin{equation}
    \sup_{\bx \in \mX}\Ep^{(2)}[R_T(\bx)] = \Omega(\underline{R}(T, \sigma_T^2)).
\end{equation}
By simplifying the constant factors in the above equation, we obtain the desired results. 
\qed

\section{RLS-kernelized Exp3 with approximated MVR-sequence}
\label{app:rke3_approx_mvr}

As described in the paragraph ``Remarks on exploration distribution'' of \secref{sec:exp3_nystrom_approx}, by applying recursive RLS-Nystr\"om, we can alleviate the computational cost of preparation of the MVR-based exploration distribution. \algoref{alg:exp3_nystrom_approx_mvr} provides this variant of RLS-kernelized Exp3. The difference from the original RLS-kernelized Exp3 is the definition of the exploration distribution $\tilde{\pi}$, which is calculated by the approximated MVR algorithm in \algoref{algo:approx_mvr}. 
With proper modification of the parameters and proof, we can obtain the following theorem, which provides the improved computational cost for the preparation of the exploration distribution, while maintaining the same per-round computational cost and regret upper bound in \thmref{thm:reg_exp3_nystrom_approx}.

\begin{theorem}[Regret upper bound and computational cost for RLS-kernelized Exp3 with approximated MVR sequence]
    \label{thm:reg_rke3_approx_mvr}
    Suppose that Assumption~\ref{asmp:regularity} holds. Assume that $\mX$ is finite and $|\mX| = o(\exp(T\gamma_T^{-1}))$ holds. Furthermore, suppose that $d$ and $B$ are $\Theta(1)$. Set algorithm parameters as $\beta = B(1 + \sqrt{2})\sqrt{\lambda/T}$, $\lambda \geq 1$, $\eta = \Theta(\sqrt{\log |\mX|}/\sqrt{T \gamma_T})$, $\alpha =  \eta \rbr{36(3 + \sqrt{6}) B \gamma_{T} + \beta \sqrt{108\gamma_T}}$, and $\delta = 1/T^2$. Then, when running \algoref{alg:exp3_nystrom_approx_mvr}, we have $\bar{R}_T = O(\sqrt{T \gamma_T \log |\mX|})$.
    Furthermore, with probability at least $1 - 1/T^2$, the following two statements hold simultaneously:
    \begin{itemize}
        \item For any $t \in [T]$, the computational cost at round $t$ is $\tilde{O}(|\mX|\gamma_T^2 + \gamma_T^3)$.
        \item The computational cost for calculating approximated MVR sequence-based exploration distribution $\tilde{\pi}$ is $\tilde{O}\rbr{\sum_{t=1}^{\lceil T\alpha \rceil} (t\gamma_t^2 + \gamma_t^3 + |\mX|\gamma_t^2)}$.
    \end{itemize}
\end{theorem}

\paragraph{Computational cost.} To compare the costs of preparing the exploration distribution, fix $\lambda = \Theta(1)$ and assume $T \ll |\mX| \leq T^c$ for some fixed $c > 1$. Suppose also that $\gamma_t = \tilde{O}(t^a)$ for some fixed $a \in [0, 1)$; this holds for the SE kernel with $a = 0$ and for the $\nu$-Mat\'ern kernel with $a = d/(2\nu+d)$.
Under these assumptions and the parameter settings in \thmref{thm:reg_rke3_approx_mvr}, $\alpha = \Theta\rbr{\sqrt{\gamma_T \log |\mX|/T}}$. Let $M \coloneqq \lceil T\alpha \rceil = \Theta\rbr{\sqrt{T\gamma_T \log |\mX|}}$. Since $M \leq T$ for sufficiently large $T$ and $\gamma_M = O(M)$, we obtain
\begin{equation}
    \tilde{O}\rbr{\sum_{t=1}^{M} (t\gamma_t^2 + \gamma_t^3 + |\mX|\gamma_t^2)} = \tilde{O}\rbr{|\mX| M \gamma_M^2} = \tilde{O}\rbr{|\mX| M^{1+2a}} = o\rbr{|\mX| (T\gamma_T)^{3/2}}.
\end{equation}
The last relation follows from $a < 1$ and $\log |\mX| = O(\log T)$. Thus, under the stated assumptions, the resulting computational cost bound is smaller than the $O(|\mX|M^3)$ bound for preparing the exploration distribution using exact MVR.

\paragraph{Summary of the proof.} The core lemma for the proof of \thmref{thm:reg_rke3_approx_mvr} is \lemref{lem:prop_approx_mvr}, which guarantees that the posterior variance under the approximated MVR sequence diminishes at the same rate as that under the exact MVR sequence. See \lemref{lem:prop_approx_mvr} and its proof for details.
Once we establish \lemref{lem:prop_approx_mvr}, \thmref{thm:reg_rke3_approx_mvr} follows by applying the same proof as that of \thmref{thm:reg_exp3_nystrom_approx} with slight modifications.
Thus, we omit the proof of \thmref{thm:reg_rke3_approx_mvr}, while we summarize the required modifications below:
\begin{itemize}
\item In addition to events $(A_t)$ in Eq.~\eqref{eq:thm_rke3_proof_eps_acc}, we consider additional event $A_0$ such that the three events in \corref{cor:acc_comp_rls_standard} with $\tilde{\bS} = \tilde{\bS}_{t}$, $\bPsi = \bPsi_{t-1}^{(\mathrm{MVR})}$, and $\tilde{s} = \tilde{s}_t$ hold for all rounds $t \in [\lceil T\alpha\rceil]$ in the approximated MVR procedures (Line 1 in \algoref{alg:exp3_nystrom_approx_mvr}).
Then, we also define the modified event $\bar{A} \coloneqq \cap_{t \in [T]\cup \{0\}} A_t$ as the replacement of event $A$ in the original proof. Then, note that $\Pr(A_0) \geq 1 - 1/(2T^2)$ and $\Pr(\cap_{t \in [T]} A_t) \geq 1 - 1/(2T^2)$ hold due to Lemma~\ref{lem:acc_comp_rls}, \corref{cor:acc_comp_rls_standard}, and the settings of the confidence level parameters (Lines 1 and 7 in \algoref{alg:exp3_nystrom_approx_mvr}). Thus, the union bound implies $\Pr(\bar{A}^c) \leq 1/T^2$, which matches the upper bound $\Pr(A^c)$ used in the proof of the original RLS-kernelized Exp3.
\item The proof for the upper bounds for terms (i) and (iii) in Eq.~\eqref{eq:thm_rke3_decomp} remains unchanged from the original proof except for the replacement of the constant factor arising from the difference in $\Pr(A_t^c)$. (In the original proof, we have $\Pr(A_t^c) \leq 1/T^3$, while $\Pr(A_t^c) \leq 1/(2T^3)$ under \algoref{alg:exp3_nystrom_approx_mvr}.)
\item Regarding the upper bounds for term (ii) in Eq.~\eqref{eq:thm_rke3_decomp}, we can leverage \lemref{lem:ub_ut_tilde_approx_mvr} under event $\bar{A}$ instead of using \lemref{lem:ub_ut_tilde} under event $A$. Note that event: $\cbr{\max_{\bx \in \mX} \sigma^2\rbr{\bx; \tilde{\bX}_{\lceil T\alpha \rceil}^{(\text{MVR})}, \lambda} \leq \frac{36 \lambda \gamma_{\lceil T\alpha \rceil}}{\lceil T\alpha \rceil}}$, which is required for the application of \lemref{lem:ub_ut_tilde_approx_mvr}, is true under event $A_0$.
Then, from \lemref{lem:ub_ut_tilde_approx_mvr}, for guaranteeing $\eta \tilde{u}_t(\bx) \leq 1$, we must set $\alpha$ such that $\frac{\eta}{\alpha} \rbr{36(3 + \sqrt{6}) B \gamma_{T} + \beta \sqrt{108\gamma_T}} \leq 1$ holds. The setting of $\alpha$ in \thmref{thm:reg_rke3_approx_mvr} satisfies this condition; thus, we can apply the standard Exp3 analysis in \lemref{lem:exp3_standard} under event $\bar{A}$. After that, we can directly follow the original proof of the upper bound for term (ii). By aggregating the upper bounds for terms (i), (ii), and (iii), we can obtain the desired regret guarantees.
\item Finally, we prove the statements regarding the computational costs. As with the original RLS-kernelized Exp3, the upper bound for the per-round computational cost (the first statement in \thmref{thm:reg_rke3_approx_mvr}) directly follows from \lemref{lem:acc_comp_rls}. Regarding the computational cost for the preparation of the exploration distribution, under event $A_0$, we can conclude the following two facts: (i) the computational cost of recursive RLS-Nystr\"om is $\tilde{O}(t\gamma_t^2)$, and (ii) the dimension of the Nystr\"om embedding $\tilde{s}_t$ is $\tilde{O}(\gamma_t)$, which implies the computation of Eq.~\eqref{eq:nystrom_approx_mvr_expression} for all $\bx \in \mX$ is $\tilde{O}(|\mX| \gamma_t^2 + \gamma_t^3)$. Thus, the computational cost for approximated MVR procedures at round $t$ is $\tilde{O}(t\gamma_t^2 + |\mX| \gamma_t^2 + \gamma_t^3)$, which implies the desired statement.
\end{itemize}

\section{Helper lemmas}
\label{app:helper_lemma}
\subsection{Basic properties of feature map, RKHS, and GP}
\paragraph{Overview.} The lemmas in this subsection are well-known, useful properties for handling the feature representation of the kernel~\citep{chowdhury2021no,valko2013finite}, while we provide the proof for completeness.
\begin{lemma}[Feature map and RKHS norm]
    \label{lem:feature_map_rkhs}
    Let $k: \mX \times \mX \rightarrow \R$ be a positive definite kernel function on a finite input domain $\mX$. Let $\psi: \mX \rightarrow \R^{|\mX|}$ be any feature map that satisfies $k(\bx, \bx') = \psi(\bx)^{\top} \psi(\bx')$ for all $\bx, \bx' \in \mX$. Fix any $f \in \mH_k$. Then, there exists a vector $\bm{\theta} \in \R^{|\mX|}$ such that $f(\cdot) = \psi(\cdot)^{\top} \bm{\theta}$. Furthermore, $\|\bm{\theta}\|_2 = \|f\|_{k}$ holds.
\end{lemma}
\begin{proof}
    Since $f \in \mH_k$ and $|\mX| < \infty$, the function $f$ is represented as $f(\cdot) = \sum_{\bx \in \mX} c_{\bx} k(\bx, \cdot)$ with some coefficients $c_{\bx} \in \R$. In addition, the RKHS norm $\|f\|_k$ is defined as $\|f\|_k = \sqrt{\sum_{\bx, \bx' \in \mX} c_{\bx} c_{\bx'} k(\bx, \bx')}$~\citep[e.g., see Chap.~2.3 in][]{kanagawa2018gaussian}.
    By using the feature map, we have $f(\cdot) = \psi(\cdot)^{\top} (\sum_{\bx \in \mX} c_{\bx} \psi(\bx))$. This implies that we can set $\bm{\theta} = \sum_{\bx \in \mX} c_{\bx} \psi(\bx)$. 
    Furthermore, we have
    \begin{align}
        \|\bm{\theta}\|_2 
        &= \sqrt{\sum_{j=1}^{|\mX|} \rbr{\sum_{\bx \in \mX} c_{\bx} \psi_j(\bx)}^2} 
        = \sqrt{\sum_{j=1}^{|\mX|} \sum_{\bx, \bx' \in \mX} c_{\bx} c_{\bx'} \psi_j(\bx) \psi_j(\bx')} \\
        &= \sqrt{\sum_{\bx, \bx' \in \mX} c_{\bx} c_{\bx'} \rbr{\sum_{j=1}^{|\mX|} \psi_j(\bx) \psi_j(\bx')}} 
        = \sqrt{\sum_{\bx, \bx' \in \mX} c_{\bx} c_{\bx'} \psi(\bx)^{\top} \psi(\bx')} \\
        &= \sqrt{\sum_{\bx, \bx' \in \mX} c_{\bx} c_{\bx'} k(\bx, \bx')} = \|f\|_k.
    \end{align}
    Thus, we complete the proof.
\end{proof}

\begin{lemma}[Basic matrix identities]
    \label{lem:rewrite_inv}
    Fix any $n, m \in \N_+$ and $\lambda > 0$. 
    Then, for any matrix $\bA \in \R^{n \times m}$, the following two identities hold:
    \begin{align}
        \label{eq:first_identity}
        (\bA \bA^{\top} + \lambda \bI_n)^{-1} \bA = \bA (\bA^{\top}\bA  + \lambda \bI_m)^{-1}, \\
        \bI_{m} - \bA^{\top} (\bA \bA^{\top} + \lambda \bI_n)^{-1} \bA = \lambda (\bA^{\top}\bA  + \lambda \bI_m)^{-1}.
    \end{align}
\end{lemma}
\begin{proof}
    The first identity is implied by multiplying 
    $(\bA^{\top}\bA  + \lambda \bI_m)^{-1}$ and $(\bA \bA^{\top} + \lambda \bI_n)^{-1}$ from both sides of the equation below:
    \begin{equation}
        \bA (\bA^{\top}\bA  + \lambda \bI_m) = \bA \bA^{\top}\bA  + \lambda \bA = (\bA \bA^{\top} + \lambda \bI_n) \bA.
    \end{equation}
    Regarding the second identity, we have
    \begin{align}
        \bI_{m} - \bA^{\top} (\bA \bA^{\top} + \lambda \bI_n)^{-1} \bA
        &= (\bA^{\top} \bA + \lambda \bI_{m})^{-1} (\bA^{\top} \bA + \lambda \bI_{m}) - (\bA^{\top} \bA + \lambda \bI_{m})^{-1} \bA^{\top} \bA \\
        &= \lambda (\bA^{\top} \bA + \lambda \bI_{m})^{-1},
    \end{align}
    where the first equality follows from Eq.~\eqref{eq:first_identity}.
\end{proof}

\begin{lemma}[Feature map and posterior variance of GP]
    \label{lem:feature_map_gp}
    Let $k: \mX \times \mX \rightarrow \R$ be a positive definite kernel function on a finite input domain $\mX$. Let $\psi: \mX \rightarrow \R^{|\mX|}$ be any feature map that satisfies $k(\bx, \bx') = \psi(\bx)^{\top} \psi(\bx')$ for all $\bx, \bx' \in \mX$. Furthermore, let $\bX \coloneqq (\bx^{(1)}, \ldots, \bx^{(t)})$ be any training inputs of the GP, which satisfies $\bx^{(i)} \in \mX$ for any $i \in [t]$. Assume that the variance parameter $\lambda$ of the GP is strictly positive.
    Then, for any $\bx, \bx' \in \mX$, the posterior variance in Eq.~\eqref{eq:posterior_var} is represented as
    \begin{align}
    \sigma^2(\bx; \bX, \lambda) &= \lambda \psi(\bx)^{\top} \left(\sum_{i=1}^t \psi(\bx^{(i)})\psi(\bx^{(i)})^{\top} + \lambda \bI_{|\mX|}\right)^{-1} \psi(\bx) \\
    &= \lambda \psi(\bx)^{\top} \left(\bPsi \bPsi^{\top} + \lambda \bI_{|\mX|}\right)^{-1} \psi(\bx),
    \end{align}
    where $\bPsi = (\psi(\bx^{(1)}), \ldots, \psi(\bx^{(t)})) \in \R^{|\mX| \times t}$.
\end{lemma}
\begin{proof}
    The desired statement can be obtained by the following basic matrix calculus:
    \begin{align}
        \sigma^2(\bx; \bX, \lambda)
        &= k(\bx, \bx) - \bk(\bx, \bX)^{\top} (\bK(\bX, \bX) + \lambda \bI_t)^{-1} \bk(\bx, \bX) \\
        &= \psi(\bx)^{\top} \psi(\bx) - 
        (\bPsi^{\top} \psi(\bx))^{\top} (\bPsi^{\top} \bPsi + \lambda \bI_t)^{-1} \bPsi^{\top} \psi(\bx) \\
        &= \psi(\bx)^{\top} \psi(\bx) - 
        \psi(\bx)^{\top} \bPsi (\bPsi^{\top} \bPsi + \lambda \bI_t)^{-1} \bPsi^{\top} \psi(\bx) \\
        \label{eq:matrix_var_convert}
        &= \psi(\bx)^{\top} \psi(\bx) - 
        \psi(\bx)^{\top} (\bPsi \bPsi^{\top} + \lambda \bI_{|\mX|})^{-1} \bPsi \bPsi^{\top} \psi(\bx) \\
        &= \psi(\bx)^{\top} \sbr{\bI_{|\mX|} - (\bPsi \bPsi^{\top} + \lambda \bI_{|\mX|})^{-1} \bPsi \bPsi^{\top} } \psi(\bx) \\
        &= \psi(\bx)^{\top} (\bPsi \bPsi^{\top} + \lambda \bI_{|\mX|})^{-1} \sbr{(\bPsi \bPsi^{\top} + \lambda \bI_{|\mX|}) - \bPsi \bPsi^{\top}}  \psi(\bx) \\
        &= \lambda \psi(\bx)^{\top} (\bPsi \bPsi^{\top} + \lambda \bI_{|\mX|})^{-1} \psi(\bx).
    \end{align}
    In the above expressions, Eq.~\eqref{eq:matrix_var_convert} follows from \lemref{lem:rewrite_inv}.
\end{proof}

\subsection{Helper lemmas for Kernelized Exp3}
\paragraph{Overview.} The lemmas in this subsection are leveraged in the proof of kernelized Exp3 in \appref{app:proof_reg_exp3}. \lemref{lem:exp3_standard} is an adaptation of the standard Exp3 analysis. For completeness, we provide the full proof.
\lemref{lem:ub_ut} is the upper bound for the optimistic estimator, which is used to guarantee the required condition to apply \lemref{lem:exp3_standard}. 
In the proof of \lemref{lem:ub_ut}, we leverage the known property of posterior variance regarding the MVR sequence, which can be found in the existing stochastic KB literature~\citep{cai2021lenient,li2022gaussian,vakili2021optimal}.

\begin{lemma}[Standard Exp3 analysis, adapted from e.g., the proof of Theorem 11.1 in \citep{lattimore2020bandit}]
\label{lem:exp3_standard}
    Fix any $T \in \N_+$, $\eta > 0$, and finite input domain $\mX$. Furthermore, for any $t \in [T]$, let $u_t: \mX \rightarrow \R$ be a function that satisfies $\eta u_t(\bx) \leq 1$ for any $\bx \in \mX$.
    Then, for any $\bz \in \mX$, the following inequality holds:
    \begin{equation}
        \sum_{t=1}^T u_t(\bz) - \sum_{t=1}^T \sum_{\bx \in \mX} \tilde{P}_t(\bx) u_t(\bx) 
        \leq \frac{\log |\mX|}{\eta} + \eta \sum_{t=1}^T \sum_{\bx \in \mX} \tilde{P}_t(\bx) u_t^2(\bx),
\end{equation}
where $\tilde{P}_{1}(\bx) = 1/|\mX|$ and
\begin{equation}
    \tilde{P}_{t+1}(\bx) = \frac{\exp\rbr{\eta \sum_{i=1}^t u_i(\bx)}}{\sum_{\bx' \in \mX} \exp\rbr{\eta \sum_{i=1}^t u_i(\bx')}}
\end{equation}
for any $t \in [T-1]$.
\end{lemma}
\begin{proof}
    Let us define $w_j$ as $w_j = \sum_{\bx \in \mX} \exp\rbr{\eta \sum_{t=1}^j u_t(\bx)}$ for any $j \in [T]$. Furthermore, we define $w_0 = |\mX|$. Then, we have
    \begin{equation}
        \label{eq:exp_ratio_eq}
        \exp\rbr{\eta \sum_{t=1}^T u_t(\bz)}
        \leq \sum_{\bx \in \mX} \exp\rbr{\eta \sum_{t=1}^T u_t(\bx)} = w_T = w_0 \prod_{j=1}^T \frac{w_j}{w_{j-1}}.
    \end{equation}
    Regarding the term $\frac{w_j}{w_{j-1}}$ in the product, by noting the definition of $\tilde{P}_t(\cdot)$, we have
    \begin{align}
        \frac{w_j}{w_{j-1}} 
        &= \sum_{\bx \in \mX} \frac{\exp\rbr{\eta \sum_{t=1}^j u_t(\bx)}}{w_{j-1}} \\
        &= \sum_{\bx \in \mX} \frac{\exp\rbr{\eta \sum_{t=1}^{j-1} u_t(\bx)}}{w_{j-1}} \exp\rbr{\eta u_j(\bx)} \\
        \label{eq:ratio_w_ub}
        &= \sum_{\bx \in \mX} \tilde{P}_j(\bx) \exp\rbr{\eta u_j(\bx)}.
    \end{align}
    Note that $\eta u_j(\bx) \leq 1$ holds. Thus, by leveraging the elementary inequalities $\forall a \leq 1, \exp(a) \leq 1 + a + a^2$ and $\forall a \in \R,~1 + a \leq \exp(a)$, we have
    \begin{align}
        \sum_{\bx \in \mX} \tilde{P}_j(\bx) \exp\rbr{\eta u_j(\bx)}
        &\leq \sum_{\bx \in \mX} \tilde{P}_j(\bx) \rbr{1 + \eta u_j(\bx) + \eta^2 u_j^2(\bx)} \\
        &= 1 + \eta \sum_{\bx \in \mX} \tilde{P}_j(\bx) u_j(\bx) + \eta^2 \sum_{\bx \in \mX} \tilde{P}_j(\bx) u_j^2(\bx) \\
        \label{eq:exp_P_ub}
        &\leq \exp\rbr{\eta \sum_{\bx \in \mX} \tilde{P}_j(\bx) u_j(\bx) + \eta^2 \sum_{\bx \in \mX} \tilde{P}_j(\bx) u_j^2(\bx)}.
    \end{align}
    By aggregating Eqs.~\eqref{eq:exp_ratio_eq}, \eqref{eq:ratio_w_ub}, and \eqref{eq:exp_P_ub}, we have
    \begin{align}
        \exp\rbr{\eta \sum_{t=1}^T u_t(\bz)} \leq |\mX| \exp\rbr{\eta \sum_{j=1}^T \sum_{\bx \in \mX} \tilde{P}_j(\bx) u_j(\bx) + \eta^2 \sum_{j=1}^T \sum_{\bx \in \mX} \tilde{P}_j(\bx) u_j^2(\bx)}.
    \end{align}
    By taking the logarithm and dividing by the factor $\eta > 0$ on both sides of the above inequality, we obtain the desired result.
\end{proof}

\begin{lemma}[Uniform upper bound of the optimistic estimator $u_t$]
\label{lem:ub_ut}
    Fix any $T \in \N_+$, $\alpha \in (0, 1)$, $\eta > 0$, $\beta > 0$, and $\lambda \geq 1$. Fix any finite input domain $\mX$.
    Suppose that \asmpref{asmp:regularity} holds. Then, when running \algoref{alg:exp3}, the following upper bound of the optimistic estimator $u_t$ holds for any $t \in [T]$ and $\bx \in \mX$:
    \begin{equation}
        u_t(\bx) \leq \frac{2}{\alpha} (2 B \gamma_{T} + \beta \sqrt{\gamma_T}).
    \end{equation}
\end{lemma}
\begin{proof}
    From the definition of $\hat{f}_t$, we have
    \begin{align}
        \hat{f}_t(\bx) 
        &= \psi(\bx)^{\top} G_t(\lambda)^{-1} \psi(\bx_t) f_t(\bx_t) \\
        &\leq \|\psi(\bx)\|_{G_t(\lambda)^{-1}} \|G_t(\lambda)^{-1/2} \psi(\bx_t) f_t(\bx_t)\|_2 \\
        &\leq B \|\psi(\bx)\|_{G_t(\lambda)^{-1}} \|\psi(\bx_t)\|_{G_t(\lambda)^{-1}} \\
        &\leq B \max_{\bx' \in \mX} \|\psi(\bx')\|_{G_t(\lambda)^{-1}}^2.
    \end{align}
    Thus, the optimistic estimator $u_t$ satisfies the following upper bound:
    \begin{align}
        u_t(\bx) 
        &= \hat{f}_t(\bx) + \beta \|\psi(\bx)\|_{G_t(\lambda)^{-1}} \\
        \label{eq:psi_max_ub}
        &\leq B \max_{\bx' \in \mX} \|\psi(\bx')\|_{G_t(\lambda)^{-1}}^2 + \beta \max_{\bx' \in \mX} \|\psi(\bx')\|_{G_t(\lambda)^{-1}}.
    \end{align}
    Furthermore, we obtain the upper bound of $\|\psi(\bx')\|_{G_t(\lambda)^{-1}}$ as follows:
    \begin{align}
        \label{eq:psi_ub_first}
        \|\psi(\bx')\|_{G_t(\lambda)^{-1}}^2
        &= \psi(\bx')^{\top} \rbr{\sum_{\bx \in \mX} P_t(\bx) \psi(\bx) \psi(\bx)^{\top} + \frac{\lambda}{T} \bI_{|\mX|}}^{-1} \psi(\bx') \\
        \label{eq:pi_ub_psd}
        &\leq \psi(\bx')^{\top} \rbr{ \alpha \sum_{\bx \in \mX} \pi(\bx) \psi(\bx) \psi(\bx)^{\top} + \frac{\lambda}{T} \bI_{|\mX|}}^{-1} \psi(\bx') \\
        \label{eq:pi_mvr_def}
        &= \frac{1}{\alpha}\psi(\bx')^{\top} \rbr{ \frac{1}{\lceil T\alpha \rceil} \sum_{j = 1}^{\lceil T\alpha \rceil} \psi(\bx_j^{(\text{MVR})}) \psi(\bx_j^{(\text{MVR})})^{\top} + \frac{\lambda}{T\alpha} \bI_{|\mX|}}^{-1} \psi(\bx') \\
        \label{eq:lambda_ceil_ub}
        &\leq \frac{1}{\alpha}\psi(\bx')^{\top} \rbr{ \frac{1}{\lceil T\alpha \rceil} \sum_{j = 1}^{\lceil T\alpha \rceil} \psi(\bx_j^{(\text{MVR})}) \psi(\bx_j^{(\text{MVR})})^{\top} + \frac{\lambda}{\lceil T\alpha \rceil} \bI_{|\mX|}}^{-1} \psi(\bx') \\
        &= \frac{\lceil T\alpha \rceil}{\alpha} \psi(\bx')^{\top} \rbr{\sum_{j = 1}^{\lceil T\alpha \rceil} \psi(\bx_j^{(\text{MVR})}) \psi(\bx_j^{(\text{MVR})})^{\top} + \lambda \bI_{|\mX|}}^{-1} \psi(\bx') \\
        \label{eq:psi_mvr_sigma}
        &= \frac{\lceil T\alpha \rceil}{\alpha \lambda} \sigma^2\rbr{\bx'; \bX_{\lceil T\alpha \rceil}^{(\text{MVR})}, \lambda} \\
        \label{eq:psi_mvr_sigma_ub}
        &\leq \frac{4 \gamma_{T}}{\alpha},
    \end{align}
    where:
    \begin{itemize}
        \item Eq.~\eqref{eq:pi_ub_psd} follows from $\sum_{\bx \in \mX} \alpha \pi(\bx) \psi(\bx) \psi(\bx)^{\top} \preceq \sum_{\bx \in \mX} P_t(\bx) \psi(\bx) \psi(\bx)^{\top}$, which is implied by the definition $P_t(\bx) = \alpha \pi(\bx) + (1 - \alpha) \tilde{P}_t(\bx)$.
        \item Eqs.~\eqref{eq:lambda_ceil_ub} and \eqref{eq:psi_mvr_sigma} follow from $\frac{\lambda}{\lceil T\alpha \rceil} \bI_{|\mX|} \preceq \frac{\lambda}{T\alpha} \bI_{|\mX|}$ and the feature representation of the posterior variance in Eq.~\eqref{eq:feature_sigma2_rep}, respectively.
        \item Eq.~\eqref{eq:psi_mvr_sigma_ub} follows from $\sigma^2\rbr{\bx'; \bX_{\lceil T\alpha \rceil}^{(\text{MVR})}, \lambda} \leq 4 \lambda \gamma_{\lceil T\alpha \rceil} / \lceil T\alpha \rceil$ and $\gamma_{\lceil T\alpha \rceil} \leq \gamma_{T}$. Indeed, from the definition of $\bX_{\lceil T\alpha \rceil}^{(\text{MVR})}$ and monotonicity with respect to the training inputs of the posterior variance, we have
        \begin{equation}
            \sigma^2\rbr{\bx'; \bX_{\lceil T\alpha \rceil}^{(\text{MVR})}, \lambda} 
            \leq \frac{1}{\lceil T\alpha \rceil} \sum_{t=1}^{\lceil T\alpha \rceil}\sigma^2\rbr{\bx_t^{(\text{MVR})}; \bX_{t-1}^{(\text{MVR})}, \lambda} \leq \frac{2 \gamma_{\lceil T\alpha \rceil}}{\lceil T\alpha \rceil \log (1 + \lambda^{-1})}.
        \end{equation}
        See \citep{cai2021lenient,li2022gaussian,vakili2021optimal} for details.
        Since $\forall a \in [0, 1], \log(1 + a) \geq a/2$ and $\lambda \geq 1$, 
        we can confirm $\sigma^2\rbr{\bx'; \bX_{\lceil T\alpha \rceil}^{(\text{MVR})}, \lambda} \leq 4 \lambda \gamma_{\lceil T\alpha \rceil} / \lceil T\alpha \rceil$.
    \end{itemize}
    Hence, by combining Eq.~\eqref{eq:psi_mvr_sigma_ub} with Eq.~\eqref{eq:psi_max_ub}, we obtain
    \begin{equation}
        \label{eq:ub_ut_last}
        u_t(\bx) \leq \frac{4 B \gamma_{T}}{\alpha} + \beta \sqrt{\frac{4 \gamma_{T}}{\alpha}}
        \leq \frac{2}{\alpha} (2 B \gamma_{T} + \beta \sqrt{\gamma_T}),
    \end{equation}
    which is the desired inequality. The last inequality in Eq.~\eqref{eq:ub_ut_last} follows from $1/\sqrt{\alpha} \leq 1/\alpha$ due to $\alpha \in (0, 1)$.
\end{proof}

\subsection{Helper lemmas for RLS-Kernelized Exp3}
\paragraph{Overview.} The lemmas in this subsection are leveraged in the proof of RLS-kernelized Exp3 in \appref{app:proof_random_exp3}. \lemref{lem:identity_nystrom} provides the identities of the approximated estimators and alternative representations based on Nyström embedding. We would like to note that, in a multi-task setting under an infinite-dimensional feature map, Lemma 8 in \cite{chowdhury2021no} already provides a similar result to that in \lemref{lem:identity_nystrom}. 
\lemref{lem:acc_comp_rls} provides the theoretical property of the computational cost and accuracy of the recursive RLS-Nystr\"om subroutine. The proof of \lemref{lem:acc_comp_rls} is obtained by combining the existing effective dimension-based result by \citep{musco2017recursive} with \lemref{lem:est_var}. 
\lemref{lem:ub_ut_tilde} provides the upper bound of approximated optimistic estimator $\tilde{u}_t$, which serves as the counterpart of \lemref{lem:ub_ut} for the proof of the original kernelized Exp3 algorithm. 
Lemmas~\ref{lem:Ppsi_Qpsi} and \ref{lem:ort_proj_rel} provide the spectral results. Roughly speaking, under the sufficient accuracy of the sampling matrix $\bS_t$, these results are useful for connecting the approximated quantities in RLS-kernelized Exp3 with non-approximated quantities in the proof. The proofs of Lemmas~\ref{lem:Ppsi_Qpsi} and \ref{lem:ort_proj_rel} are almost directly obtained by adapting the proof of \citep{calandriello2018statistical} and \citep{calandriello2019gaussian}, which study the kernel approximation for k-means clustering and the computationally efficient KB algorithm, respectively.
Finally, \lemref{lem:sup-norm_utilde} is used to obtain the upper bound on the regret arising from the unfavorable rare event, where the accuracy of recursive RLS-Nystr\"om is not sufficient. The proof of \lemref{lem:sup-norm_utilde} is obtained with elementary matrix calculations.

\begin{lemma}[Identities in Eqs.~\eqref{eq:tilde_f_nystrom} and \eqref{eq:tilde_sigma_nystrom}]
\label{lem:identity_nystrom}
    Fix any finite input domain $\mX \coloneqq \{\bx^{(1)}, \ldots, \bx^{(|\mX|)}\}$, any probability mass function $P: \mX \rightarrow (0, 1]$, any positive definite kernel $k: \mX \times \mX \rightarrow \R$, any feature map $\psi: \mX \rightarrow \R^{|\mX|}$ of $k$ such that $k(\bx, \bx') = \psi(\bx)^{\top} \psi(\bx')$ holds for any $\bx, \bx' \in \mX$. Fix any $T \in \N_+$, $\lambda > 0$, and $s \in [|\mX|]$.
    Furthermore, define a sampling matrix $\bS \in \R^{|\mX| \times s}$ with exactly one strictly positive entry in each column and zeros elsewhere.
    In addition, define $\tilde{G}(\lambda)$ as 
    $\tilde{G}(\lambda) = (\bQ \bar{\bPsi}) (\bQ \bar{\bPsi})^{\top} + \frac{\lambda}{T} \bI_{|\mX|}$, where $\bar{\bPsi} = (\bar{\psi}(\bx^{(1)}), \ldots, \bar{\psi}(\bx^{(|\mX|)})) \in \R^{|\mX| \times |\mX|}$, $\bar{\psi}(\bx) = \sqrt{P(\bx)}\psi(\bx)$, and 
    $\bQ = \bar{\bPsi} \bS [(\bar{\bPsi} \bS)^{\top} (\bar{\bPsi} \bS)]^{\dagger} (\bar{\bPsi} \bS)^{\top}$. Then, the following identities hold:
    \begin{align}
        \label{eq:tilde_f_identity_proof}
        \psi(\bx)^{\top} \tilde{G}(\lambda)^{-1} \bQ \psi(\bz) &= \frac{1}{\sqrt{P(\bx) P(\bz)}} \phi(\bx)^{\top} \rbr{\sum_{\bx' \in \mX} \phi(\bx') \phi(\bx')^{\top} + \frac{\lambda}{T} \bI_{s}}^{-1} \phi(\bz), \\
        \label{eq:tilde_sigma_identity_proof}
        \begin{split}
        \|\psi(\bx)\|_{\tilde{G}(\lambda)^{-1}}^2 &= 
        \frac{T}{\lambda P(\bx)} \rbr{\tilde{k}(\bx, \bx) - \phi(\bx)^{\top} \phi(\bx)} \\ 
        ~~~~&+ \frac{1}{P(\bx)}\phi(\bx)^{\top} \rbr{\sum_{\bx' \in \mX} \phi(\bx') \phi(\bx')^{\top} + \frac{\lambda}{T} \bI_{s}}^{-1} \phi(\bx),    
        \end{split}
    \end{align}
    where $\phi(\bx) = [(\bS^{\top} \tilde{\bK} \bS)^{1/2}]^{\dagger} \bS^{\top} \tilde{\bk}(\bx) \in \R^{s}$. Here, $\tilde{k}(\bx, \bx')$, $\tilde{\bk}(\bx) \in \R^{|\mX|}$, and $\tilde{\bK} \in \R^{|\mX| \times |\mX|}$ are defined as $\tilde{k}(\bx, \bx') = \sqrt{P(\bx) P(\bx')}k(\bx, \bx')$, $\tilde{\bk}(\bx) = [\tilde{k}(\bx, \bx^{(i)})]_{i \in [|\mX|]}$, and $\tilde{\bK} = [\tilde{k}(\bx^{(i)}, \bx^{(j)})]_{i,j \in [|\mX|]}$, respectively.
\end{lemma}
\begin{proof}
    For any $\bz, \bx \in \mX$, we first confirm the two identities: $\phi(\bz)^{\top} \phi(\bx) = (\bQ \bar{\psi}(\bz))^{\top} \bar{\psi}(\bx)$ and $\phi(\bz)^{\top} \phi(\bx) = (\bQ \bar{\psi}(\bx))^{\top} (\bQ \bar{\psi}(\bz))$. From the definition of $\phi(\cdot)$, we have
    \begin{align}
        \phi(\bz)^{\top} \phi(\bx) 
        &=  \tilde{\bk}(\bz)^{\top} \bS [(\bS^{\top} \tilde{\bK} \bS)^{1/2}]^{\dagger}
        [(\bS^{\top} \tilde{\bK} \bS)^{1/2}]^{\dagger} \bS^{\top} \tilde{\bk}(\bx) \\
        &=  \tilde{\bk}(\bz)^{\top} \bS (\bS^{\top} \tilde{\bK} \bS)^{\dagger} \bS^{\top} \tilde{\bk}(\bx) \\
        &=  \bar{\psi}(\bz)^{\top} \bar{\bPsi} \bS (\bS^{\top} \bar{\bPsi}^{\top} \bar{\bPsi} \bS)^{\dagger} \bS^{\top} \bar{\bPsi}^{\top} \bar{\psi}(\bx) \\
        &=  \bar{\psi}(\bz)^{\top} \bQ  \bar{\psi}(\bx) \\
        &=  \bar{\psi}(\bz)^{\top} \bQ^{\top} \bar{\psi}(\bx) \\
        \label{eq:basic_phi_identity1}
        &=  (\bQ \bar{\psi}(\bz))^{\top} \bar{\psi}(\bx).
    \end{align}
    Furthermore, since the orthogonal projection matrix $\bQ$ satisfies $\bQ^{\top} = \bQ$ and $\bQ^2 = \bQ$, we also obtain
    \begin{equation}
        \label{eq:basic_phi_identity2}
        \phi(\bz)^{\top} \phi(\bx) = \bar{\psi}(\bz)^{\top} \bQ  \bar{\psi}(\bx) = \bar{\psi}(\bz)^{\top} \bQ^2 \bar{\psi}(\bx) = (\bQ\bar{\psi}(\bz))^{\top} (\bQ \bar{\psi}(\bx)).
    \end{equation}
    Here, we define $\tilde{\bPsi} \in \R^{|\mX| \times |\mX|}$ and $\bPhi \in \R^{s \times |\mX|}$ as $\tilde{\bPsi} = \bQ \bar{\bPsi}$ and $\bPhi = (\phi(\bx^{(1)}), \ldots, \phi(\bx^{(|\mX|)}))$, respectively. By leveraging the identities $\phi(\bz)^{\top} \phi(\bx) = (\bQ \bar{\psi}(\bz))^{\top} \bar{\psi}(\bx)$ and $\phi(\bz)^{\top} \phi(\bx) = (\bQ \bar{\psi}(\bx))^{\top} (\bQ \bar{\psi}(\bz))$ shown in Eqs.~\eqref{eq:basic_phi_identity1} and \eqref{eq:basic_phi_identity2},  we have
    \begin{align}
        \label{eq:basic_phi_identity3}
        \bPhi^{\top} \phi(\bx) = \tilde{\bPsi}^{\top} \bar{\psi}(\bx)~~\mathrm{and}~~\bPhi^{\top} \bPhi = \tilde{\bPsi}^{\top} \tilde{\bPsi}.
    \end{align}
    Using the above identities, we prove Eqs.~\eqref{eq:tilde_f_identity_proof} and \eqref{eq:tilde_sigma_identity_proof}. Firstly, regarding Eq.~\eqref{eq:tilde_f_identity_proof}, by using the unit vector $\bm{e}_{\bz} \in \R^{|\mX|}$ such that $(\bm{e}_{\bz})_i = \1\{\bz = \bx^{(i)}\}$, 
    we have
    \begin{align}
        \phi(\bx)^{\top} \rbr{\sum_{\bx' \in \mX} \phi(\bx') \phi(\bx')^{\top} + \frac{\lambda}{T} \bI_{s}}^{-1} \phi(\bz) 
        &= \phi(\bx)^{\top} \rbr{\bPhi \bPhi^{\top} + \frac{\lambda}{T} \bI_{s}}^{-1} \bPhi \bm{e}_{\bz} \\
        \label{eq:ftilde_matrix_identity_1}
        &= \phi(\bx)^{\top} \bPhi \rbr{\bPhi^{\top} \bPhi  + \frac{\lambda}{T} \bI_{|\mX|}}^{-1}  \bm{e}_{\bz} \\
        \label{eq:ftilde_basic_phi_identity}
        &= \bar{\psi}(\bx)^{\top} \tilde{\bPsi} \rbr{\tilde{\bPsi}^{\top} \tilde{\bPsi}  + \frac{\lambda}{T} \bI_{|\mX|}}^{-1}  \bm{e}_{\bz} \\
        \label{eq:ftilde_matrix_identity_2}
        &= \bar{\psi}(\bx)^{\top} \rbr{\tilde{\bPsi}\tilde{\bPsi}^{\top}  + \frac{\lambda}{T} \bI_{|\mX|}}^{-1} \tilde{\bPsi} \bm{e}_{\bz} \\
        &= \bar{\psi}(\bx)^{\top} \rbr{(\bQ \bar{\bPsi})(\bQ\bar{\bPsi})^{\top}  + \frac{\lambda}{T} \bI_{|\mX|}}^{-1} \bQ \bar{\psi}(\bz),
    \end{align}
    where Eqs.~\eqref{eq:ftilde_matrix_identity_1} and \eqref{eq:ftilde_matrix_identity_2} follow from \lemref{lem:rewrite_inv}, and Eq.~\eqref{eq:ftilde_basic_phi_identity} uses Eq.~\eqref{eq:basic_phi_identity3}. Since $\tilde{G}(\lambda) = (\bQ \bar{\bPsi})(\bQ\bar{\bPsi})^{\top}  + \frac{\lambda}{T} \bI_{|\mX|}$ and $\bar{\psi}(\bx) = \sqrt{P(\bx)} \psi(\bx)$, the above equation implies Eq.~\eqref{eq:tilde_f_identity_proof}. Regarding Eq.~\eqref{eq:tilde_sigma_identity_proof}, we have
    \begin{align}
        &\tilde{k}(\bx, \bx) - \phi(\bx)^{\top} \phi(\bx) + \frac{\lambda}{T} \phi(\bx)^{\top} \rbr{\sum_{\bx' \in \mX} \phi(\bx') \phi(\bx')^{\top} + \frac{\lambda}{T} \bI_{s}}^{-1} \phi(\bx) \\
        &= \tilde{k}(\bx, \bx) - \phi(\bx)^{\top} \sbr{\bI_{s} - \frac{\lambda}{T} \rbr{\bPhi \bPhi^{\top} + \frac{\lambda}{T} \bI_{s}}^{-1}} \phi(\bx) \\
        \label{eq:tildesigma_matrix_identity_1}
        &= \tilde{k}(\bx, \bx) - \phi(\bx)^{\top} \bPhi \rbr{\bPhi^{\top} \bPhi + \frac{\lambda}{T} \bI_{|\mX|}}^{-1} \bPhi^{\top} \phi(\bx) \\
        \label{eq:tildesigma_basic_phi_identity}
        &= \bar{\psi}(\bx)^{\top}\bar{\psi}(\bx) - \bar{\psi}(\bx)^{\top} \tilde{\bPsi}  \rbr{\tilde{\bPsi}^{\top} \tilde{\bPsi} + \frac{\lambda}{T} \bI_{|\mX|}}^{-1}  \tilde{\bPsi}^{\top} \bar{\psi}(\bx) \\
        &= \bar{\psi}(\bx)^{\top}\rbr{\bI_{|\mX|} - \tilde{\bPsi}  \rbr{\tilde{\bPsi}^{\top} \tilde{\bPsi} + \frac{\lambda}{T} \bI_{|\mX|}}^{-1}  \tilde{\bPsi}^{\top}}\bar{\psi}(\bx) \\
        \label{eq:tildesigma_matrix_identity_2}
        &= \frac{\lambda}{T}\bar{\psi}(\bx)^{\top}\rbr{\tilde{\bPsi} \tilde{\bPsi}^{\top} + \frac{\lambda}{T} \bI_{|\mX|}}^{-1} \bar{\psi}(\bx),
    \end{align}
    where Eqs.~\eqref{eq:tildesigma_matrix_identity_1} and \eqref{eq:tildesigma_matrix_identity_2} follow from \lemref{lem:rewrite_inv}, and Eq.~\eqref{eq:tildesigma_basic_phi_identity} uses Eq.~\eqref{eq:basic_phi_identity3}. By aligning the above equation, we obtain Eq.~\eqref{eq:tilde_sigma_identity_proof}.
\end{proof}

\begin{lemma}[Accuracy and computational cost of recursive RLS-Nystr\"om with weighted kernel, adapted from Theorem~8 in \citep{musco2017recursive}]
    \label{lem:acc_comp_rls}
    Fix any $T \in \N_+$, $\lambda \geq 1$, $\delta \in (0, 1/32)$, any finite input domain $\mX \coloneqq \{\bx^{(1)}, \ldots, \bx^{(|\mX|)}\}$, any probability mass function $P: \mX \rightarrow [0, 1]$, and any positive definite kernel function $k: \mX \times \mX \rightarrow \R$ such that $\forall \bx \in \mX,~k(\bx, \bx) \leq 1$. Furthermore, let $\psi(\cdot): \mX \rightarrow \R^{|\mX|}$ be a feature map of $k$ such that $k(\bx, \bx') = \psi(\bx)^{\top} \psi(\bx')$ holds for any $\bx, \bx' \in \mX$. We also define the weighted kernel $\tilde{k}$ as  $\tilde{k}(\bx, \bx') = \bar{\psi}(\bx)^{\top} \bar{\psi}(\bx')$, where $\bar{\psi}(\bx) =\sqrt{P(\bx)} \psi(\bx)$.
    Then, when running \algoref{algo:recursive_RLS} with the inputs $(\mX,~\tilde{k},~\lambda/T,~\delta)$, the following three statements simultaneously hold with probability at least $1 - 3\delta$:
    \begin{enumerate}
        \item The number of sampled columns $s \in [|\mX|]$ satisfies $s \leq s_{\mathrm{max}}(4\gamma_T, \delta)$, where $s_{\mathrm{max}}(w, z) = 384(w+1) \log((w+1)/z)$.
        \item The computational time of \algoref{algo:recursive_RLS} is at most $O(|\mX| s_{\mathrm{max}}^2(4\gamma_T, \delta))$.
        \item The weighted sampling matrix $\bS \in \R^{|\mX| \times s}$ returned by \algoref{algo:recursive_RLS} satisfies the following:
        \begin{equation}
            \label{eq:epsilon_acc_dict}
            \frac{1}{2}\rbr{\bar{\bPsi} \bar{\bPsi}^{\top} + \frac{\lambda}{T} \bI_{|\mX|}} \preceq \bar{\bPsi} \bS \bS^{\top} \bar{\bPsi}^{\top} + \frac{\lambda}{T} \bI_{|\mX|} \preceq \frac{3}{2}\rbr{\bar{\bPsi} \bar{\bPsi}^{\top} + \frac{\lambda}{T} \bI_{|\mX|}},
        \end{equation}
        where $\bar{\bPsi} = (\bar{\psi}(\bx^{(1)}), \ldots, \bar{\psi}(\bx^{(|\mX|)})) \in \R^{|\mX| \times |\mX|}$.
    \end{enumerate}
    In statements 1 and 2, $\gamma_T \coloneqq \gamma_T(\lambda, \mX)$ denotes the maximum information gain of kernel $k$ on $\mX$ with the variance parameter $\lambda$.
\end{lemma}
\begin{proof}
    Let $\bar{\bK} \coloneqq [\tilde{k}(\bx^{(i)}, \bx^{(j)})]_{i, j \in [|\mX|]} \in \R^{|\mX| \times |\mX|}$ be the Gram matrix of the weighted kernel function $\tilde{k}$. Let us define the following quantity $d_{\text{eff}}(\bar{\bK})$:
    \begin{equation}
        d_{\text{eff}}(\bar{\bK}) = \mathrm{Tr}\rbr{\bar{\bK} \rbr{\bar{\bK} + \frac{\lambda}{T} \bI_{|\mX|}}^{-1}}.
    \end{equation}
    Here, note that $\bar{\bK} = \bar{\bPsi}^{\top} \bar{\bPsi}$ holds.
    Then, with probability at least $1 - 3\delta$, statement 3 and the following two statements simultaneously hold as the direct consequence of Theorem 8 in \citep{musco2017recursive}:
    \begin{itemize}
        \item The number of sampled columns $s \in [|\mX|]$ satisfies $s \leq s_{\mathrm{max}}(d_{\text{eff}}(\bar{\bK}), \delta)$.
        \item The computational time of \algoref{algo:recursive_RLS} is at most $O(|\mX| s_{\mathrm{max}}^2(d_{\text{eff}}(\bar{\bK}), \delta))$.
    \end{itemize}
    Thus, it remains to prove the inequality $d_{\text{eff}}(\bar{\bK}) \leq 4 \gamma_T$. By leveraging the feature representation $\bar{\bK} = \bar{\bPsi}^{\top} \bar{\bPsi}$, we obtain the desired inequality as follows:
    \begin{align}
        d_{\text{eff}}(\bar{\bK}) 
        &= \mathrm{Tr}\rbr{\bar{\bK} \rbr{\bar{\bK} + \frac{\lambda}{T} \bI_{|\mX|}}^{-1}} \\
        &= \mathrm{Tr}\rbr{\bar{\bPsi}^{\top} \bar{\bPsi} \rbr{\bar{\bPsi}^{\top} \bar{\bPsi} + \frac{\lambda}{T} \bI_{|\mX|}}^{-1}} \\
        \label{eq:deff_conj_matrix}
        &= \mathrm{Tr}\rbr{\bar{\bPsi}^{\top} \rbr{\bar{\bPsi} \bar{\bPsi}^{\top} + \frac{\lambda}{T} \bI_{|\mX|}}^{-1} \bar{\bPsi}} \\
        &= \sum_{\bx \in \mX} \bar{\psi}(\bx)^{\top} \rbr{\sum_{\bx' \in \mX} \bar{\psi}(\bx') \bar{\psi}(\bx')^{\top} + \frac{\lambda}{T} \bI_{|\mX|}}^{-1} \bar{\psi}(\bx) \\
        &= \sum_{\bx \in \mX} P(\bx) \psi(\bx)^{\top} \rbr{\sum_{\bx' \in \mX} P(\bx') \psi(\bx') \psi(\bx')^{\top} + \frac{\lambda}{T} \bI_{|\mX|}}^{-1} \psi(\bx) \\
        \label{eq:deff_P_gamma}
        &\leq 4\gamma_T,
    \end{align}
    where Eq.~\eqref{eq:deff_conj_matrix} follows from \lemref{lem:rewrite_inv}, and Eq.~\eqref{eq:deff_P_gamma} follows from \lemref{lem:est_var} with $\lambda \geq 1$.
    Thus, we complete the proof.
\end{proof}

\begin{lemma}[Uniform upper bound of the approximated optimistic estimator $\tilde{u}_t$]
\label{lem:ub_ut_tilde}
    Fix any $T \in \N_+$, $\alpha \in (0, 1)$, $\eta > 0$, $\beta > 0$, and $\lambda \geq 1$. Fix any finite input domain $\mX$.
    Suppose that \asmpref{asmp:regularity} holds. Then, when running \algoref{alg:exp3_nystrom_approx}, under event $\cap_{t \in [T]} A_t$ (Eq.~\eqref{eq:thm_rke3_proof_eps_acc}), the following upper bound of the approximated optimistic estimator $\tilde{u}_t$ holds for any $t \in [T]$ and $\bx \in \mX$:
    \begin{equation}
        \tilde{u}_t(\bx) \leq \frac{4 (3 + \sqrt{6}) B \gamma_{T}}{\alpha} + 2 \beta \sqrt{\frac{3 \gamma_{T}}{\alpha}}
        \leq \frac{2\sqrt{3}}{\alpha} \rbr{2(\sqrt{3} + \sqrt{2}) B \gamma_{T} + \beta \sqrt{\gamma_T}}.
    \end{equation}
\end{lemma}
\begin{proof}
    From the definition of $\tilde{f}_t$, we have
    \begin{align}
        \tilde{f}_t(\bx) 
        &= \psi(\bx)^{\top} \tilde{G}_t(\lambda)^{-1} \tilde{\psi}_t(\bx_t) f_t(\bx_t) \\
        &= \psi(\bx)^{\top} \tilde{G}_t(\lambda)^{-1} \psi(\bx_t) f_t(\bx_t) - \psi(\bx)^{\top} \tilde{G}_t(\lambda)^{-1} \rbr{\psi(\bx_t) - \bQ_t \psi(\bx_t)}  f_t(\bx_t) \\
        &= \psi(\bx)^{\top} \tilde{G}_t(\lambda)^{-1} \psi(\bx_t) f_t(\bx_t) - \psi(\bx)^{\top} \tilde{G}_t(\lambda)^{-1} (\bI_{|\mX|} - \bQ_t) \psi(\bx_t) f_t(\bx_t) \\
        &\leq B |\psi(\bx)^{\top} \tilde{G}_t(\lambda)^{-1} \psi(\bx_t)| + B |\psi(\bx)^{\top} \tilde{G}_t(\lambda)^{-1} (\bI_{|\mX|} - \bQ_t) \psi(\bx_t)|.
    \end{align}
    The first term of the above inequality is bounded from above as
    \begin{equation}
        B |\psi(\bx)^{\top} \tilde{G}_t(\lambda)^{-1} \psi(\bx_t)| \leq B \|\psi(\bx)\|_{\tilde{G}_t(\lambda)^{-1}} \|\psi(\bx_t)\|_{\tilde{G}_t(\lambda)^{-1}} \leq 3B \max_{\bx' \in \mX} \|\psi(\bx')\|_{G_t(\lambda)^{-1}}^2,
    \end{equation}
    where the last inequality follows from \lemref{lem:Ppsi_Qpsi}.
    Regarding the second term $B |\psi(\bx)^{\top} \tilde{G}_t(\lambda)^{-1} (\bI_{|\mX|} - \bQ_t) \psi(\bx_t)|$, we have
    \begin{align}
        &B |\psi(\bx)^{\top} \tilde{G}_t(\lambda)^{-1} (\bI_{|\mX|} - \bQ_t) \psi(\bx_t)| \\
        \label{eq:ub_ut_tilde_ftsec_sh}
        &\leq B \|\tilde{G}_t(\lambda)^{-1} \psi(\bx)\|_2 \|(\bI_{|\mX|} - \bQ_t) \psi(\bx_t)\|_2 \\
        &= B \sqrt{\psi(\bx)^{\top} \tilde{G}_t(\lambda)^{-2} \psi(\bx)} \sqrt{\psi(\bx_t)^{\top} 
        (\bI_{|\mX|} - \bQ_t)^2 \psi(\bx_t)} \\
        \label{eq:ub_ut_tilde_ftsec_meig}
        &\leq B \sqrt{\frac{T}{\lambda}}\|\psi(\bx)\|_{\tilde{G}_t(\lambda)^{-1}} \sqrt{\psi(\bx_t)^{\top} (\bI_{|\mX|} - \bQ_t) \psi(\bx_t)} \\
        \label{eq:ub_ut_tilde_ftsec_pslem}
        &\leq B \sqrt{\frac{3T}{\lambda}} \|\psi(\bx)\|_{G_t(\lambda)^{-1}} \sqrt{\frac{2\lambda}{T}} \|\psi(\bx_t)\|_{G_t(\lambda)^{-1}} \\
        \label{eq:ub_ut_tilde_second}
        &\leq \sqrt{6} B  \max_{\bx' \in \mX} \|\psi(\bx')\|_{G_t(\lambda)^{-1}}^2,
    \end{align}
    where:
    \begin{itemize}
        \item Eq.~\eqref{eq:ub_ut_tilde_ftsec_sh} follows from the Cauchy--Schwarz inequality.
        \item Eq.~\eqref{eq:ub_ut_tilde_ftsec_meig} follows from $\psi(\bx)^{\top} \tilde{G}_t(\lambda)^{-2} \psi(\bx) \leq \|\tilde{G}_t(\lambda)^{-1}\| \|\psi(\bx)\|_{\tilde{G}_t(\lambda)^{-1}}^2 \leq \frac{T}{\lambda} \|\psi(\bx)\|_{\tilde{G}_t(\lambda)^{-1}}^2$. Furthermore, note that $(\bI_{|\mX|} - \bQ_t)^{2} = (\bI_{|\mX|} - \bQ_t)$ holds since the matrix $\bI_{|\mX|} - \bQ_t$ is also a projection matrix.
        \item Eq.~\eqref{eq:ub_ut_tilde_ftsec_pslem} follows from $\|\psi(\bx)\|_{\tilde{G}_t(\lambda)^{-1}} \leq \sqrt{3}\|\psi(\bx)\|_{G_t(\lambda)^{-1}}$
        and $\psi(\bx_t)^{\top} (\bI_{|\mX|} - \bQ_t) \psi(\bx_t) \leq \frac{2\lambda}{T}\psi(\bx_t)^{\top} G_t(\lambda)^{-1} \psi(\bx_t)$, which are implied by \lemref{lem:Ppsi_Qpsi} and \lemref{lem:ort_proj_rel}, respectively.
    \end{itemize}
    Thus, the approximated optimistic estimator $\tilde{u}_t$ satisfies the following upper bound:
    \begin{align}
        \tilde{u}_t(\bx) 
        &= \tilde{f}_t(\bx) + \beta \|\psi(\bx)\|_{\tilde{G}_t(\lambda)^{-1}} \\
        \label{eq:ub_ut_tilde_psi_max_ub}
        &\leq (3 + \sqrt{6}) B \max_{\bx' \in \mX} \|\psi(\bx')\|_{G_t(\lambda)^{-1}}^2 + \sqrt{3} \beta \max_{\bx' \in \mX} \|\psi(\bx')\|_{G_t(\lambda)^{-1}}.
    \end{align}
    Furthermore, as with the proof of \lemref{lem:ub_ut} (Eqs.~\eqref{eq:psi_ub_first}--\eqref{eq:psi_mvr_sigma_ub}), we obtain the upper bound of $\|\psi(\bx')\|_{G_t(\lambda)^{-1}}^2$ as $\|\psi(\bx')\|_{G_t(\lambda)^{-1}}^2 \leq 4\gamma_T/\alpha$.
    Hence, by combining Eq.~\eqref{eq:ub_ut_tilde_psi_max_ub} with this upper bound of $\|\psi(\bx')\|_{G_t(\lambda)^{-1}}^2$, we obtain
    \begin{equation}
        \label{eq:ub_ut_tilde_last}
        \tilde{u}_t(\bx) \leq \frac{4 (3 + \sqrt{6}) B \gamma_{T}}{\alpha} + 2 \beta \sqrt{\frac{3 \gamma_{T}}{\alpha}}
        \leq \frac{2\sqrt{3}}{\alpha} \rbr{2(\sqrt{3} + \sqrt{2}) B \gamma_{T} + \beta \sqrt{\gamma_T}}.
    \end{equation}
\end{proof}

\begin{lemma}[Approximation accuracy of multiplicative form for the weighted kernel matrix, adapted from the proof of Lemma~6 in \citep{calandriello2019gaussian}]
    \label{lem:Ppsi_Qpsi}
    Assume that the sampling matrix $\bS \in \R^{|\mX| \times s}$ satisfies Eq.~\eqref{eq:epsilon_acc_dict}.
    Then, under the same notations in Lemma~\ref{lem:acc_comp_rls}, the following relation holds:
    \begin{equation}
        \sum_{\bx \in \mX} P(\bx) \psi(\bx) \psi(\bx)^{\top} + \frac{\lambda}{T} \bI_{|\mX|} \preceq 3 \rbr{ \bQ \bar{\bPsi} (\bQ \bar{\bPsi})^{\top} + \frac{\lambda}{T} \bI_{|\mX|}},
    \end{equation}
    where $\bQ \coloneqq \bar{\bPsi} \bS \sbr{(\bar{\bPsi} \bS)^{\top} \bar{\bPsi} \bS}^{\dagger} (\bar{\bPsi} \bS)^{\top}$ is the projection matrix into the subspace spanned by the column vectors of $\bar{\bPsi} \bS$.
\end{lemma}
\begin{proof}
    Since $\bQ$ is the orthogonal projection matrix into the column span of $\bar{\bPsi} \bS$, the identity $\bQ \bar{\bPsi} \bS = \bar{\bPsi} \bS$ holds. 
    Furthermore, by aligning Eq.~\eqref{eq:epsilon_acc_dict}, we can confirm
    \begin{align}
        \bar{\bPsi} \bar{\bPsi}^{\top} \succeq \frac{2}{3} \bar{\bPsi} \bS \bS^{\top} \bar{\bPsi}^{\top} - \frac{\lambda}{3T} \bI_{|\mX|}~~\mathrm{and}~~\bar{\bPsi} \bS \bS^{\top} \bar{\bPsi}^{\top} \succeq  \frac{1}{2} \bar{\bPsi} \bar{\bPsi}^{\top} - \frac{\lambda}{2T} \bI_{|\mX|}.
    \end{align}
    Thus, we have
    \begin{align}
        \bQ \bar{\bPsi} (\bQ \bar{\bPsi})^{\top} 
        = \bQ \bar{\bPsi} \bar{\bPsi}^{\top} \bQ^{\top} 
        &\succeq \frac{2}{3} \bQ \bar{\bPsi} \bS \bS^{\top} \bar{\bPsi}^{\top} \bQ^{\top} - \frac{\lambda}{3T} \bQ \bQ^{\top} \\
        &= \frac{2}{3} \bar{\bPsi} \bS \bS^{\top} \bar{\bPsi}^{\top} - \frac{\lambda}{3T} \bQ \\
        &\succeq \frac{1}{3} \bar{\bPsi} \bar{\bPsi}^{\top} - \frac{\lambda}{3T} \bI_{|\mX|} - \frac{\lambda}{3T} \bQ \\
        &\succeq \frac{1}{3} \bar{\bPsi} \bar{\bPsi}^{\top} - \frac{2\lambda}{3T} \bI_{|\mX|},
    \end{align}
    where the last line follows from $\bI_{|\mX|} \succeq \bQ$, which is implied by the fact that the eigenvalue of the orthogonal projection matrix is $1$ or $0$. By aligning the above relation, we obtain
    \begin{equation}
        \bar{\bPsi} \bar{\bPsi}^{\top} + \frac{\lambda}{T} \bI_{|\mX|} \preceq 3\rbr{\bQ \bar{\bPsi} (\bQ \bar{\bPsi})^{\top} + \frac{\lambda}{T} \bI_{|\mX|}}.
    \end{equation}
    Note that $\sum_{\bx \in \mX} P(\bx) \psi(\bx) \psi(\bx)^{\top}$ can be rewritten as $\sum_{\bx \in \mX} P(\bx) \psi(\bx) \psi(\bx)^{\top} = \bar{\bPsi} \bar{\bPsi}^{\top}$. Thus, the above inequality is the desired statement.
\end{proof}

\begin{lemma}[Approximation accuracy of the difference form for the weighted kernel matrix, adapted from the proof of Lemma~1 in \citep{calandriello2018statistical}]
    \label{lem:ort_proj_rel}
    Assume that the sampling matrix $\bS \in \R^{|\mX| \times s}$ satisfies Eq.~\eqref{eq:epsilon_acc_dict}.
    Then, under the same notations in Lemma~\ref{lem:acc_comp_rls}, the following relation holds:
    \begin{equation}
        \bI_{|\mX|} - \bQ \preceq \frac{2\lambda}{T} \rbr{\bar{\bPsi} \bar{\bPsi}^{\top} + \frac{\lambda}{T} \bI_{|\mX|}}^{-1},
    \end{equation}
    where $\bQ = \bar{\bPsi} \bS \sbr{(\bar{\bPsi} \bS)^{\top} \bar{\bPsi} \bS}^{\dagger} (\bar{\bPsi} \bS)^{\top}$.
\end{lemma}
\begin{proof}
    By leveraging the definition of matrix inverse, we have
    \begin{align}
        \bI_{|\mX|} &= \rbr{\bar{\bPsi} \bS (\bar{\bPsi} \bS)^{\top} + \frac{\lambda}{T} \bI_{|\mX|}} \rbr{\bar{\bPsi} \bS (\bar{\bPsi} \bS)^{\top} + \frac{\lambda}{T} \bI_{|\mX|}}^{-1} \\
        &= \bar{\bPsi} \bS (\bar{\bPsi} \bS)^{\top} \rbr{\bar{\bPsi} \bS (\bar{\bPsi} \bS)^{\top} + \frac{\lambda}{T} \bI_{|\mX|}}^{-1} + \frac{\lambda}{T} \rbr{\bar{\bPsi}\bS (\bar{\bPsi} \bS)^{\top} + \frac{\lambda}{T} \bI_{|\mX|}}^{-1} \\
        \label{eq:third_to_last_pseudo_inv}
        &\preceq \bar{\bPsi} \bS (\bar{\bPsi} \bS)^{\top} \rbr{\bar{\bPsi} \bS (\bar{\bPsi} \bS)^{\top}}^{\dagger} + \frac{\lambda}{T} \rbr{\bar{\bPsi} \bS (\bar{\bPsi} \bS)^{\top} + \frac{\lambda}{T} \bI_{|\mX|}}^{-1} \\
        \label{eq:sec_to_last_2lam}
        &\preceq \bar{\bPsi} \bS (\bar{\bPsi} \bS)^{\top} \rbr{\bar{\bPsi} \bS (\bar{\bPsi} \bS)^{\top}}^{\dagger} + \frac{2\lambda}{T} \rbr{\bar{\bPsi} \bar{\bPsi}^{\top} + \frac{\lambda}{T} \bI_{|\mX|}}^{-1} \\
        \label{eq:last_Q}
        &\preceq \bQ + \frac{2\lambda}{T} \rbr{\bar{\bPsi} \bar{\bPsi}^{\top} + \frac{\lambda}{T} \bI_{|\mX|}}^{-1},
    \end{align}
    where Eq.~\eqref{eq:third_to_last_pseudo_inv} follows from the basic fact that $\bm{M}\bm{M}^{\top}(\bm{M}\bm{M}^{\top} + \lambda \bI)^{-1} \preceq \bm{M}\bm{M}^{\top}(\bm{M}\bm{M}^{\top})^{\dagger}$ is valid for any matrix $\bm{M}$ and $\lambda > 0$, and Eq.~\eqref{eq:sec_to_last_2lam} follows from Eq.~\eqref{eq:epsilon_acc_dict}. In addition, Eq.~\eqref{eq:last_Q} follows from the uniqueness of the orthogonal projection matrix by noting that $\bar{\bPsi} \bS (\bar{\bPsi} \bS)^{\top} \rbr{\bar{\bPsi} \bS (\bar{\bPsi} \bS)^{\top}}^{\dagger}$ is also an orthogonal projection matrix into the column span of $\bar{\bPsi} \bS$ from the definition of the pseudo inverse. 
    The above expression is the desired statement.
\end{proof}

\begin{lemma}[Upper bound for sup-norm of $\tilde{u}_t(\cdot)$]
    \label{lem:sup-norm_utilde}
    Suppose that Assumption~\ref{asmp:regularity} holds.
    Then, for any $\bx \in \mX$ and $t \in [T]$, the approximated optimistic estimator $\tilde{u}_t$ used in RLS-kernelized Exp3 (\algoref{alg:exp3_nystrom_approx}) or its variant (\algoref{alg:exp3_nystrom_approx_mvr}) satisfies
    \begin{equation}
        |\tilde{u}_t(\bx)| \leq \frac{BT}{\lambda} + \beta \sqrt{\frac{T}{\lambda}}.
    \end{equation}
\end{lemma}
\begin{proof}
    From the definition of $\tilde{f}_t(\bx)$, we have
    \begin{align}
    |\tilde{f}_t(\bx)| 
    &= \left|\psi(\bx)^{\top} \tilde{G}_t(\lambda)^{-1} \tilde{\psi}_t(\bx_t) f_t(\bx_t)\right| \\
    &\leq B \|\psi(\bx)\|_{\tilde{G}_t(\lambda)^{-1}}  \|\tilde{\psi}_t(\bx_t)\|_{\tilde{G}_t(\lambda)^{-1}} \\
    &\leq B \|\psi(\bx)\|_2 \|\tilde{G}_t(\lambda)^{-1}\| \|\tilde{\psi}_t(\bx_t)\|_2 \\
    &\leq \frac{BT}{\lambda},
\end{align}
where the first inequality is the consequence of the Cauchy--Schwarz inequality and $\|f_t\|_{\infty} \leq \|f_t\|_k \leq B$, and the second inequality follows from the definition of the matrix operator norm. 
Furthermore, the third inequality follows from $\|\psi(\bx)\|_2 = \sqrt{k(\bx, \bx)} \leq 1$, $\|\tilde{\psi}_t(\bx)\|_2 = \|\bQ_t \psi(\bx)\|_2 \leq \|\bQ_t\| \|\psi(\bx)\|_2 \leq \sqrt{k(\bx, \bx)} \leq 1$, and the fact that the minimum eigenvalue of $\tilde{G}_t(\lambda)$ is not smaller than $\lambda/T$.
In addition, similarly to the above upper bound, we have
\begin{align}
    \|\psi(\bx)\|_{\tilde{G}_t(\lambda)^{-1}} 
    \leq \sqrt{\|\tilde{G}_t(\lambda)^{-1}\|} \|\psi(\bx)\|_2 \leq \sqrt{\frac{T}{\lambda}}.
\end{align}
Thus, for any $\bx \in \mX$ and $t \in [T]$, the approximated optimistic estimator $\tilde{u}_t$ satisfies
\begin{equation}
    |\tilde{u}_t(\bx)| \leq |\tilde{f}_t(\bx)| + \beta \|\psi(\bx)\|_{\tilde{G}_t(\lambda)^{-1}} \leq \frac{BT}{\lambda} + \beta \sqrt{\frac{T}{\lambda}},
\end{equation}
which is the desired inequality.
\end{proof}

\subsection{Helper lemmas for RLS-Kernelized Exp3 with approximated MVR sequence}
\paragraph{Overview.} The lemmas in this subsection are leveraged in the proof of a variant of RLS-kernelized Exp3 using the approximated MVR sequence in \appref{app:rke3_approx_mvr}. \corref{cor:acc_comp_rls_standard} and \lemref{lem:Ppsi_Qpsi_standard} provide the approximation properties related to the standard kernel functions. We omit the proof of \corref{cor:acc_comp_rls_standard} since it directly follows by noting the fact that the effective dimension of the kernel is upper bounded by MIG (see, e.g., \cite{calandriello2019gaussian}). Furthermore, \lemref{lem:Ppsi_Qpsi_standard} also directly follows by adapting the proof of \lemref{lem:Ppsi_Qpsi} for the standard kernel instead of the weighted kernel; thus, we omit the proof of \lemref{lem:Ppsi_Qpsi_standard}. \lemref{lem:prop_approx_mvr} shows that the maximum of the posterior variance is also upper-bounded by MIG even under the approximated MVR sequence. This lemma serves as a replacement for the upper bound on the posterior variance in the exact MVR sequence in the proof. Finally, based on the result of \lemref{lem:Ppsi_Qpsi_standard}, we can prove \lemref{lem:ub_ut_tilde_approx_mvr}, which serves as a replacement for \lemref{lem:ub_ut_tilde} in the proof of original RLS-kernelized Exp3.

\begin{corollary}[Accuracy and computational cost of recursive RLS-Nystr\"om with standard kernel, adapted from Theorem~8 in \citep{musco2017recursive}]
    \label{cor:acc_comp_rls_standard}
    Fix any $T \in \N_+$, $\lambda \geq 1$, $\delta \in (0, 1/32)$, any finite input set $\bX \coloneqq \{\bx^{(1)}, \ldots, \bx^{(T)}\}$, and any positive definite kernel function $k: \mX \times \mX \rightarrow \R$. Furthermore, let $\psi(\cdot): \mX \rightarrow \R^{|\mX|}$ be a feature map of $k$ such that $k(\bx, \bx') = \psi(\bx)^{\top} \psi(\bx')$ holds for any $\bx, \bx' \in \mX$. Then, when running \algoref{algo:recursive_RLS} with the inputs $(\bX,~k,~\lambda,~\delta)$, the following three statements simultaneously hold with probability at least $1 - 3\delta$:
    \begin{enumerate}
        \item The number of sampled columns $\tilde{s} \in [T]$ satisfies $\tilde{s} \leq s_{\mathrm{max}}(2\gamma_T, \delta)$, where $s_{\mathrm{max}}(w, z) = 384(w+1) \log((w+1)/z)$.
        \item The computational time of \algoref{algo:recursive_RLS} is at most $O(T s_{\mathrm{max}}^2(2\gamma_T, \delta))$.
        \item The weighted sampling matrix $\tilde{\bS} \in \R^{T \times \tilde{s}}$ returned by \algoref{algo:recursive_RLS} satisfies the following:
        \begin{equation}
            \label{eq:standard_epsilon_acc_dict}
            \frac{1}{2}\rbr{\bPsi \bPsi^{\top} + \lambda \bI_{|\mX|}} \preceq \bPsi \tilde{\bS} \tilde{\bS}^{\top} \bPsi^{\top} + \lambda \bI_{|\mX|} \preceq \frac{3}{2} \rbr{\bPsi \bPsi^{\top} + \lambda \bI_{|\mX|}},
        \end{equation}
        where $\bPsi = (\psi(\bx^{(1)}), \ldots, \psi(\bx^{(T)})) \in \R^{|\mX| \times T}$.
    \end{enumerate}
    In statements 1 and 2, $\gamma_T \coloneqq \gamma_T(\lambda, \mX)$ denotes the maximum information gain of kernel $k$ on $\mX$ with the variance parameter $\lambda$.
\end{corollary}

\begin{lemma}[Approximation accuracy of multiplicative form for the standard kernel matrix, adapted from the proof of Lemma~6 in \citep{calandriello2019gaussian}]
    \label{lem:Ppsi_Qpsi_standard}
    Assume that the sampling matrix $\tilde{\bS} \in \R^{T \times \tilde{s}}$ satisfies Eq.~\eqref{eq:standard_epsilon_acc_dict}.
    Then, under the same notations in \corref{cor:acc_comp_rls_standard}, the following relation holds:
    \begin{equation}
        \frac{1}{3}\rbr{ \tilde{\bQ} \bPsi (\tilde{\bQ} \bPsi)^{\top} + \lambda \bI_{|\mX|}} \preceq \sum_{i=1}^T \psi(\bx^{(i)}) \psi(\bx^{(i)})^{\top} + \lambda \bI_{|\mX|} \preceq 3 \rbr{ \tilde{\bQ} \bPsi (\tilde{\bQ} \bPsi)^{\top} + \lambda \bI_{|\mX|}},
    \end{equation}
    where $\tilde{\bQ} \coloneqq \bPsi \tilde{\bS} \sbr{(\bPsi \tilde{\bS})^{\top} \bPsi \tilde{\bS}}^{\dagger} (\bPsi \tilde{\bS})^{\top}$ is the projection matrix into the subspace spanned by the column vectors of $\bPsi \tilde{\bS}$.
\end{lemma}

\begin{lemma}[Property of the approximated MVR-sequence]
\label{lem:prop_approx_mvr}
    Fix any finite input domain $\mX$. Fix any $T \in \N_+$, $\lambda \geq 1$, $\delta \in (0, \min\{1, T/32\})$, any positive definite kernel $k: \mX \times \mX \rightarrow \R$ with $\forall \bx \in \mX, k(\bx, \bx) \leq 1$. 
    Then, when running \algoref{algo:approx_mvr} with the input $(\mX, k, \lambda, T, \delta)$, the following event holds with probability at least $1 - 3\delta$:
    \begin{equation}
        \label{eq:prop_approx_mvr}
        \max_{\bx \in \mX} \sigma^2\rbr{\bx; \tilde{\bX}_{T}^{(\text{MVR})}, \lambda} \leq \frac{36 \lambda \gamma_{T}}{T}.
    \end{equation}
\end{lemma}
\begin{proof}
    Let $\bPsi_t = (\psi(\tilde{\bx}_1^{(\mathrm{MVR})}), \ldots, \psi(\tilde{\bx}_t^{(\mathrm{MVR})})) \in \R^{|\mX| \times t}$ be a feature matrix defined by the approximated MVR-sequence. Furthermore, let $\tilde{\bS}_t \in \R^{(t-1) \times \tilde{s}_t}$ be the sampling matrix returned by the recursive RLS-Nystr\"om algorithm at round $t$ of \algoref{algo:approx_mvr} (Line 3).
    Then, from \corref{cor:acc_comp_rls_standard} and the union bound, we can confirm that the following statement holds with probability at least $1 - 3\delta$:
    \begin{equation}
        \label{eq:union_eps_acc}
        \forall t \in [T],~
        \frac{1}{2}\rbr{\bPsi_{t-1} \bPsi_{t-1}^{\top} + \lambda \bI_{|\mX|}} \preceq \bPsi_{t-1} \tilde{\bS}_{t} \tilde{\bS}_{t}^{\top} \bPsi_{t-1}^{\top} + \lambda \bI_{|\mX|} \preceq \frac{3}{2} \rbr{\bPsi_{t-1} \bPsi_{t-1}^{\top} + \lambda \bI_{|\mX|}}.
    \end{equation}
    Therefore, it is enough to show the desired statement under the above event.
    Henceforth, we suppose that event \eqref{eq:union_eps_acc} holds in the following arguments.
    By using \lemref{lem:Ppsi_Qpsi_standard}, we can obtain the following inequality:
    \begin{align}
        \bar{\sigma}^2\rbr{\bx; \tilde{\bX}_{t-1}^{(\mathrm{MVR})}, \tilde{\bS}_t, \lambda} &= \lambda \psi(\bx)^{\top} \rbr{ \tilde{\bQ}_t \bPsi_{t-1} (\tilde{\bQ}_t \bPsi_{t-1})^{\top} + \lambda \bI_{|\mX|}}^{-1} \psi(\bx) \\
        &\leq 3 \lambda \psi(\bx)^{\top} \rbr{ \bPsi_{t-1} \bPsi_{t-1}^{\top} + \lambda \bI_{|\mX|}}^{-1} \psi(\bx) \\
        &= 3 \sigma^2\rbr{\bx; \tilde{\bX}_{t-1}^{(\mathrm{MVR})}, \lambda},
        \end{align}
    where $\tilde{\bQ}_t = \bPsi_{t-1} \tilde{\bS}_t \sbr{(\bPsi_{t-1} \tilde{\bS}_t)^{\top} \bPsi_{t-1} \tilde{\bS}_t}^{\dagger} (\bPsi_{t-1} \tilde{\bS}_t)^{\top}$. Note that the identity of $\lambda \psi(\bx)^{\top} \rbr{ \tilde{\bQ}_t \bPsi_{t-1} (\tilde{\bQ}_t \bPsi_{t-1})^{\top} + \lambda \bI_{|\mX|}}^{-1} \psi(\bx)$ and Eq.~\eqref{eq:nystrom_approx_mvr_expression} is straightforwardly obtained by the proof of \lemref{lem:identity_nystrom} as the special case.
    Similar to the above inequalities, we can also obtain the converse inequality $\sigma^2\rbr{\bx; \tilde{\bX}_{t-1}^{(\mathrm{MVR})}, \lambda} \leq 3\bar{\sigma}^2\rbr{\bx; \tilde{\bX}_{t-1}^{(\mathrm{MVR})}, \tilde{\bS}_t, \lambda}$ by using \lemref{lem:Ppsi_Qpsi_standard}. Thus, for any $\bx \in \mX$ and $t \in [T]$, we have
    \begin{align}
        \label{eq:barsigma_sigma_rel}
        \frac{1}{3} \sigma^2\rbr{\bx; \tilde{\bX}_{t-1}^{(\mathrm{MVR})}, \lambda} \leq \bar{\sigma}^2\rbr{\bx; \tilde{\bX}_{t-1}^{(\mathrm{MVR})}, \tilde{\bS}_t, \lambda} \leq 3 \sigma^2\rbr{\bx; \tilde{\bX}_{t-1}^{(\mathrm{MVR})}, \lambda}.
    \end{align}
    Here, note that the following statement also holds from the definition of $\tilde{\bx}_t^{(\text{MVR})}$:
    \begin{equation}
        \label{eq:barsigma_xt_rel}
        \forall t \in [T],~\max_{\bx \in \mX} \bar{\sigma}^2\rbr{\bx; \tilde{\bX}_{t-1}^{(\mathrm{MVR})}, \tilde{\bS}_t, \lambda} \leq \bar{\sigma}^2\rbr{\tilde{\bx}_t^{(\mathrm{MVR})}; \tilde{\bX}_{t-1}^{(\mathrm{MVR})}, \tilde{\bS}_t, \lambda}.
    \end{equation}
    By combining Eq.~\eqref{eq:barsigma_sigma_rel} with Eq.~\eqref{eq:barsigma_xt_rel}, we have
    \begin{align}
        \label{eq:apmvr_mean}
        \max_{\bx \in \mX} \sigma^2\rbr{\bx; \tilde{\bX}_{T}^{(\text{MVR})}, \lambda} 
        &\leq \frac{1}{T} \sum_{t=1}^T \max_{\bx \in \mX} \sigma^2\rbr{\bx; \tilde{\bX}_{t-1}^{(\text{MVR})}, \lambda} \\
        &\leq \frac{3}{T} \sum_{t=1}^T \max_{\bx \in \mX} \bar{\sigma}^2\rbr{\bx; \tilde{\bX}_{t-1}^{(\text{MVR})}, \tilde{\bS}_t, \lambda} \\
        &\leq \frac{3}{T} \sum_{t=1}^T \bar{\sigma}^2\rbr{\tilde{\bx}_t^{(\mathrm{MVR})}; \tilde{\bX}_{t-1}^{(\text{MVR})}, \tilde{\bS}_t, \lambda} \\
        &\leq \frac{9}{T} \sum_{t=1}^T \sigma^2\rbr{\tilde{\bx}_t^{(\mathrm{MVR})}; \tilde{\bX}_{t-1}^{(\text{MVR})}, \lambda} \\
        \label{eq:apmvr_mig}
        &\leq \frac{18 \gamma_T}{T \log (1 + \lambda^{-1})} \\
        \label{eq:apmvr_elm}
        &\leq \frac{36 \lambda \gamma_T}{T},
    \end{align}
    where Eq.~\eqref{eq:apmvr_mean} uses the monotonicity of the posterior variance with respect to the training input of the GP, Eq.~\eqref{eq:apmvr_mig} follows from Eq.~\eqref{eq:sigma_2_cum_ub}, and Eq.~\eqref{eq:apmvr_elm} follows from the elementary inequality $\forall a \in [0, 1], \log(1 + a) \geq a/2$.
\end{proof}

\begin{lemma}[Uniform upper bound of the approximated optimistic estimator $\tilde{u}_t$ with approximated MVR sequence]
\label{lem:ub_ut_tilde_approx_mvr}
    Fix any $T \in \N_+$, $\alpha \in (0, 1)$, $\eta > 0$, $\beta > 0$, and $\lambda > 0$. Fix any finite input domain $\mX$.
    Suppose that \asmpref{asmp:regularity} holds. Then, when running \algoref{alg:exp3_nystrom_approx_mvr}, under event $\cap_{t \in [T]} A_t$ (see Eq.~\eqref{eq:thm_rke3_proof_eps_acc}) and $\max_{\bx \in \mX} \sigma^2\rbr{\bx; \tilde{\bX}_{\lceil T\alpha \rceil}^{(\text{MVR})}, \lambda} \leq \frac{36 \lambda \gamma_{\lceil T\alpha \rceil}}{\lceil T\alpha \rceil}$ (see Eq.~\eqref{eq:prop_approx_mvr}), the following upper bound of the approximated optimistic estimator $\tilde{u}_t$ holds for any $t \in [T]$ and $\bx \in \mX$:
    \begin{equation}
        \tilde{u}_t(\bx) \leq \frac{36 (3 + \sqrt{6}) B \gamma_{T}}{\alpha} + \beta \sqrt{\frac{108 \gamma_{T}}{\alpha}}
        \leq \frac{1}{\alpha} \rbr{36(3 + \sqrt{6}) B \gamma_{T} + \beta \sqrt{108\gamma_T}}.
    \end{equation}
\end{lemma}
\begin{proof}
    By applying the proof of \lemref{lem:ub_ut_tilde}, we can confirm the following inequality (see Eq.~\eqref{eq:ub_ut_tilde_psi_max_ub}).
    \begin{align}
        \tilde{u}_t(\bx) 
        \leq (3 + \sqrt{6}) B \max_{\bx' \in \mX} \|\psi(\bx')\|_{G_t(\lambda)^{-1}}^2 + \sqrt{3} \beta \max_{\bx' \in \mX} \|\psi(\bx')\|_{G_t(\lambda)^{-1}}.
    \end{align}
    Then, we obtain the upper bound of $\|\psi(\bx')\|_{G_t(\lambda)^{-1}}^2$ as 
    \begin{align}
        \|\psi(\bx')\|_{G_t(\lambda)^{-1}}^2
        &\leq
        \frac{\lceil T\alpha \rceil}{\alpha \lambda} \sigma^2\rbr{\bx'; \tilde{\bX}_{\lceil T\alpha \rceil}^{(\text{MVR})}, \lambda} \\
        &\leq \frac{36 \gamma_{T}}{\alpha},
    \end{align}
    where the first inequality follows from the same arguments provided in Eqs.~\eqref{eq:psi_ub_first}--\eqref{eq:psi_mvr_sigma}, and the second inequality follows from event $\max_{\bx \in \mX} \sigma^2\rbr{\bx; \tilde{\bX}_{\lceil T\alpha \rceil}^{(\text{MVR})}, \lambda} \leq \frac{36 \lambda \gamma_{\lceil T\alpha \rceil}}{\lceil T\alpha \rceil}$.
    Hence, we obtain
    \begin{equation}
        \label{eq:ub_ut_tilde_approx_mvr_last}
        \tilde{u}_t(\bx) \leq \frac{36 (3 + \sqrt{6}) B \gamma_{T}}{\alpha} + \beta \sqrt{\frac{108 \gamma_{T}}{\alpha}}
        \leq \frac{1}{\alpha} \rbr{36(3 + \sqrt{6}) B \gamma_{T} + \beta \sqrt{108\gamma_T}}.
    \end{equation}
\end{proof}

\subsection{Helper lemma for lower bounds}

\begin{lemma}[Application of the lower bounds for stochastic setting, adapted from Theorem 2 in \citep{cai2021on}]
    \label{lem:lower}
    Fix any $d \in \N_+$, $B, \ell, \nu, \sigma^2 > 0$ and $\mX = [0, 1]^d$. 
    Fix kernel function $k: \mX \times \mX \rightarrow \R$. Assume that $B$, $d$, $\ell$, and $\nu$ are $\Theta(1)$. 
    Furthermore, let us consider the adversarial KB problem with $f_t(\cdot) \coloneqq f(\cdot) + \tilde{\eta}_t k(\bm{0}, \cdot)$, where $f \in \mF_{k}(B) \coloneqq \{f \in \mH_{k} \mid \|f\|_k \leq B\}$ and $\tilde{\eta}_t \sim \mN(0, \sigma^2)$. In addition, we assume $\sigma^2 = O(T)$ with a sufficiently small implied constant.
    Then, for any algorithm, there exists a function $f \in \mF_k(B)$ such that
    \begin{equation}
        \sup_{\bx \in \mX} \Ep\sbr{\sum_{t=1}^T f(\bx) - f(\bx_t)} =
        \begin{cases}
            \Omega\rbr{\sqrt{T\sigma^2 (\log (T/\sigma^2))^{d/2}}}~~&\mathrm{if}~~k = \sek, \\
            \Omega\rbr{\sigma^{\frac{2\nu}{2\nu + d}} T^{\frac{\nu +d}{2\nu+d}}}~~&\mathrm{if}~~k = \matk.
        \end{cases}
    \end{equation}
\end{lemma}
The above lemma is obtained with a subtle modification of the proof of Theorem~2 in \citep{cai2021on}. Thus, we omit the full proof while providing a description of the required modification. The only part of the proof of Theorem~2 in \citep{cai2021on} that requires modification is the upper bound of Kullback–Leibler (KL) divergence in the change of measure argument. In this section, hereafter, we use the same notations as those in \citep{cai2021on}. When we apply the proof of \citep{cai2021on}, the maximum KL divergence term: $\bar{D}_{f, f'}^{j}$ under the two functions $f, f'$ and some partition $\mR_j \subset \mX$ indexed with $j$ becomes different from the original proof. Since they study the Gaussian noise model whose variance $\sigma^2$ is homogeneous over the input domain, the term $\bar{D}_{f, f'}^{j}$ is given as 
\begin{equation}
    \bar{D}_{f, f'}^{j} = \max_{\bx \in \mR_j} \frac{(f(\bx) - f'(\bx))^2}{2\sigma^2}.
\end{equation}
See Lemma 1 and Eqs.~(9) and (19) in \citep{cai2021on}.
On the other hand, the term $\bar{D}_{f, f'}^{j}$ in our model takes the following form due to the input dependence of the noise term $\tilde{\eta}_t k(\bm{0}, \cdot)$:
\begin{equation}
    \bar{D}_{f, f'}^{j} = \max_{\bx \in \mR_j} \frac{(f(\bx) - f'(\bx))^2}{2 \sigma^2 k^2(\bm{0}, \bx)}. 
\end{equation}
The upper bound of the above term $\bar{D}_{f, f'}^{j}$ can be further obtained as 
\begin{equation}
    \bar{D}_{f, f'}^{j} \leq \max_{\bx \in \mR_j} \frac{(f(\bx) - f'(\bx))^2}{2 \sigma^2 k^2(\bm{0}, \bm{1})},
\end{equation}
where $\bm{1} = (1, \ldots, 1) \in \R^{d}$. 
By using the above upper bound instead of the exact value of the KL divergence $\frac{(f(\bx) - f'(\bx))^2}{2\sigma^2}$, we can obtain the desired lower bound by directly following the proof of \citep{cai2021on}. Specifically, the resulting lower bound matches the lower bound in \citep{cai2021on} with the smaller noise variance $\sigma^2 k^2(\bm{0}, \bm{1})$. Note that $k^2(\bm{0}, \bm{1})$ depends on only $d$, $\nu$, and $\ell$, which are $\Theta(1)$; thus, we can recover the stochastic KB lower bound up to $(d, \nu, \ell)$-dependent constant factors.

\section{Additional discussion about exploration distribution}
\label{app:exp_dist_discuss}
Our algorithm uses the empirical distribution of the MVR sequence as the exploration distribution $\pi$. 
In this section, we discuss the other possible choices of $\pi$.

\paragraph{Exploration distribution based on generalized G-optimal design.} The role of exploration distribution in our proof is to obtain Lemmas~\ref{lem:ub_ut} and \ref{lem:ub_ut_tilde}, which are required for applying Exp3 analysis by guaranteeing $\forall \bx \in \mX, \eta u_t(\bx) \leq 1$ and $\forall \bx \in \mX, \eta \tilde{u}_t(\bx) \leq 1$~(see Lemma~\ref{lem:exp3_standard}).
Although we used the MVR-based empirical distribution to obtain Lemmas~\ref{lem:ub_ut} and \ref{lem:ub_ut_tilde}, this is not the only choice. The requirement of $\pi(\cdot)$ for extending our proof of Lemmas~\ref{lem:ub_ut} and \ref{lem:ub_ut_tilde} is that the upper bound of the following supremum of the quadratic form increases with $O(\gamma_{\lceil T \alpha\rceil})$:
\begin{equation}
    \label{eq:necessary_cond_for_pi}
    \sup_{\bx \in \mX} \psi(\bx)^{\top} \rbr{\sum_{\bx' \in \mX} \pi(\bx') \psi(\bx') \psi(\bx')^{\top} + \frac{\lambda}{T\alpha} \bI_{|\mX|}}^{-1} \psi(\bx) = O(\gamma_{\lceil T \alpha \rceil}).
\end{equation}
See Eqs.~\eqref{eq:psi_ub_first}--\eqref{eq:psi_mvr_sigma_ub} in the proof of Lemma~\ref{lem:ub_ut}. 
From the above fact, the natural choice is to choose the following distribution $\pi^{\ast}$ by minimizing the left-hand side of the above expression:
\begin{equation}
    \pi^{\ast} \in \argmin_{P \in \mP_{\mX}}\sup_{\bx \in \mX} \psi(\bx)^{\top} \rbr{\sum_{\bx' \in \mX} P(\bx') \psi(\bx') \psi(\bx')^{\top} + \frac{\lambda}{T\alpha} \bI_{|\mX|}}^{-1} \psi(\bx),
\end{equation}
where $\mP_{\mX} \coloneqq \{P: \mX \rightarrow [0, 1] \mid \sum_{\bx \in \mX} P(\bx) = 1\}$ denotes the set of all probability mass functions on $\mX$. This $\pi^{\ast}$ is the generalization of G-optimal design~\citep[e.g., Chap.~21 in][]{lattimore2020bandit} and is used in stochastic KB literature~\citep{camilleri2021high}. Even if we replace $\pi$ with $\pi^{\ast}$ in Algs.~\ref{alg:exp3} and \ref{alg:exp3_nystrom_approx}, we can obtain the same regret guarantees as those in Theorems~\ref{thm:reg_exp3} and \ref{thm:reg_exp3_nystrom_approx}, and \corref{coro:ke3_reg_tuned}. 
Then, the resulting algorithm with $\pi^{\ast}$ can be interpreted as the kernelized extension of Exp3 with Kiefer-Wolfowitz exploration in adversarial LB~(Alg.~3 in \cite{zimmert2022return}).
The drawback of this approach is the computational cost of $\pi^{\ast}$. To our knowledge, even in the calculation of the near-optimal approximation of $\pi^{\ast}$, the total computational cost is $\Omega(|\mX|^3)$.
For example, the well-known iterative method based on the Frank-Wolfe algorithm~\citep[e.g., Chap.~21 in ][]{lattimore2020bandit} requires $O(|\mX| \log \log |\mX|)$ iterations with $O(|\mX|^3)$ per-iteration complexity\footnote{Specifically, even if we rely on the rank-one update formula of the matrix inverse, the calculation of the maximum of the quadratic term per-iteration requires $O(|\mX|^2)$-calculation of the quadratic term for all input $\bx \in \mX$, which leads to $O(|\mX|^3)$ per-iteration complexity.}. 
For the Frank--Wolfe implementation described above, the resulting bound is $O(|\mX|^4 \log \log |\mX|)$. In the typical KB regime $|\mX| \gg T$, this bound is asymptotically larger than the $O(|\mX|\lceil T\alpha \rceil^3)$ bound for preparing the MVR-based exploration distribution $\pi$.

\paragraph{Uniform distribution.} 
Another possible choice is the uniform distribution $\pi_{\mathrm{unif}}(\cdot) \coloneqq 1/|\mX|$. If we can use $\pi_{\mathrm{unif}}$, the computational cost for exploration distribution prior to the algorithm run is completely omitted.
In this case, by applying \lemref{lem:est_var} to the quadratic term in Eq.~\eqref{eq:necessary_cond_for_pi}, we observe that 
\begin{equation}
    \sup_{\bx \in \mX} \psi(\bx)^{\top} \rbr{\sum_{\bx' \in \mX} \pi_{\mathrm{unif}}(\bx') \psi(\bx') \psi(\bx')^{\top} + \frac{\lambda}{T\alpha} \bI_{|\mX|}}^{-1} \psi(\bx) \leq \frac{\lceil T\alpha \rceil}{\lambda} \sup_{\bx \in \mX} \Ep_{\bX_{\text{unif}}}[\sigma^2(\bx; \bX_{\text{unif}}, \lambda)],
\end{equation}
where $\bX_{\text{unif}} \coloneqq (\bx_{\text{unif}}^{(1)}, \ldots, \bx_{\text{unif}}^{(\lceil T \alpha \rceil)})$ with $\bx_{\text{unif}}^{(1)}, \ldots, \bx_{\text{unif}}^{(\lceil T \alpha \rceil)} \sim_{\mathrm{i.i.d.}} \pi_{\mathrm{unif}}$. Thus, if $\sup_{\bx \in \mX} \Ep_{\bX_{\text{unif}}}[\sigma^2(\bx; \bX_{\text{unif}}, \lambda)] = O\rbr{\frac{\lambda \gamma_{\lceil T\alpha \rceil}}{\lceil T\alpha \rceil}}$, we can verify the condition \eqref{eq:necessary_cond_for_pi}. The examination for the upper bound $\sup_{\bx \in \mX} \Ep_{\bX_{\text{unif}}}[\sigma^2(\bx; \bX_{\text{unif}}, \lambda)] = O\rbr{\frac{\lambda \gamma_{\lceil T\alpha \rceil}}{\lceil T\alpha \rceil}}$ is found in the stochastic KB literature~\citep{salgiarandom}. \citet{salgiarandom} show that $\sup_{\bx \in \mX} \sigma^2(\bx; \bX_{\text{unif}}, \lambda) = O\rbr{\frac{\lambda \gamma_{\lceil T\alpha \rceil}}{\lceil T\alpha \rceil}}$ asymptotically holds with high probability under the \emph{uniform boundedness assumption} on the eigenfunctions of the kernel with respect to $\pi_{\mathrm{unif}}$. However, to our knowledge, the validity of the uniform boundedness assumption for commonly-used kernels in KB (such as SE or $\nu$-Mat\'ern kernels) is unclear\footnote{To our knowledge, the examples of the kernels for which the uniform boundedness assumption is rigorously verified are $1/2$-Mat\'ern kernel with $d = 1$ (discussed in \citep{janz2022sequential}) and Hilbert-space approximation of SE and $\nu$-Mat\'ern kernels~\citep{riutort2023practical,solin2020hilbert}.}. For example, see discussion in Chap.~4.4 in \citep{janz2022sequential}. Thus, we believe that further careful theoretical examination is desired for the utilization of $\pi_{\text{unif}}$.

\section{Simulation experiments}
\label{app:sim_experiment}
We conduct simulation experiments to confirm the empirical behavior of our algorithms. 

\paragraph{Input domain, kernel, and strategy of the environment.} We conduct experiments using the following four input domains $\mX$:
\begin{itemize}
    \item \textbf{Setting 1.} We set $\mX$ as $400$ uniformly spaced grid points of $[0, 1]$. 
    \item \textbf{Setting 2.} We set $\mX$ as $20 \times 20$ uniformly spaced grid points of $[0, 1]^2$. 
    \item \textbf{Setting 3.} We set $\mX$ as $10000$ uniformly spaced grid points of $[0, 1]$. 
    \item \textbf{Setting 4.} We set $\mX$ as $100 \times 100$ uniformly spaced grid points of $[0, 1]^2$. 
\end{itemize}
For each setting of input domain, we conduct experiments with three kernel functions: (i) the SE kernel, (ii) the $\nu$-Mat\'ern kernel with $\nu = 5/2$, and (iii) the $\nu$-Mat\'ern kernel with $\nu = 3/2$. For each kernel, we set the lengthscale parameter $\ell$ as $\ell = 0.3\sqrt{d}$. To define $f_t$ in the experiment, we consider the fully-adversarial environment. Specifically, with some pre-defined candidate set $\mF \subset \mF_k(B)$ and the learner's sampling distribution $P_t$, we assume that the environment chooses $f_t$ so that the learner's conditional expected instantaneous regret becomes highest, as follows: $f_t \in \argmax_{f \in \mF} (\max_{\bx \in \mX} f(\bx) - \sum_{\bx' \in \mX} P_t(\bx') f(\bx'))$. We define the candidate set $\mF$ as $\mF = (f^{(i)})_{i \in [1000]}$ with $f^{(i)}(\cdot) = \min\{1, B/\|\tilde{f}^{(i)}\|_k\} \tilde{f}^{(i)}(\cdot)$ and $\tilde{f}^{(i)}(\cdot) = \sum_{m=1}^{100} c^{(m, i)} k(\bx^{(m, i)}, \cdot)$, where $\bx^{(m, i)} \sim \mathrm{Uniform}(\mX)$ and $c^{(m, i)} \sim \mathrm{Uniform}(-1, 1)$ are independent random variables. Note that we can calculate $\|\tilde{f}^{(i)}\|_k$ as $\|\tilde{f}^{(i)}\|_k = \sqrt{\sum_{m, m'}c^{(m,i)}c^{(m',i)}k(\bx^{(m,i)}, \bx^{(m',i)}) }$ under the definition of $\tilde{f}^{(i)}$. Regarding the RKHS norm upper bound $B$, we set $B = 2$.

\paragraph{Algorithms.} We compare the following algorithms:

\begin{itemize}
    \item \textbf{Random.} At each round $t$, this algorithm draws $\bx_t$ uniformly at random from $\mX$.
    \item \textbf{APG-Exp3.} This is a variant of the Exp3 algorithm based on a linear approximation, proposed in \citep{takemori2021approximation}. We use the G-optimal design $\tilde{\pi}^{\ast} \in \argmin_{P \in \mP_{\mX}} \max_{\bx \in \mX}\|\bm{N}(\bx)\|_{(\sum_{\bx' \in \mX}P(\bx') \bm{N}(\bx') \bm{N}(\bx')^{\top})^{-1}}$ as the exploration distribution, where $\bm{N}(\bx) \in \R^{D_T}$ is the feature vector based on the Newton basis functions. To calculate $\tilde{\pi}^{\ast}$, we use the Frank-Wolfe algorithm~\citep{fedorov2013theory}.
    The basis function $\bm{N}(\bx)$ and $D_T$ are determined in the P-greedy algorithm as described in \citep{takemori2021approximation}. We fix the admissible error $\mathfrak{e}$ of the P-greedy algorithm as the theoretically suggested value $\mathfrak{e} = (\log |\mX|)/T$.
    Regarding the choice of the learning rate $\eta > 0$ and mixing ratio $\alpha \in (0, 1)$\footnote{The parameter $\alpha$ is defined as notation $\gamma$ in \citep{takemori2021approximation}.}, we set $\eta = c_1 / \sqrt{D_T T}$ and $\alpha = c_2 \eta D_T$ for some constants $c_1 > 0$ and $c_2 > 0$. See \citep{takemori2021approximation} for details. Here, these settings are based on the theoretically suggested diminishing rate of $\eta$ and $\alpha$ for achieving $\tilde{O}(\sqrt{TD_T})$ regret proved in \citep{takemori2021approximation}. 
    For each kernel, we fix the constants $c_1$ and $c_2$ based on the average performance over $10$ different seeds under setting 1 described above, with $T = 50$. Specifically, for each kernel, we choose the constants $c_1$ and $c_2$ by grid search across $c_1, c_2 \in \{0.1, 0.5, 1.0, 5.0, 10.0\}$. 
    The selected settings of $c_1, c_2$ are $(c_1, c_2) = (0.5, 0.1)$, $(c_1, c_2) = (0.5, 0.1)$, and $(c_1, c_2) = (0.5, 0.1)$ for the SE, $5/2$-Mat\'ern, and $3/2$-Mat\'ern kernels, respectively.
    \item \textbf{Kernelized Exp3.} This is the kernelized Exp3 algorithm proposed in \secref{sec:kernelized_exp3}. Regarding the choice of the parameters, based on the theoretically suggested values in \thmref{thm:reg_exp3}, 
    we set $\eta = c_3 / \sqrt{T \bar{\gamma}_T}$ and $\alpha = c_4 \eta \bar{\gamma}_T$ with some constants $c_3 > 0$ and $c_4 > 0$. Here, $\bar{\gamma}_T$ denotes a known growth rate upper bound on MIG, which is defined as $\bar{\gamma}_T = (\log T)^{d+1}(\log \log T)^{-d}$ and $\bar{\gamma}_T = T^{\frac{d}{2\nu+d}} (\log T)^{\frac{4\nu+d}{2\nu+d}}$ for $k = \sek$ and $k = \matk$, respectively~\citep{iwazaki2025improved,iwazaki2026tighter}. For each kernel, to choose constants $c_3$, $c_4$, and the regularization parameter $\lambda \geq 1$, we conduct the grid search across $c_3, c_4 \in \{0.1, 0.5, 1.0, 5.0, 10.0\}$ and $\lambda \in \{1.0, 5.0, 10.0\}$ as with the parameter selections of APG-Exp3. Here, we set the confidence width parameter $\beta$ as the theoretically suggested value $\beta \coloneqq B\sqrt{\lambda/T}$ in \thmref{thm:reg_exp3}. Here, we omit the full experiments on kernelized Exp3 for the large-scale input domain setting (settings 3 and 4) due to its high computational cost; however, as a baseline for the computational time, we report the estimated per-round computational time for kernelized Exp3 by running the algorithm for only 1 round.
    The selected settings of $c_3, c_4, \lambda$ are $(c_3, c_4, \lambda) = (5.0, 0.1, 1.0)$, $(c_3, c_4, \lambda) = (10.0, 0.1, 10.0)$, and $(c_3, c_4, \lambda) = (5.0, 0.1, 5.0)$ for the SE, $5/2$-Mat\'ern, and $3/2$-Mat\'ern kernels, respectively.
    \item \textbf{RLS-Kernelized Exp3.} This is the RLS-based kernelized Exp3 algorithm proposed in \secref{sec:exp3_nystrom_approx}. Regarding the choice of $\eta$, $\alpha$, and $\lambda$, we use the same settings as used in kernelized Exp3.
    We set the confidence width parameter $\beta$ as $\beta = B (1 + \sqrt{2})\sqrt{\lambda/T}$ as suggested in \thmref{thm:reg_exp3_nystrom_approx}. Furthermore, we fix the confidence level of RLS-Nystr\"om as $0.1$ for each iteration~\footnote{This setting formally corresponds to $\delta=0.3T$ in \algoref{alg:exp3_nystrom_approx} and is not covered by our theoretical guarantees. We adopt it as a practical choice to limit the sampling overhead associated with the horizon-dependent confidence allocation. In \tabref{tab:results}, we observe broadly similar average regret to kernelized Exp3 in settings 1 and 2, and lower per-round runtimes in settings 3 and 4 compared with the kernelized Exp3 runtime estimates obtained from a single round.}.
\end{itemize}

\paragraph{Resource for computation.} Our experiments are conducted on an AMD EPYC 7702P 64-core processor with 16 GB RAM. 

\paragraph{Results.} We report the average regret and per-round computational time over $10$ trials with different seeds. \tabref{tab:results} shows the results.
Firstly, we confirm that the regret of our proposed algorithms is substantially lower than that of other baseline methods. Secondly, we confirm that the computational time of our algorithms is worse than that of the baseline algorithms. 
However, in the large input domain $|\mX| = 10000$~(settings 3 and 4), the per-round cost of RLS-kernelized Exp3 is improved compared with the estimated per-round cost of the original kernelized Exp3.

In our experiments, we confirm that the computational time of RLS-kernelized Exp3 for SE kernel is worse than that of $5/2$-Mat\'ern and $3/2$-Mat\'ern. 
This phenomenon appears to contradict our theoretical guarantees, 
where the kernel with the small information gain leads to small computational cost of RLS-kernelized Exp3. However, as described in the algorithm setting, the regularization parameter $\lambda = 1.0$ chosen by the grid search for SE-kernel is smaller than $\lambda=10.0$ and $\lambda=5.0$ for $5/2$-Mat\'ern and $3/2$-Mat\'ern kernels.
Therefore, RLS-kernelized Exp3 takes longer to run with the SE kernel than with the $5/2$-Mat\'ern and $3/2$-Mat\'ern kernels in our experiments, because the smaller regularization parameter $\lambda$ leads to a larger information gain $\gamma_T$.
We also found the computational cost of RLS-kernelized Exp3 exceeds that of APG-Exp3. We observe that the reason for this phenomenon is the large constant factor for $\tilde{O}(|\mX|\gamma_T^2)$-computation of the recursive Nystr\"om algorithm provided in \lemref{lem:acc_comp_rls}.

\begin{landscape}
\begin{table}
    \centering
    \footnotesize
    \renewcommand{\arraystretch}{0.95}
    \caption{Experimental results for the average regret and per-round computational cost with $T \in \{50, 100, 150\}$. The abbreviations K-Exp3 and RLS-K-Exp3 in the table denote our kernelized Exp3 and RLS-kernelized Exp3 algorithms, respectively. All numerical values are rounded to one decimal place. 
    Two standard errors are shown in parentheses. Note that we do not conduct experiments of kernelized Exp3 in settings 3 and 4, while we report the estimated per-round computational time by running the algorithm for only one round.
    }
    \begin{tabular}{llllllllllll}
    \toprule
                           & & \multicolumn{1}{c}{} & \multicolumn{3}{c}{SE Kernel} & \multicolumn{3}{c}{$5/2$-Mat\'ern} & \multicolumn{3}{c}{$3/2$-Mat\'ern} \\
                           & &                     & $T=50$ & $T=100$ & $T=150$ & $T=50$  & $T=100$ & $T=150$  & $T=50$  & $T=100$ & $T=150$ \\ \midrule \midrule
 \multirow{16}{*}{Regret} & \multirow{4}{*}{Setting 1} & Random               &  89.2 (2.8) & 178.0 (4.3) & 268.1 (4.4) & 84.4 (2.6) & 168.3 (3.7) & 253.6 (4.1) & 78.1 (4.0) & 156.4 (7.6) & 234.4 (7.2) \\
                          & & APG-Exp3                   &  28.2 (7.7) & 42.9 (9.9) & 60.7 (10.7) & 28.0 (5.5) & 48.2 (8.0) & 69.3 (4.5) & 23.5 (4.3) & 42.9 (4.9) & 60.6 (6.4) \\
                          & & K-Exp3            & 18.8 (3.0) & 22.4 (10.3) & 34.9 (13.9) & 16.1 (4.7) & 32.7 (8.5) & 41.5 (6.1) & 14.8 (4.7) & 28.2 (7.9) & 42.9 (11.7) \\
                          & & RLS-K-Exp3        & 20.4 (5.8) & 27.9 (4.8) & 43.7 (11.3) & 19.1 (4.7) & 31.0 (5.6) & 51.9 (6.5) & 14.7 (6.7) & 24.5 (7.9) & 41.4 (13.2) \\ \cmidrule(lr){2-12}
& \multirow{4}{*}{Setting 2} & Random         & 79.6 (2.3) & 158.2 (2.8) & 235.1 (3.0) & 75.0 (2.1) & 149.2 (2.6) & 221.5 (2.8) & 72.4 (2.0) & 144.0 (2.4) & 213.7 (2.7) \\
                          & & APG-Exp3                    & 35.8 (7.6) & 50.8 (7.5) & 69.0 (10.2) & 29.1 (7.5) & 55.5 (13.6) & 73.1 (12.3) & 36.6 (9.3) & 80.5 (12.6) & 129.6 (19.1) \\
                         & & K-Exp3      &   18.4 (4.0) & 25.6 (4.5) & 40.3 (7.7) & 12.9 (3.2) & 22.6 (5.9) & 26.3 (8.2) & 13.3 (4.4) & 25.0 (8.5) & 24.2 (8.2) \\
                          & & RLS-K-Exp3  &  16.1 (3.1) & 32.6 (5.5) & 50.3 (10.4) & 11.1 (4.5) & 21.5 (6.0) & 20.9 (5.7) & 16.8 (4.0) & 26.0 (8.2) & 29.9 (6.5) \\ \cmidrule(lr){2-12}
& \multirow{4}{*}{Setting 3} & Random                      &  89.2 (2.8) & 178.1 (4.3) & 268.2 (4.4) & 84.5 (2.6) & 168.4 (3.7) & 253.7 (4.1) & 81.8 (2.5) & 163.0 (3.5) & 245.5 (3.9)  \\
                         &  & APG-Exp3                   & 27.4 (6.4) & 45.4 (6.7) & 65.9 (8.2) & 31.0 (4.8) & 47.1 (8.0) & 67.7 (11.4) & 23.5 (5.8) & 41.5 (6.4) & 62.6 (11.1) \\
                          & & K-Exp3      &    \multicolumn{9}{c}{N/A}   \\
                          & & RLS-K-Exp3  &      20.3 (5.4) & 26.0 (10.3) & 37.4 (13.0) & 26.7 (5.2) & 37.6 (8.4) & 57.8 (11.6) & 21.9 (5.7) & 44.8 (11.2) & 44.7 (9.1) \\ \cmidrule(lr){2-12}
& \multirow{4}{*}{Setting 4} & Random                     &  76.4 (2.5) & 152.7 (4.8) & 229.2 (5.5) & 72.9 (3.4) & 145.3 (3.4) & 219.4 (3.3) & 70.4 (3.3) & 140.4 (3.3) & 211.9 (3.2) \\
                          & & APG-Exp3                    & 35.8 (3.9) & 73.0 (7.3) & 95.3 (14.7) & 33.1 (5.1) & 61.9 (7.6) & 116.6 (12.1) & 38.8 (5.7) & 87.3 (7.7) & 132.4 (13.0) \\
                          & & K-Exp3      &     \multicolumn{9}{c}{N/A}     \\
                          & & RLS-K-Exp3  &   18.7 (5.2) & 30.1 (8.8) & 48.9 (10.8) & 18.3 (4.5) & 35.6 (7.2) & 38.0 (5.4) & 18.0 (5.6) & 24.0 (4.0) & 28.4 (7.7)  \\ \midrule
 & \multirow{4}{*}{Setting 1} & Random     & 0.0 (0.0) & 0.0 (0.0) & 0.0 (0.0) & 0.0 (0.0) & 0.0 (0.0) & 0.0 (0.0) & 0.0 (0.0) & 0.0 (0.0) & 0.0 (0.0)  \\
                          & & APG-Exp3                   &  0.0 (0.0) & 0.0 (0.0) & 0.0 (0.0) & 0.0 (0.0) & 0.0 (0.0) & 0.0 (0.0) & 0.0 (0.0) & 0.0 (0.0) & 0.0 (0.0)   \\
                          & & K-Exp3      &   0.0 (0.0) & 0.0 (0.0) & 0.0 (0.0) & 0.0 (0.0) & 0.0 (0.0) & 0.0 (0.0) & 0.0 (0.0) & 0.0 (0.0) & 0.0 (0.0) \\
                          & & RLS-K-Exp3  &   0.1 (0.0) & 0.1 (0.0) & 0.1 (0.0) & 0.1 (0.0) & 0.1 (0.0) & 0.1 (0.0) & 0.1 (0.0) & 0.1 (0.0) & 0.1 (0.0)  \\ \cmidrule(lr){2-12}
& \multirow{4}{*}{Setting 2} & Random                      &  0.0 (0.0) & 0.0 (0.0) & 0.0 (0.0) & 0.0 (0.0) & 0.0 (0.0) & 0.0 (0.0) & 0.0 (0.0) & 0.0 (0.0) & 0.0 (0.0)  \\
                          & & APG-Exp3                   & 0.0 (0.0) & 0.0 (0.0) & 0.0 (0.0) & 0.0 (0.0) & 0.0 (0.0) & 0.0 (0.0) & 0.0 (0.0) & 0.0 (0.0) & 0.0 (0.0)  \\
                          & & K-Exp3      &  0.0 (0.0) & 0.0 (0.0) & 0.0 (0.0) & 0.0 (0.0) & 0.0 (0.0) & 0.0 (0.0) & 0.0 (0.0) & 0.0 (0.0) & 0.0 (0.0)  \\
    Per-round             & & RLS-K-Exp3         &  0.1 (0.0) & 0.1 (0.0) & 0.1 (0.0) & 0.1 (0.0) & 0.1 (0.0) & 0.1 (0.0) & 0.1 (0.0) & 0.1 (0.0) & 0.1 (0.0)  \\ \cmidrule(lr){2-12}
 time cost [s] & \multirow{4}{*}{Setting 3} & Random                      & 0.0 (0.0) & 0.0 (0.0) & 0.0 (0.0) & 0.0 (0.0) & 0.0 (0.0) & 0.0 (0.0) & 0.0 (0.0) & 0.0 (0.0) & 0.0 (0.0)  \\
                          & & APG-Exp3                    &   0.0 (0.0) & 0.0 (0.0) & 0.0 (0.0) & 0.0 (0.0) & 0.0 (0.0) & 0.0 (0.0) & 0.0 (0.0) & 0.0 (0.0) & 0.0 (0.0)  \\
                          & & K-Exp3      &  \multicolumn{9}{c}{39.5 (0.0)}   \\
                          & & RLS-K-Exp3  & 0.9 (0.0) & 1.1 (0.0) & 1.3 (0.0) & 0.2 (0.0) & 0.4 (0.0) & 0.6 (0.0) & 0.5 (0.0) & 0.7 (0.0) & 1.1 (0.0)    \\ \cmidrule(lr){2-12}
& \multirow{4}{*}{Setting 4} & Random                      &   0.0 (0.0) & 0.0 (0.0) & 0.0 (0.0) & 0.0 (0.0) & 0.0 (0.0) & 0.0 (0.0) & 0.0 (0.0) & 0.0 (0.0) & 0.0 (0.0) \\
                          & & APG-Exp3                    &  0.0 (0.0) & 0.0 (0.0) & 0.0 (0.0) & 0.0 (0.0) & 0.0 (0.0) & 0.0 (0.0) & 0.0 (0.0) & 0.0 (0.0) & 0.0 (0.0)  \\
                          & & K-Exp3      &  \multicolumn{9}{c}{40.6 (0.0)}   \\
                          & & RLS-K-Exp3  &        1.5 (0.0) & 2.2 (0.0) & 3.4 (0.0) & 0.3 (0.0) & 0.5 (0.0) & 1.0 (0.0) & 0.7 (0.0) & 1.4 (0.0) & 2.5 (0.0) \\
    \bottomrule
\end{tabular}
    \label{tab:results}
\end{table}
\end{landscape}

\end{document}